\documentclass{article}

\usepackage[left=1.25in,top=1.25in,right=1.25in,bottom=1.25in,head=1.25in]{geometry}
\usepackage{amsfonts,amsmath,amssymb,amsthm}
\usepackage{verbatim,float,url}
\usepackage{graphicx,subcaption,psfrag}
\usepackage{natbib}
\usepackage[percent]{overpic}

\usepackage{floatpag}
\floatpagestyle{empty}

\newtheorem{algorithm}{Algorithm}
\newtheorem{theorem}{Theorem}
\newtheorem{lemma}{Lemma}

\theoremstyle{remark}

\theoremstyle{definition}

\newcommand{\argmin}{\mathop{\mathrm{argmin}}}
\newcommand{\argmax}{\mathop{\mathrm{argmax}}}

\newcommand{\st}{\mathop{\mathrm{subject\,\,to}}}

\def\half{\frac{1}{2}}

\def\row{\mathrm{row}}
\def\nul{\mathrm{null}}

\def\sign{\mathrm{sign}}

\def\hx{\hat{x}}

\def\hu{\hat{u}}
\def\hbeta{\hat{\beta}}
\def\tbeta{\tilde{\beta}}
\def\tx{\tilde{x}}

\def\R{\mathbb{R}}

\def\cI{\mathcal{I}}

\title{A General Framework for Fast Stagewise Algorithms}
\author{Ryan J. Tibshirani \\ Carnegie Mellon University \\ 
{\tt\small ryantibs@stat.cmu.edu}}  
\date{}

\begin{document}
\maketitle

\begin{abstract}
Forward stagewise regression follows a very simple strategy for 
constructing a sequence of sparse regression estimates: it starts with   
all coefficients equal to zero, and iteratively updates the
coefficient (by a small amount $\epsilon$) of the 
variable that achieves the maximal absolute inner product with the
current residual.  This procedure has an interesting connection to the 
lasso: under some conditions, it is known that the sequence of
forward stagewise estimates exactly coincides with the
lasso path, as the step size $\epsilon$ goes to zero. Furthermore,
essentially the same equivalence holds outside of least squares
regression, with the minimization of a differentiable convex loss
function subject to an $\ell_1$ norm constraint (the stagewise
algorithm now updates the coefficient corresponding to the maximal
absolute component of the gradient).   

Even when they do not match their $\ell_1$-constrained analogues,
stagewise estimates provide a useful approximation, and are
computationally appealing. 
Their success in sparse modeling motivates the question: can
a simple, effective strategy like forward stagewise be
applied more broadly in other regularization settings, beyond the 
$\ell_1$ norm and sparsity?  The current paper is an attempt to do
just this.  We present a general framework for stagewise estimation,
which yields fast algorithms for problems such as group-structured
learning, matrix completion, image denoising, and more.   

\end{abstract}


\section{Introduction}
\label{sec:introduction}

In a regression setting, let $y \in \R^n$ denote an
outcome vector and $X \in \R^{n\times p}$ a matrix of predictor
variables, with columns $X_1,\dots X_p \in \R^n$.  
For modeling $y$ as a linear function of $X$, we begin by considering
(among the many possible candidates for sparse estimation tools) a
simple method: {\it forward stagewise regression.}  In words,  
forward stagewise regression produces a sequence of coefficient 
estimates $\beta^{(k)}$, $k=0,1,2,\ldots$, by iteratively decreasing
the maximal absolute inner product of a variable with the current
residual, each time by only a small amount. A more
precise description of the algorithm is as follows.  

\begin{algorithm}[\textbf{Forward stagewise regression}]
\label{alg:fsr}
\hfill\par
\smallskip\smallskip
\noindent
Fix $\epsilon>0$, initialize $\beta^{(0)}=0$, and repeat for
$k=1,2,3,\ldots$, 
\begin{gather}
\label{eq:fsrup}
\beta^{(k)} = \beta^{(k-1)} + \epsilon \cdot 
\sign\big(X_i^T (y-X\beta^{(k-1)})\big) \cdot e_i, \\
\label{eq:fsrdir}
\text{where}\;\, i \in 
\argmax_{j=1,\ldots p} \, |X_j^T(y-X\beta^{(k-1)})|.
\end{gather}
\end{algorithm}

In the above, $\epsilon>0$ is a small fixed constant (e.g.,
$\epsilon=0.01$), commonly referred to as the step size or learning
rate; $e_i$ denotes the $i$th standard basis vector in $\R^p$; and 
the element notation in \eqref{eq:fsrdir} emphasizes that the
maximizing index $i$ need not be unique.  
The basic idea behind the 
forward stagewise updates \eqref{eq:fsrup}, \eqref{eq:fsrdir} is
highly intuitive: at each iteration we greedily select the variable
$i$ that has the largest absolute inner product (or correlation, for    
standardized variables) with the residual, and we add $s_i \epsilon$
to its coefficient, where $s_i$ is the sign of this inner product.
Accordingly, the fitted values undergo the update:
\begin{equation*}
X\beta^{(k)} = X\beta^{(k-1)} + \epsilon \cdot s_i X_i.
\end{equation*}
Such greediness, in selecting variable $i$, is counterbalanced by the
small step size $\epsilon>0$; instead of increasing the coefficient of
$X_i$ by a (possibly) large amount in the fitted model, forward
stagewise only increases it by $\epsilon$,
which ``slows down'' the learning process.  As a result, it typically
requires many iterations to produce estimates of reasonable interest
with forward stagewise regression, e.g., it could easily take
thousands of iterations to reach a model with only tens of active
variables (we use ``active'' here to refer to variables that are
assigned nonzero coefficients). See the left panel of Figure
\ref{fig:lasso} for a small example.    

\begin{figure}[htb]
\centering
\includegraphics[width=0.475\textwidth]{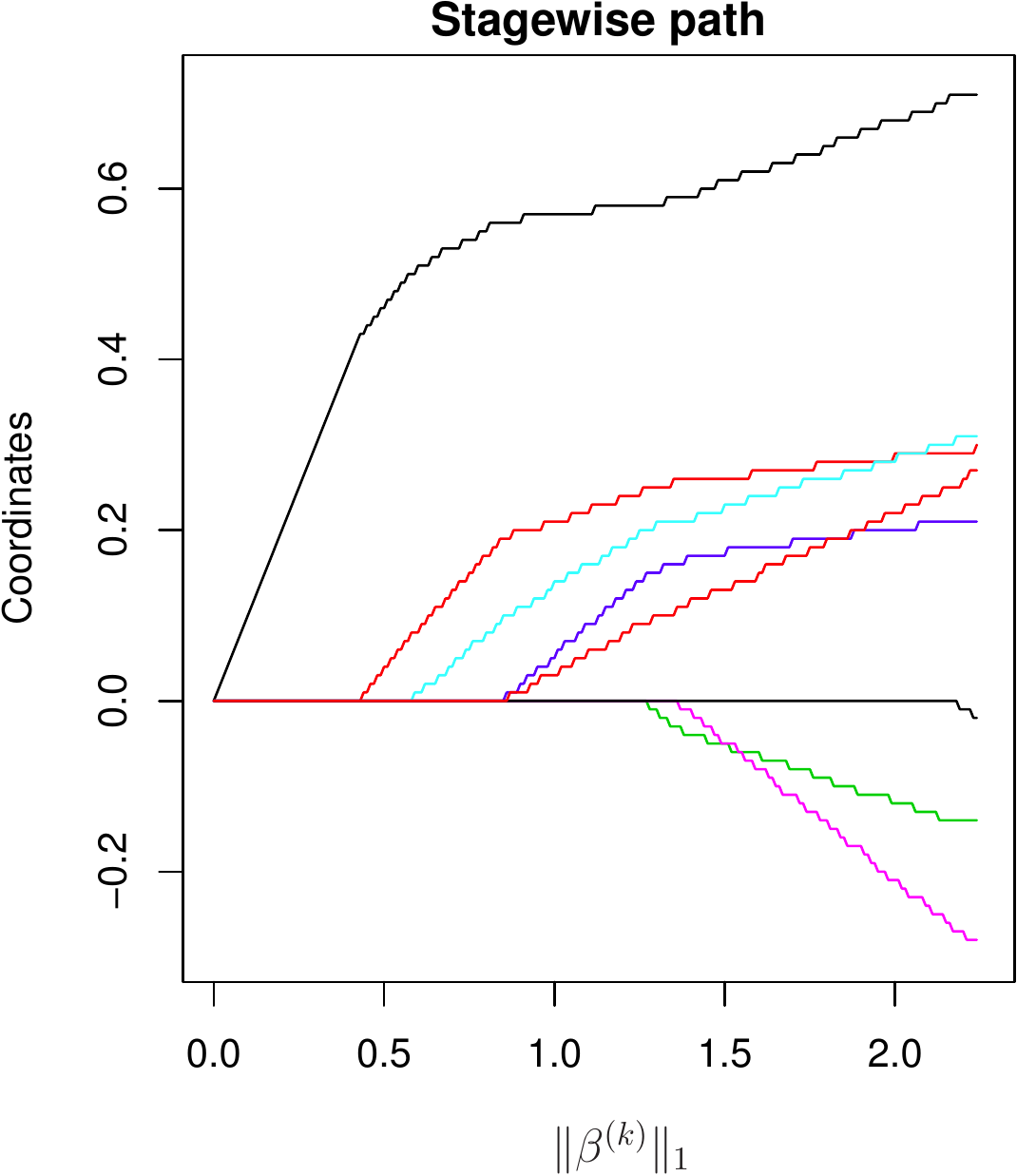}
\includegraphics[width=0.475\textwidth]{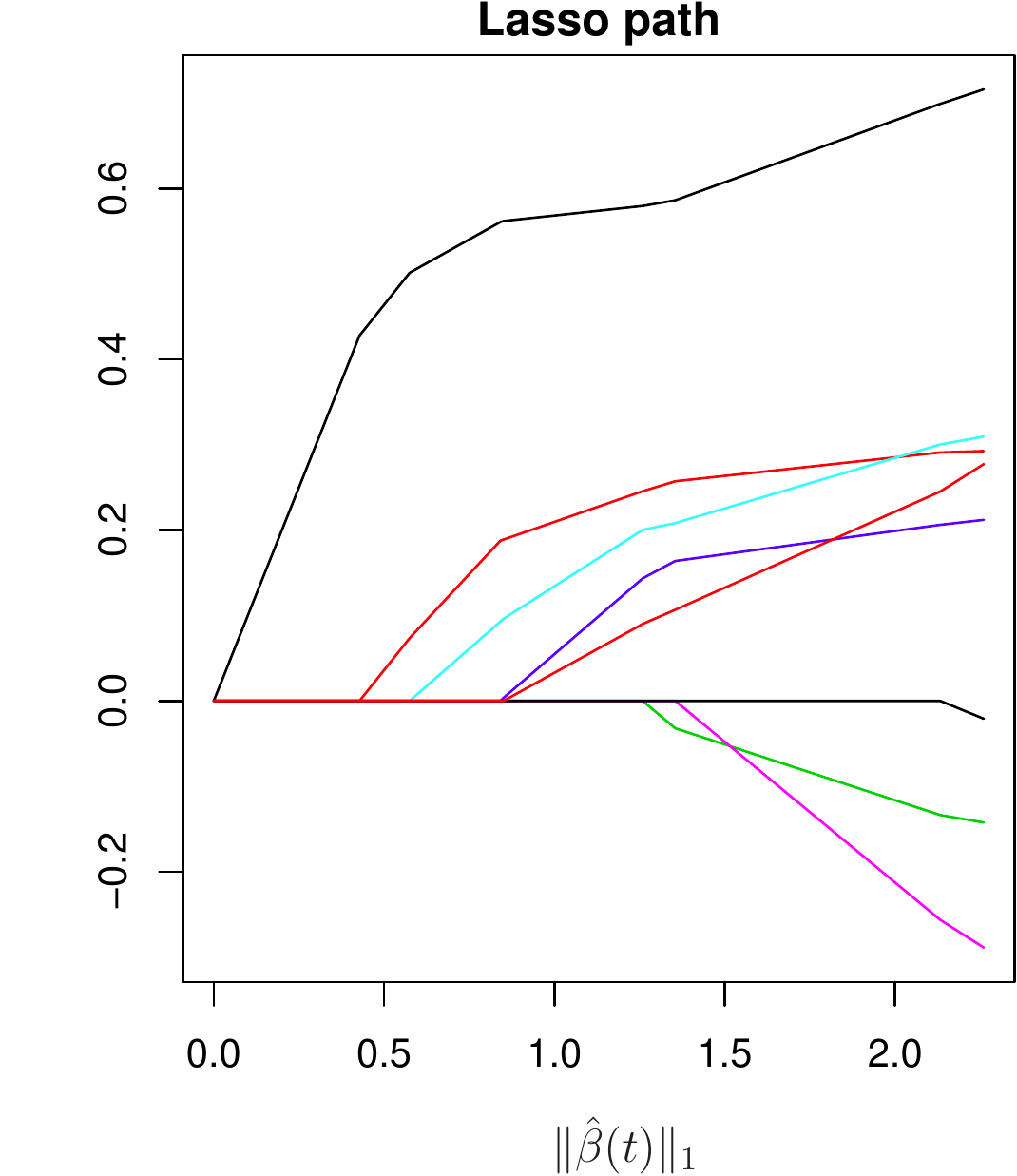} \\
\caption{\it A simple example using the prostate cancer data from  
  \citet{esl}, where the log PSA score of $n=67$ men with prostate  
  cancer is modeled as a linear function of $p=8$ biological
  predictors. The left panel shows the
  forward stagewise regression estimates $\beta^{(k)} \in \R^8$,
  $k=1,2,3,\ldots$, with the 8 coordinates plotted in 
  different colors.  The stagewise algorithm was run with
  $\epsilon=0.01$ for 250 iterations, and the x-axis here gives the
  $\ell_1$ norm of the estimates across iterations.  The right panel
  shows the lasso solution path, also parametrized by the
  $\ell_1$ norm of the estimate.  The similarity between the stagewise
  and lasso paths is visually striking; for small enough $\epsilon$,
  they appear identical.  This is not a coincidence and
  has been rigorously studied by \citet{lars}, and other authors; in
  Section \ref{sec:lasso} we provide an intuitive explanation
  for this phenomenon.}
\label{fig:lasso}
\end{figure}

This ``slow learning'' property is a key difference between
forward stagewise regression and the closely-named {\it forward
stepwise regression} procedure: at each iteration, the latter algorithm
chooses a variable in a similar manner to that in
\eqref{eq:fsrdir}\footnote{If $A$ denotes the active set at the end of
iteration $k-1$, then at iteration $k$ forward stepwise chooses the
variable $i$ such that the sum of squared errors from regressing $y$
onto the variables in $A \cup \{i\}$ is smallest.  This is
equivalent to choosing $i$ such that 
\smash{$|\widetilde{X}_i^T (y-X\beta^{(k-1)})|$} is largest, where 
$\beta^{(k-1)}$ denote the coefficients from regressing $y$ on the
variables in $A$, and \smash{$\widetilde{X}_i$} is the residual from
regressing $X_i$ on the variables in $A$.}, but once it does so, it
updates the fitted model by regressing $y$ on all variables
selected thus far.  While both are greedy algorithms, the stepwise
procedure is much greedier; after $k$ iterations, it produces a
model with exactly $k$ active variables.  Forward stagewise and
forward stepwise are old techniques (some classic references for
stepwise regression methods are \citet{efroy} and \citet{drapersmith}, 
but there could have been
earlier relevant work).  According to \citet{esl}, forward stagewise
was historically dismissed by statisticians as being  
``inefficient'' and hence less useful than methods like forward
or backward stepwise.  This is perhaps understandable, 
if we keep in mind the limited computational resources of
the time.  From a modern perspective, however, we now appreciate
that ``slow learning'' is a form of regularization and can present
considerable benefits in terms of the generalization error of the
fitted models---this is seen not only in regression, but across
variety of settings.  
Furthermore, by modern standards, forward stagewise is
computationally cheap: to trace out a path of regularized estimates,
we repeat very simple iterations, each one requiring (at most) $p$
inner products, computations that could be trivially parallelized.

The revival of interest in stagewise regression 
began with the work of \citet{lars}, where the authors derived a
surprising connection between the sequence of forward stagewise 
estimates and the solution path of the {\it lasso} \citep{lasso}, 
\begin{equation}
\label{eq:lasso}
\hbeta(t) = \argmin_{\beta\in\R^p} \,\half\|y-X\beta\|_2^2 
\;\,\st\;\, \|\beta\|_1 \leq t,
\end{equation}
over the regularization parameter $t \geq 0$.  The relationship
between stagewise and the lasso will be reviewed in Section
\ref{sec:lasso} in detail, but the two
panels in Figure \ref{fig:lasso} tell the essence of the story.  The 
stagewise paths, on the left, appear to be jagged versions of their  
lasso counterparts, on the right.  Indeed, as the step size $\epsilon$
is made smaller, this jaggedness becomes less noticeable, and
eventually the two sets of paths appear exactly the same.  
This is not a coincidence, and under some conditions (on the problem 
instance in consideration), it is known that the stagewise path
converges to the lasso path, as $\epsilon \rightarrow 0$.
Interestingly, when these conditions do not hold, stagewise estimates
can deviate substantially from lasso solutions, and yet in
such situations the former estimates can still perform competitively
with the latter, say, in terms of test error (or really any other
standard error metric). This is an important point, and it supports
the use of stagewise regression as a general tool for regularized
estimation.    

\subsection{Summary of our contributions}
\label{sec:summary}

This paper departs from the lasso setting and considers the generic
convex problem   
\begin{equation}
\label{eq:genprob}
\hx(t) \in \argmin_{x\in\R^n} \, f(x) 
\;\,\st\;\, g(x) \leq t,
\end{equation}
where $f,g : \R^n \rightarrow \R$ are convex functions, and $f$ is
differentiable.
Motivated by forward stagewise regression and its connection to the
lasso, our main contribution is the following 
{\it general stagewise algorithm} for producing an approximate
solution path of \eqref{eq:genprob}, as the regularization parameter
$t$ varies over $[t_0,\infty)$. 

\begin{algorithm}[\textbf{General stagewise procedure}]
\label{alg:genstage}
\hfill\par
\smallskip\smallskip
\noindent
Fix $\epsilon>0$ and $t_0\in\R$. Initialize $x^{(0)}=\hx(t_0)$, a solution
in \eqref{eq:genprob} at $t=t_0$. Repeat, for $k=1,2,3,\ldots$,
\begin{gather}
\label{eq:stageup}
x^{(k)} = x^{(k-1)} + \Delta, \\ 
\label{eq:stagedir}
\text{where}\;\,
\Delta \in \argmin_{z\in\R^n} \,\,
\langle \nabla f(x^{(k-1)}), z \rangle
\;\,\st\;\, g(z) \leq \epsilon.
\end{gather}
\end{algorithm}

The intuition behind the general stagewise algorithm can be seen right
away: at each iteration, we update the current iterate in a direction
that minimizes the inner product with the gradient of $f$ (evaluated
at the current iterate), but simultaneously restrict this direction to
be small under $g$.  By applying these updates repeatedly, we
implicitly adjust the trade-off  between minimizing $f$ and
$g$, and hence one can imagine  that the $k$th iterate $x^{(k)}$
approximately solves \eqref{eq:genprob} with $t=g(x^{(k)})$. In Figure
\ref{fig:examples}, we show a few simple examples of the general
stagewise paths implemented for various different choices of loss
functions $f$ and regularizing functions $g$.  

\begin{figure}[htbp]
\vspace{-10pt}
\centering
\includegraphics[width=0.35\textwidth]{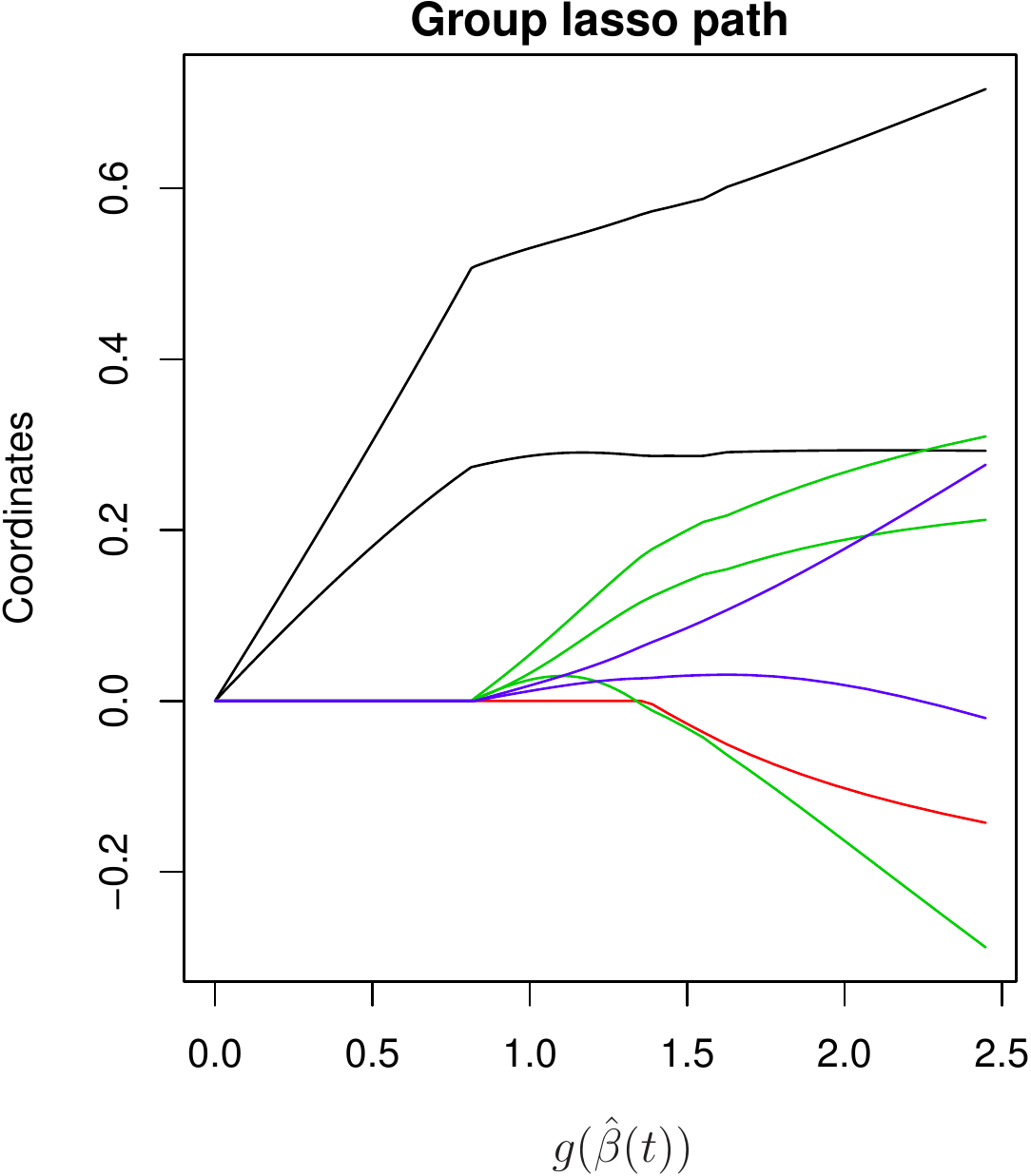} 
\includegraphics[width=0.35\textwidth]{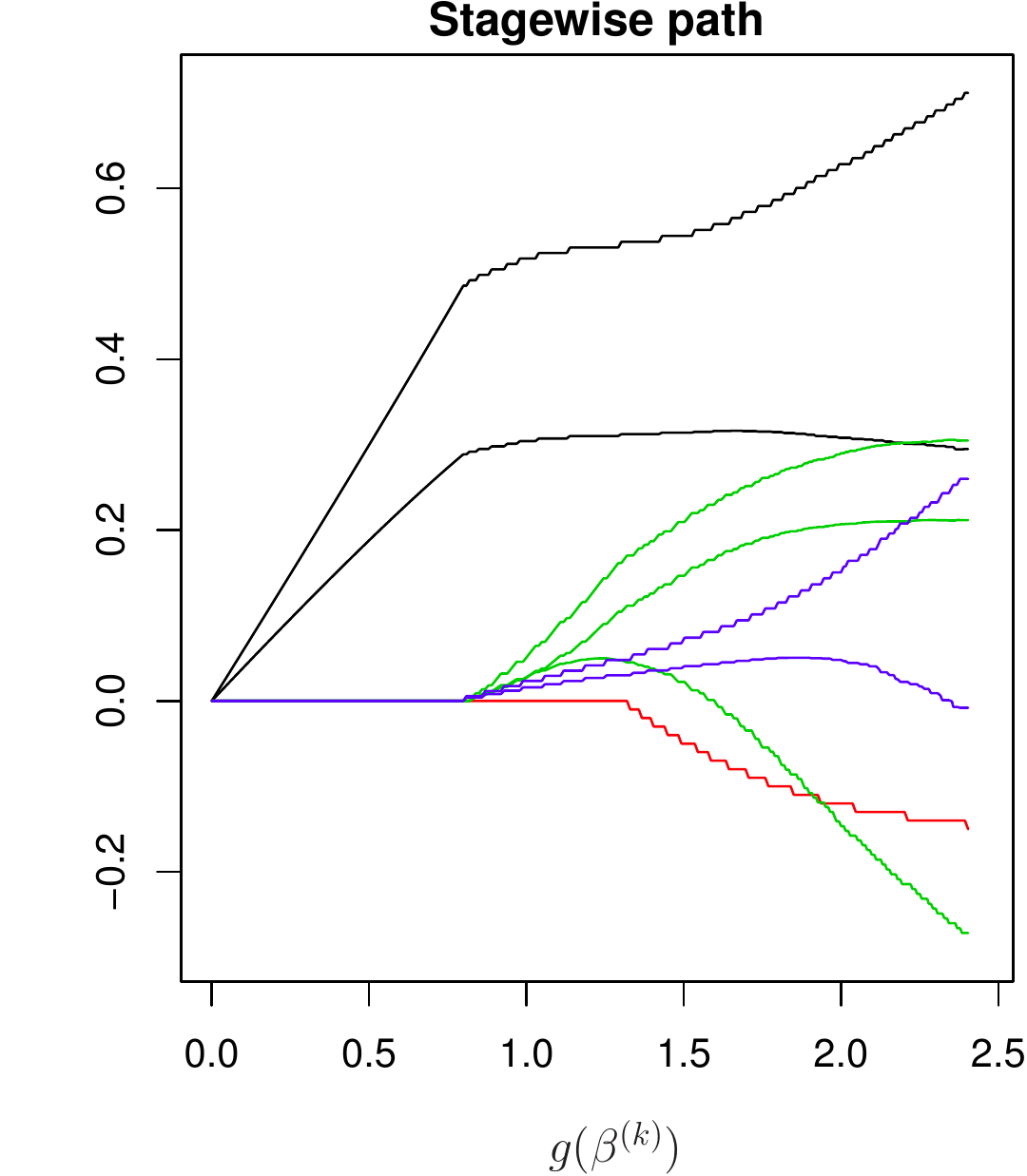} 
\\ \hspace{0pt} \\
\includegraphics[width=0.35\textwidth]{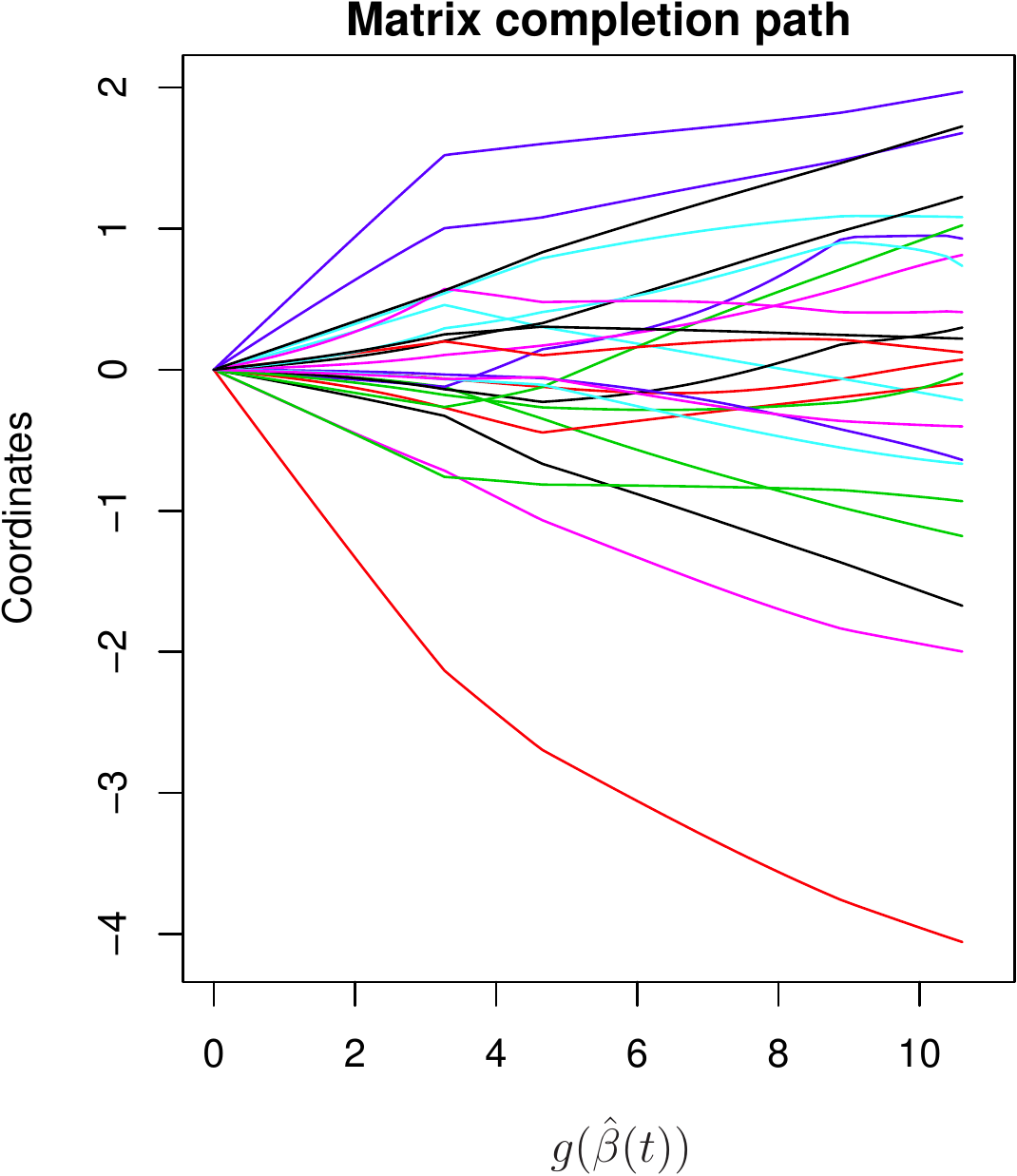} 
\includegraphics[width=0.35\textwidth]{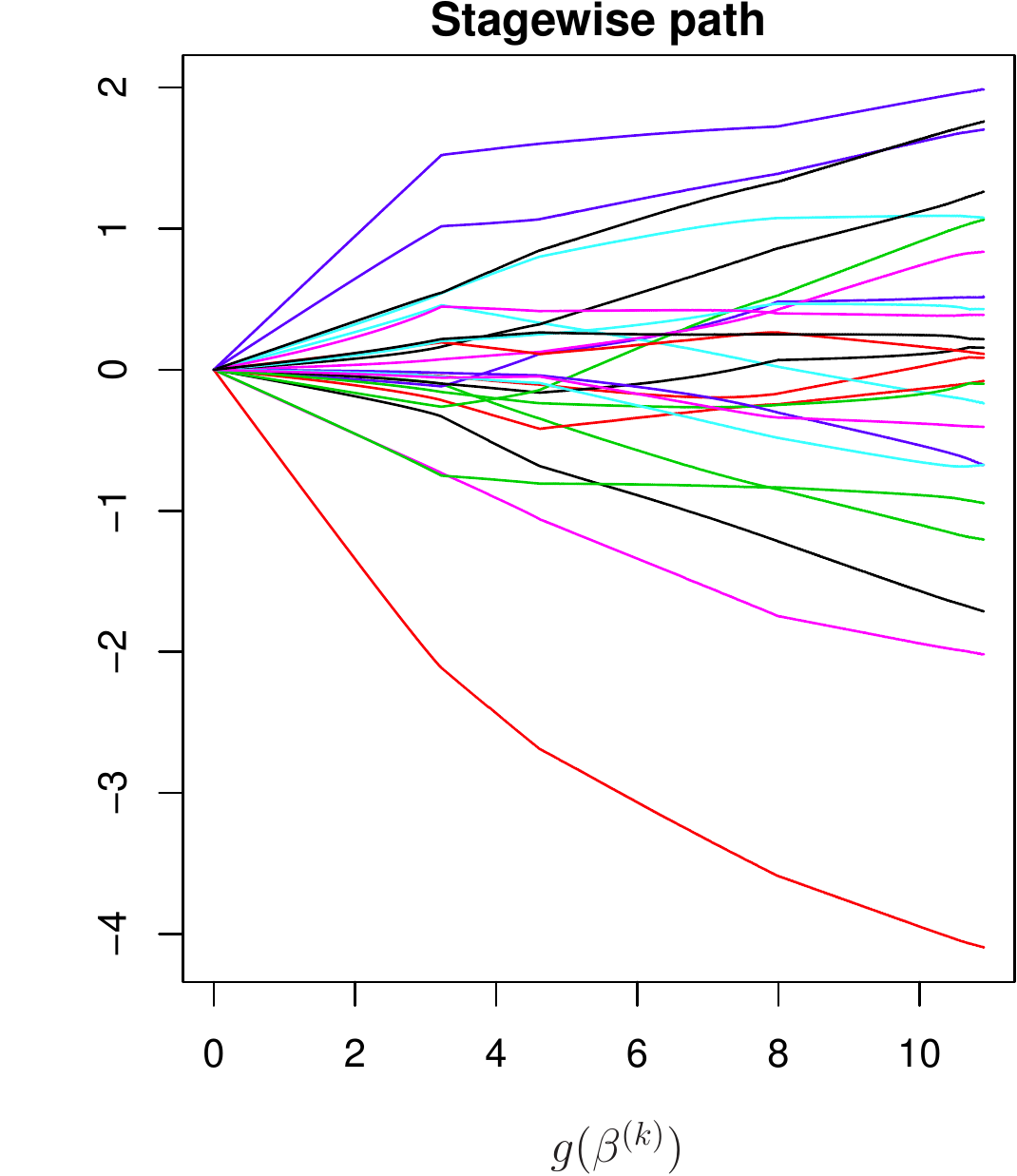} 
\\ \hspace{0pt} \\
\includegraphics[width=0.35\textwidth]{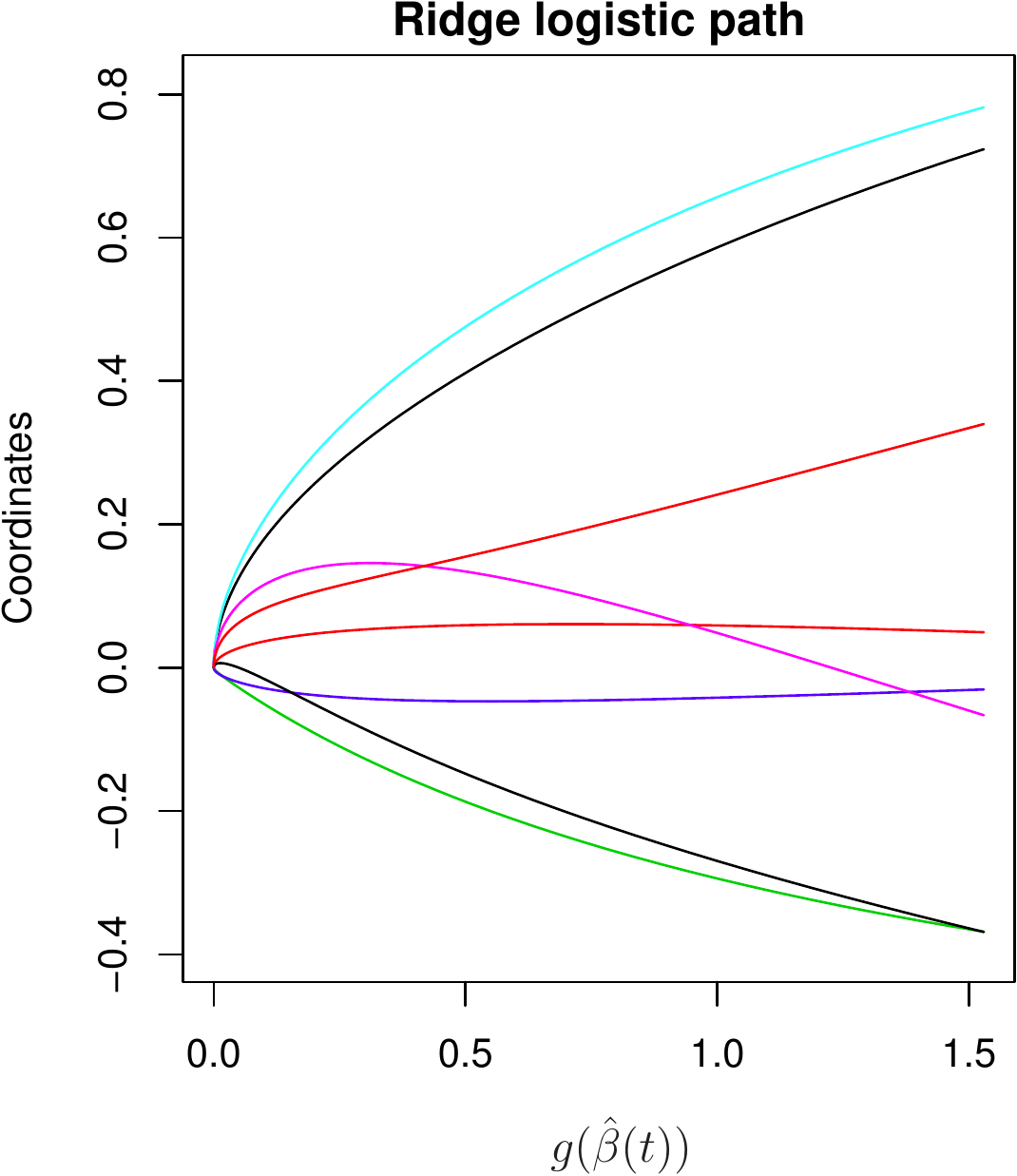} 
\includegraphics[width=0.35\textwidth]{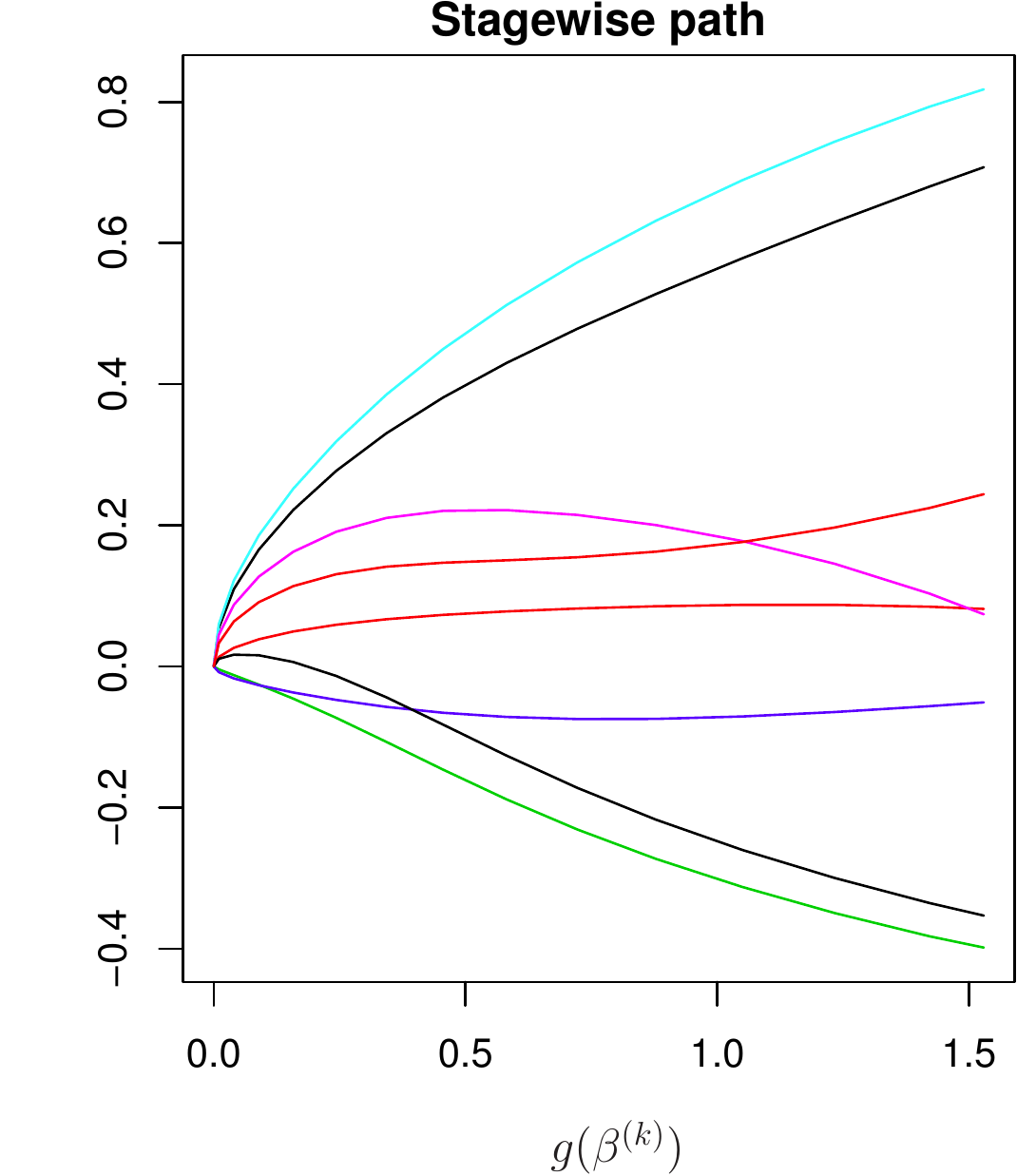} 
\caption{\it Examples comparing the actual solution paths (left 
  column) to the stagewise paths (right column) across various problem
  contexts, using the prostate cancer data set.  The first row
  considers a group lasso model on the prostate data (where the groups
  were somewhat arbitrarily chosen based on the predictor types); the
  second row considers a matrix completion task, on a partially
  observed submatrix of the full predictor matrix; the third row
  considers a logistic regression model with ridge regularization
  (the outcome being the indicator of log PSA $> 1$).  In each
  case, the stagewise estimates were very easy to compute; Sections
  \ref{sec:groupreg}, \ref{sec:tracereg}, and \ref{sec:quadreg}
  discuss these problem settings in detail.}   
\label{fig:examples}
\end{figure}

In the next section, we develop further intuition
and motivation for the general stagewise procedure, and we tie in 
forward stagewise regression as a special case.  The rest of this
article is then dedicated to the implementation and analysis of
stagewise algorithms: Section \ref{sec:applications} derives the
specific form of the stagewise updates \eqref{eq:stageup},
\eqref{eq:stagedir} for various problem setups, Section    
\ref{sec:bigexamples} conducts large-scale empirical evaluations of 
stagewise estimates, Section \ref{sec:subopt} presents some 
theory on suboptimality, and Section \ref{sec:discussion} concludes
with a discussion.  

Throughout, our arguments and examples are centered around 
three points, summarized below.      

\begin{enumerate}
\item {\it Simple, fast estimation procedures.}
  The general framework for stagewise estimation in
  Algorithm \ref{alg:genstage} leads to simple and efficient stagewise  
  procedures for group-structured regularization problems (e.g., the
  group lasso, multitask learning), trace norm regularization
  problems (e.g., matrix completion), quadratic regularization problem
  problems (e.g., nonparametric smoothing), and (some) generalized
  lasso problems (e.g., image denoising).  For such problems, the
  proposed stagewise procedures are often competitive with existing 
  commonly-used algorithms in terms of efficiency, and are generally
  much simpler.       

\item {\it Similar to actual solution paths, but more stable.}
  In many examples, the computed stagewise path is highly 
  similar to the actual solution path of the corresponding convex
  regularization problem in \eqref{eq:genprob}---typically, this
  happens when the components of the actual solution change 
  ``slowly'' with the regularization parameter $t$.  In many others,
  even though it shares gross characteristics of the actual solution 
  path, the stagewise path is different---typically, this
  happens when the components of the actual solution change
  ``rapidly'' with $t$, and the stagewise component paths are much
  more stable.     

\item {\it Competitive statistical performance.}  Across essentially
  all cases, even those in which its constructed path is not close to
  the actual solution path, the stagewise algorithm performs
  favorably from a statistical point of view.  That is,
  stagewise estimates are comparable to solutions in
  \eqref{eq:genprob} with respect to relevant error metrics, across  
  various problem settings.  This suggests that 
  stagewise estimates deserved to be studied on their own, regardless
  of their proximity to solutions in \eqref{eq:genprob}. 
\end{enumerate}

The third point above, on the favorable statistical properties of
stagewise estimates, is based on empirical arguments, rather than    
theoretical ones. Statistical theory for stagewise estimates is
an important topic for future work.

\section{Properties of the general stagewise framework}
\label{sec:properties}

\subsection{Motivation: stagewise regression and the lasso}
\label{sec:lasso}

The lasso estimator is a popular tool for sparse estimation in the
regression setting.  Displayed in \eqref{eq:lasso}, we assume for 
simplicity that the lasso solution 
\smash{$\hbeta(t)$} in \eqref{eq:lasso} is unique, which holds under
very weak conditions on $X$.\footnote{For example, it suffices to
  assume that $X$ has columns in general position, see
  \citet{lassounique}.  Note that here we are only claiming uniqueness
  for all parameter values $t < t^*$, where $t^*$ is the smallest
  $\ell_1$ norm of a least squares solution of $y$ on $X$.}  
Recall that the parameter $t$ controls the level of sparsity in the
estimate \smash{$\hbeta(t)$}: when $t=0$, we have
\smash{$\hbeta(0)=0$}, and as $t$ increases, select components of
\smash{$\hbeta(t)$} become nonzero, corresponding to 
variables entering the lasso model (nonzero components of
\smash{$\hbeta(t)$} can also become zero, corresponding to variables
leaving the model).  The solution path \smash{$\hbeta(t)$, $t\in
  [0,\infty)$} is continuous and piecewise linear as a function of
$t$, and for a large enough value of $t$, the path culminates in a
least squares estimate of $y$ on $X$.  

The right panel of Figure \ref{fig:lasso} shows an example of the
lasso path, which, as we discussed earlier, appears quite similar 
to the stagewise path on the left.
This is explained by the seminal work of \citet{lars}, who describe
two algorithms  
(actually three, but the third is unimportant for our purposes): one
for explicitly constructing the lasso path \smash{$\hbeta(t)$} as a
continuous, piecewise linear function of the regularization
parameter $t \in [0,\infty)$, and another for computing the limiting
stagewise regression paths as $\epsilon \rightarrow 0$.
One of the (many) consequences of their work is the following: if
each component of the lasso solution path \smash{$\hbeta(t)$} is a 
monotone function of $t$, then these two algorithms coincide, and
therefore so do the stagewise and lasso paths (in the limit as
$\epsilon \rightarrow 0$).  Note that the lasso paths for the
data example in Figure \ref{fig:lasso} are indeed monotone, and hence
the theory confirms the observed convergence of stagewise and
lasso estimates in this example.  

The lasso has undergone intense study as a regularized
regression estimator, and its statistical properties (e.g., its
generalization error, or its ability to detect a truly relevant set
of variables) are more or less well-understood at this point.  Many of
these properties cast the lasso in a favorable light.  Therefore, the
equivalence between the (limiting) stagewise and lasso paths
lends credibility to forward stagewise as a regularized
regression procedure: for a small step size $\epsilon$, we know that
the forward stagewise estimates will be close
to lasso estimates, at least when the individual coordinate paths are
monotone.  At a high level, it is actually somewhat remarkable that
such a simple algorithm, Algorithm \ref{alg:fsr}, can produce
estimates that can stand alongside those defined by the (relatively)
sophisticated optimization problem in \eqref{eq:lasso}.
There are now several interesting points to raise.

\begin{itemize}
\item {\it The nonmonotone case.}  In practice, the components of the
  lasso path are rarely monotone.  How do the stagewise and
  lasso paths compare in such cases?  A precise theoretical answer is
  not known, but empirically, these paths can be quite
  different. In particular, for problems in which the predictors
  $X_1,\ldots X_p$ are correlated, the lasso coordinate
  paths can be very wiggly (as variables can enter and leave the model
  repeatedly), while the stagewise paths are often very stable; see,
  e.g., \citet{monotonelasso}.  In support of these empirical
  findings, the latter authors
  derived a local characterization of the lasso and forward stagewise
  paths: they show that at any point along the path, 
  the lasso estimate decreases the sum of squares loss function at an 
  optimal rate with respect to the increase in $\ell_1$ norm, and
  the (limiting) forward stagewise estimate decreases the loss
  function at an optimal rate with respect to the increase in $\ell_1$
  arc length.  Loosely speaking, since the $\ell_1$ arc length
  accounts for the entire history of the path up until the current
  point, the (limiting) stagewise algorithm is less ``willing'' to
  produce wiggly estimates. 

  Despite these differences, stagewise estimates tend to perform
  competitively with lasso estimates in terms of test error, and this
  is true even with highly correlated predictor variables, when
  the stagewise and lasso paths are very different (such statements are
  based on simulations, and not theory; see
  \citet{monotonelasso,knudsenthesis}). 
  This is a critical point, as it suggests that stagewise should be 
  considered as an effective tool for regularized estimation, apart
  from any link to a convex problem. We return to this idea throughout
  the paper.      

\item {\it General convex loss functions.}  Fortunately, the stagewise
  method extends naturally to sparse modeling in other settings,
  beyond Gaussian regression.  Let $f : \R^p \rightarrow \R$ be a
  differentiable convex loss function, e.g.,
  \smash{$f(\beta)=\half\|y-X\beta\|_2^2$} 
  for the regression setting.  Beginning again with $\beta^{(0)}=0$,
  the analogy of the stagewise steps in \eqref{eq:fsrup},
  \eqref{eq:fsrdir} for the present general setting are
\begin{gather}
\label{eq:fsup}
\beta^{(k)} = \beta^{(k-1)} - \epsilon \cdot 
\sign\big(\nabla_i f(\beta^{(k-1)}) \big) \cdot e_i, \\
\label{eq:fsdir}
\text{where}\;\, i \in 
\argmax_{j=1,\ldots p} \, |\nabla_j f (\beta^{(k-1)})|.
\end{gather}
That is, at each iteration we update $\beta^{(k)}$ in the direction
opposite to the largest component of the gradient (largest in absolute 
value).  Note that this reduces to the usual update rules 
\eqref{eq:fsrup}, \eqref{eq:fsrdir} when \smash{$f(\beta)=\half
\|y-X\beta\|_2^2$.}  \citet{boostpath} studied the stagewise routine
\eqref{eq:fsup}, \eqref{eq:fsdir}, and its connection to the
$\ell_1$-constrained estimate
\begin{equation}
\label{eq:l1reg}
\hbeta(t) = \argmin_{\beta \in \R^p} \, f(\beta) \;\,\st\;\,
\|\beta\|_1 \leq t.
\end{equation}
Similar to the result for lasso regression, these authors prove that
if the solution \smash{$\hbeta(t)$} in \eqref{eq:l1reg} has monotone
coordinate paths, then under mild 
conditions\footnote{Essentially, \citet{boostpath} assume
that conditions on $f$ that imply a unique solution in
\eqref{eq:l1reg}, and allow for a second order Taylor expansion
of $f$.  Such conditions are that $f(\beta) = h(X\beta)$,
with $h$ twice differentiable and strictly convex, and $X$ having
columns in general position.} on $f$, the stagewise paths given by 
\eqref{eq:fsup}, \eqref{eq:fsdir} converge to the path 
\smash{$\hbeta(t)$} as 
$\epsilon \rightarrow 0$.  This covers, e.g., the 
cases of logistic regression and Poisson regression losses, with
predictor variables $X$ in general position.  The same general
message, as in the linear regression setting, applies here: compared 
to the relatively complex optimization problem \eqref{eq:l1reg},
the stagewise algorithm \eqref{eq:fsup}, \eqref{eq:fsdir} is 
very simple.  The most (or really, the only) advanced part of each  
iteration is the computation of the gradient $\nabla
f(\beta^{(k-1)})$; in the logistic or Poisson regression settings, the
components of $\nabla f(\beta^{(k-1)})$ are given by  
\begin{equation*}
\nabla_j f(\beta^{(k-1)}) = X_j^T \big(y-\mu(\beta^{(k-1)})\big),
\;\;\; j=1,\ldots p,  
\end{equation*}
where $y \in \R^n$ is the outcome and $\mu(\beta^{(k-1)}) \in \R^n$
has components 
\begin{equation*}
\mu_i(\beta^{(k-1)}) = \begin{cases}
1/[1+\exp(-(X\beta^{(k-1)})_i)] & \text{for logistic regression} \\
\exp((X\beta^{(k-1)})_i) & \text{for Poisson regression}
\end{cases},
\;\;\; i=1,\ldots n.
\end{equation*}
Its precise connection to the $\ell_1$-constrained optimization
problem \eqref{eq:l1reg} for monotone paths is
encouraging, but even outside of this case, the simple and efficient
stagewise algorithm \eqref{eq:fsup}, \eqref{eq:fsdir} produces
regularized estimates deserving of attention in their own right.  

\item {\it Forward-backward stagewise.} \citet{stagewiselasso}
  examined a novel modification of forward stagewise, under  
  a general loss function $f$: at each iteration, their
  proposal takes a backward step (i.e., moves a component
  of $\beta^{(k)}$ towards zero) if this would decrease the loss
  function by a sufficient amount $\xi$; otherwise it takes a  
  forward step as usual.  The authors prove that, as long as the 
  parameter $\xi$ used for the backward steps scales as
  $\xi=o(\epsilon)$, the path from this forward-backward stagewise
  algorithm converges to the solution 
  path in \eqref{eq:l1reg} as $\epsilon \rightarrow 0$.  The
  important distinction here is that their result does not assume
  monotonicity of the coordinate paths in \eqref{eq:l1reg}.  
  (It does, however, assume that the loss 
  function $f$ is strongly convex---in the linear regression setting, 
  \smash{$f(\beta)=\half \|y-X\beta\|_2^2$}, this is equivalent to
  assuming that $X \in \R^{n\times p}$ has linearly independent  
  predictors, which
  requires $n \geq p$).\footnote{It is also worth pointing out that
    the type of convergence considered by \citet{stagewiselasso} is  
    stronger than that considered by \citet{lars} and
    \citet{boostpath}.  The former authors prove that, under suitable
    conditions, the entire stagewise path converges globally to the
    lasso solution path; the latter authors only prove a local type of
    convergence, that has to do with the limiting stagewise and lasso
    directions at any fixed point along the path.}  The
  forward-backward stagewise algorithm hence provides another
  way to view the connection between (the usual) forward stagewise
  steps \eqref{eq:fsup}, \eqref{eq:fsdir} and the
  $\ell_1$-regularized optimization problem \eqref{eq:l1reg}: 
  the forward stagewise path is an approximation to the solution path
  in \eqref{eq:l1reg} given by skipping the requisite backward
  steps needed to correct for nonmonotonicities.
\end{itemize}

Clearly, there has been some fairly extensive work connecting the
stagewise estimates \eqref{eq:fsrup}, \eqref{eq:fsrdir} and the lasso
estimate \eqref{eq:lasso}, or more generally, the stagewise estimates 
\eqref{eq:fsup}, \eqref{eq:fsdir} and the $\ell_1$-constrained
estimate \eqref{eq:l1reg}.  Still, however, this connection seems
mysterious.  Both methods produce a regularization path, with a fully
sparse model on one end, and a fully dense model on the other---but 
beyond this basic degree of similarity, why should we expect the
stagewise path \eqref{eq:fsup}, \eqref{eq:fsdir} and the
$\ell_1$ regularization path \eqref{eq:l1reg} to be so closely
related?   
The work referenced above gives a mathematical treatment of this
question, and we feel, does not provide much intuition.  
In fact, there is a simple interpretation of the forward stagewise 
algorithm that explains its connection to the lasso problem, seen
next. 

\subsection{A new perspective on forward stagewise regression} 

We start by rewriting the steps \eqref{eq:fsup}, \eqref{eq:fsdir}
for the stagewise algorithm, under a general loss $f$, as
\begin{gather*}
\beta^{(k)}=\beta^{(k-1)}+\Delta, \\
\text{where} \;\,
\Delta = - \epsilon\cdot \sign\big(\nabla_i f(\beta^{(k-1)})\big)
\cdot e_i , \\
\text{and} \;\,
|\nabla_i f(\beta^{(k-1)})| = \|\nabla f(\beta^{(k-1)})\|_\infty.
\end{gather*}
As $\nabla_i f(\beta^{(k-1)})$ is maximal in absolute value among all
components of the gradient, the quantity $\sign(\nabla_i
f(\beta^{(k-1)}))\cdot e_i$ is a  subgradient of the $\ell_\infty$
norm evaluated at $\nabla f(\beta^{(k-1)})$: 
\begin{equation*}
\Delta \in -\epsilon \cdot \Big(\partial \|x\|_\infty 
\Big|_{x=\nabla f(\beta^{(k-1)})}\Big).
\end{equation*}
Using the duality between the $\ell_\infty$ and $\ell_1$ norms, 
\begin{equation*}
\Delta \in -\epsilon \cdot \Big(
\argmax_{z \in \R^p} \,\,
\langle \nabla f(\beta^{(k-1)}), z \rangle
\;\,\st\;\, \|z\|_1 \leq 1 \Big),
\end{equation*}
or equivalently,
\begin{equation*}
\Delta \in 
\argmin_{z \in \R^p} \,\,
\langle \nabla f(\beta^{(k-1)}), z \rangle
\;\,\st\;\, \|z\|_1 \leq \epsilon.
\end{equation*}
(Above, as before, the element notation emphasizes that the
maximizer or minimizer is not necessarily unique.)  
Hence the forward stagewise steps \eqref{eq:fsup}, \eqref{eq:fsdir}
satisfy 
\begin{gather}
\label{eq:fsup2}
\beta^{(k)} = \beta^{(k-1)} + \Delta, \\
\label{eq:fsdir2}
\text{where}\;\,
\Delta \in \argmin_{z\in\R^p} \,\,
\langle \nabla f(\beta^{(k-1)}), z \rangle
\;\,\st\;\, \|z\|_1 \leq \epsilon.
\end{gather}
Written in this form, the stagewise algorithm exhibits a
natural connection to the $\ell_1$-regularized optimization problem
\eqref{eq:l1reg}. At each iteration, forward
stagewise moves in a direction that minimizes the inner product with  
the gradient of $f$, among all directions
constrained to have a small $\ell_1$ norm; therefore, the sequence of 
stagewise estimates balance (small) decreases in the loss
function $f$ with (small) increases in the $\ell_1$ norm, just like
the solution path in \eqref{eq:l1reg}, as the regularization parameter
$t$ increases.  This intuitive perspective aside, the
representation \eqref{eq:fsup2}, \eqref{eq:fsdir2} for the forward
stagewise estimates is important because it inspires an analogous
approach for general convex regularization problems.  This was
already presented in Algorithm \ref{alg:genstage}, and next we discuss
it further.


\subsection{Basic properties of the general stagewise procedure}   
\label{sec:basicprops}

Recall the general minimization problem in \eqref{eq:genprob},
where we assume that the loss function $f$ is convex and
differentiable, and the regularizer $g$ is convex.  It can now be seen
that the steps \eqref{eq:stageup}, \eqref{eq:stagedir} in
the general stagewise procedure in Algorithm \ref{alg:genstage} are
directly motivated by the forward stagewise steps, as expressed in 
\eqref{eq:fsup2}, \eqref{eq:fsdir2}.  The explanation is similar to
that given above: as we repeat the steps of the algorithm, the
iterates are constructed to decrease the loss function $f$ (by
following its negative gradient) at the cost of a small increase in
the regularizer $g$. In this sense, the stagewise algorithm navigates
the trade-off between minimizing $f$ and $g$, and produces an
approximate regularization path for \eqref{eq:genprob}, i.e., the
$k$th iterate $x^{(k)}$ approximately solves problem
\eqref{eq:genprob} with $t=g(x^{(k)})$.    

From our work at the end of the last subsection, it is clear that 
forward stagewise regression \eqref{eq:fsup},
\eqref{eq:fsdir}, or equivalently \eqref{eq:fsup2}, \eqref{eq:fsdir2},
is a special case of the general stagewise procedure,
applied to the $\ell_1$-regularized problem \eqref{eq:l1reg}.
Moreover, the general stagewise procedure can be applied in many
other settings, well beyond $\ell_1$ regularization, as we show in the 
next section. Before presenting these applications, we now make
several basic remarks.

\begin{itemize}
\item {\it Initialization and termination.}
In many cases, initializing the algorithm is easy:
if $g(x)=0$ implies $x=0$ (e.g., this is true when $g$ is a
norm), then we can start the stagewise procedure at $t_0=0$ and
$x^{(0)}=0$.  In terms of a stopping criterion, a general strategy
for (approximately) tracing a full solution path is to stop the
algorithm when $g(x^{(k)})$ does not change very much between
successive 
iterations. If instead the algorithm has been terminated upon reaching
some maximum number of iterations or some maximum value of
$g(x^{(k)})$, and more iterations are desired, then the algorithm can
surely be restarted from the last reached iterate $x^{(k)}$.  

\item {\it First-order justification.}
If $g$ satisfies the triangle inequality (again, e.g., it would as a
norm), then the increase in the value of $g$ between successive
iterates is bounded by $\epsilon$:
\begin{equation*}
g(x^{(k)}) \leq g(x^{(k-1)}) + g(\Delta) \leq g(x^{(k-1)}) + \epsilon.
\end{equation*}
Furthermore, we can give a basic (and heuristic) justification of the
stagewise steps \eqref{eq:stageup}, \eqref{eq:stagedir}. Consider the
minimization problem \eqref{eq:genprob} at the parameter 
$t=g(x^{(k-1)})+\epsilon$; we can write this as
\begin{equation*}
\hx(t) \in \argmin_{x\in\R^n} \, f(x) - f(x^{(k-1)})
\;\,\st\;\, g(x) - g(x^{(k-1)}) \leq \epsilon,
\end{equation*}
and then reparametrize as
\begin{gather}
\label{eq:trueup}
\hx(t) = x^{(k-1)} + \Delta^*, \\
\label{eq:truedir}
\Delta^* \in \argmin_{z\in\R^n} 
\, f(x^{(k-1)}\hspace{-1pt}+\hspace{-1pt}z) 
- f(x^{(k-1)}) \;\,\st\;\, 
g(x^{(k-1)}\hspace{-1pt}+\hspace{-1pt}z) - g(x^{(k-1)}) \leq
\epsilon.
\end{gather}
We now modify the problem \eqref{eq:truedir} in two ways:
first, we replace the objective function in \eqref{eq:truedir} with
its first-order (linear) Taylor approximation around $x^{(k-1)}$,  
\begin{equation}
\label{eq:taylor}
\langle \nabla f(x^{(k-1)}), z \rangle
\approx f(x^{(k-1)}+z) - f(x^{(k-1)}),
\end{equation} 
and second, we shrink the constraint set in \eqref{eq:truedir} to
\begin{equation*}
\{z \in \R^n : g(z) \leq \epsilon \} \subseteq 
\{ z \in \R^n: g(x^{(k-1)}+z)-g(x^{(k-1)}) \leq \epsilon \},
\end{equation*}
since, as noted earlier, any element of the left-hand side above is an element of 
the right-hand side by the triangle inequality. These two modifications 
define a different update direction
\begin{equation*}
\Delta \in \argmin_{z\in\R^n} \, \langle \nabla f(x^{(k-1)}), z \rangle
\;\,\st\;\, g(z) \leq \epsilon,
\end{equation*}
which is exactly the direction \eqref{eq:stagedir} in the general
stagewise procedure. Hence the stagewise algorithm chooses $\Delta$
as above, rather than choosing the actual direction $\Delta^*$ in 
\eqref{eq:truedir}, to perform an update step from $x^{(k-1)}$. This
update results in a feasible point 
$x^{(k)}=x^{(k-1)}+\Delta$ for the problem \eqref{eq:genprob} at
$t=g^{(k-1)}+\epsilon$; 
of course, the point $x^{(k)}$ is not necessarily optimal, but as
$\epsilon$ gets smaller, the first-order Taylor approximation in
\eqref{eq:taylor} becomes tighter, so one would imagine that the point
$x^{(k)}$ becomes closer to optimal. 

\item {\it Dual update form.}
If $g$ is a norm, then the update direction defined in
\eqref{eq:stagedir} can be expressed more succinctly in terms of the 
dual norm $g^*(x) = \max_{g(z) \leq 1} x^T z$. We write 
\begin{align}
\nonumber
\Delta &\in -\epsilon\cdot \Big(\argmax_{z\in\R^n} \, \langle \nabla f(x^{(k-1)}), z \rangle
\;\,\st\;\, g(z) \leq 1\Big) \\
\label{eq:dualdir}
&= -\epsilon\cdot \partial g^*\big(\nabla f(x^{(k-1)})\big),
\end{align}
i.e., the direction $\Delta$ is $-\epsilon$ times a subgradient of the dual 
norm $g^*$ evaluated at $\nabla f(x^{(k-1)})$. 
This is a useful observation, since many norms admit a known dual norm
with known subgradients; we will see examples of this in the coming
section. 

\item {\it Invariance around $\nabla f$.}
The level of difficulty associated with computing the update
direction, i.e., in solving problem \eqref{eq:stagedir}, depends 
entirely on $g$ and not on $f$ at all (assuming that $\nabla f$ can be
readily computed). We can think of $\Delta$ as an operator on $\R^n$: 
\begin{equation}
\label{eq:opdir}
\Delta(x) \in \argmin_{z\in\R^n} \,\, \langle x,z \rangle
\;\,\st\;\, g(z) \leq \epsilon.
\end{equation}
This operator $\Delta(\cdot)$ is often called the 
{\it linear minimization oracle} associated with the function $g$, in
the optimization literature.  At each input $x$, it returns a
minimizer of the problem in \eqref{eq:opdir}.   Provided that
$\Delta(\cdot)$ can be 
expressed in closed-form---which is fortuitously the case for many
common statistical optimization problems, as we will see in the
sections that follow---the stagewise update step \eqref{eq:stageup}
simply evaluates 
this operator at $\nabla f(x^{(k-1)})$, and adds the result to $x^{(k-1)}$:
\begin{equation*}
x^{(k)} = x^{(k-1)} + \Delta\big(\nabla f(x^{(k-1)}) \big).
\end{equation*}
An analogy can be drawn here to the proximal operator 
in proximal gradient descent, used for minimizing the composite
function $f+g$, where $f$ is smooth but $g$ is (possibly)
nonsmooth. The proximal operator is defined entirely in terms of
$g$, and as long as it can be expressed analytically, the generalized
gradient update for $x^{(k)}$ simply uses the output of this operator
at $\nabla f(x^{(k-1)})$. 

\item {\it Unbounded stagewise steps.}  
Suppose that $g$ is a seminorm, i.e., it satisfies $g(ax)=|a|g(x)$
for $a \in \R$, and $g(x+y) \leq g(x) + g(y)$, but $g$ can have a
nontrivial null space, $N_g = \{x \in \R^n : g(x) = 0\}$. 
In this case, the stagewise update step in \eqref{eq:stageup} can be
unbounded; in particular, if  
\begin{equation}
\label{eq:nullg}
\langle \nabla f(x^{(k)}), z \rangle
\not= 0 \;\;\; \text{for some $z \in N_g$},
\end{equation}
then we can drive $\langle \nabla f(x^{(k)}), z \rangle \rightarrow
-\infty$ along a sequence with $g(z)=0$, and so the stagewise update
step would be clearly undefined.
Fortunately, a simple modification of the general stagewise
algorithm can account for this problem.  Since we are assuming that
$g$ is a seminorm, the set $N_g$ is a linear subspace.  To 
initialize the general stagewise algorithm at say $t_0=0$, therefore,
we solve the linearly constrained optimization problem  
\begin{equation*}
x^{(0)} \in \argmin_{x\in N_g} \, f(x).
\end{equation*}
In subsequent stagewise steps, we then restrict the updates to lie in
the subspace orthogonal to $N_g$.  That is, to be explicit, we
replace \eqref{eq:stageup} \eqref{eq:stagedir} in Algorithm
\ref{alg:genstage} with  
\begin{gather}
\label{eq:stageupmod}
x^{(k)} = x^{(k-1)} + \Delta, \\ 
\text{where}\;\,
\label{eq:stagedirmod}
\Delta \in \argmin_{z \in N_g^\perp} \,\,
\langle \nabla f(x^{(k-1)}), z \rangle
\;\,\st\;\, g(z) \leq \epsilon,
\end{gather}
where \smash{$N_g^\perp$} denotes the orthocomplement of $N_g$.  We
will see this modification, e.g., put to use for the quadratic regularizer
$g(\beta)=\beta^T Q \beta$, where $Q$ is positive semidefinite and
singular. 
\end{itemize}

Some readers may wonder why we are working with the 
constrained problem \eqref{eq:genprob}, and not 
\begin{equation}
\label{eq:genprob2}
\hx(\lambda) \in \argmin_{x\in\R^n} \, f(x) + \lambda g(x),
\end{equation}
where $\lambda \geq 0$ is now the regularization parameter, and is
called the Lagrange multiplier associated with $g$. It is probably
more common in the current statistics and machine learning literature
for optimization problems to be expressed in the Lagrange form
\eqref{eq:genprob2}, rather than the constrained form
\eqref{eq:genprob}.  The solution paths of \eqref{eq:genprob} and
\eqref{eq:genprob2}   
(given by varying $t$ and $\lambda$ in their respective problems)
are not necessarily equal for general convex functions $f$ and 
$g$; however, they are equal under very mild assumptions\footnote{For
  example, it is enough to assume that $g \geq 0$, and that for all
  parameters $t, \lambda \geq 0$, the solution sets of
  \eqref{eq:genprob}, \eqref{eq:genprob2} are nonempty.}, which hold 
for all of the examples visited in this paper. Therefore, there is not
an important difference in terms of studying \eqref{eq:genprob} versus 
\eqref{eq:genprob2}. We choose to focus on \eqref{eq:genprob} as we
feel that the intuition for stagewise algorithms is easier to see 
with this formulation.   

\subsection{Related work}
\label{sec:related}

There is a lot of work related to the proposal of this paper.  
Readers familiar with optimization will likely identify the 
general stagewise procedure, in Algorithm \ref{alg:genstage}, 
as a particular type of (normalized) {\it steepest descent}.
Steepest descent is an iterative algorithm for minimizing a smooth
convex function $f$,
in which we update the current iterate in a direction that minimizes
the inner product with the gradient of $f$ (evaluated at the current
iterate), among all vectors constrained to have norm $\|\cdot\|$ 
bounded by 1 (e.g., see \citet{convex}); the step size for the
update can be chosen in any one of the usual ways
for descent methods.   Note that gradient descent is simply a
special case of steepest descent with $\|\cdot\|=\|\cdot\|_2$ 
(modulo normalizing factors). Meanwhile,
the general stagewise algorithm is just steepest
descent with $\|\cdot\|=g(\cdot)$, and a
constant step size $\epsilon$.  It is important to point out that our 
interest in the general stagewise procedure is different from
typical interest in steepest descent.  
In the classic usage of steepest descent, we seek to minimize a
differentiable convex function $f$; our choice of norm $\|\cdot\|$
affects the speed with which we can find such a minimizer, but under
weak conditions, any choice of norm will eventually bring us to a
minimizer nonetheless.  In the general stagewise algorithm, we are not
really interested in the final minimizer itself, but rather,
the path traversed in order to get to this minimizer. 
The stagewise path is composed of iterates that have
interesting statistical properties, 
given by gradually balancing $f$ and $g$; choosing different
functions $g$ will lead to generically different paths. Focusing
on the path, instead of its endpoint, may seem strange
to a researcher in optimization, but it is quite natural for
researchers in statistics and machine learning.   

Another method related to our general stagewise proposal 
is the {\it Frank-Wolfe algorithm} \citep{frankwolfe},
used to minimize a differentiable convex function $f$ over a convex 
set $C$.  Similar to (projected) gradient descent, which
iteratively minimizes local quadratic approximations of $f$ over $C$,
the Frank-Wolfe algorithm iteratively minimizes local linear
approximations of $f$ over $C$.
In a recent paper, \citet{jaggi} shed light on Frank-Wolfe as an
efficient, scalable algorithm for modern machine learning problems.
For a single value of the regularization parameter $t$, the
Frank-Wolfe algorithm can be used to solve problem
\eqref{eq:genprob}, taking as the constraint set 
$C=\{x : g(x) \leq t\}$; the Frank-Wolfe steps here look very similar
to the general stagewise steps \eqref{eq:stageup},
\eqref{eq:stagedir}, but an important distinction is that the iterates
from Frank-Wolfe result in a single estimate, rather than each iterate
constituting its own estimate along the regularization path, as in the
general stagewise procedure.  This connection deserves
more discussion, and so we dedicate a subsection of the appendix to 
it: see Appendix \ref{app:frankwolfe}.  Other well-known methods
based on 
local linearization are {\it cutting-plane} \citep{cuttingplane} and  
{\it bundle} \citep{cvxanalandmin} methods.  \citet{teo1} present a 
general bundle method for regularized risk minimization that is
particularly relevant to our proposal (see also \citet{teo2}); this is
similar to the Frank-Wolfe approach in that it solves the
problem \eqref{eq:genprob} at a fixed value of the parameter $t$ (one
difference is that its local linearization steps are based on the 
entire history of previous iterates, instead of just the single
last iterate).  For brevity, we do not conduct a
detailed comparison between their bundle method and our general 
stagewise procedure, though we believe it would be interesting to do
so.  

Yet another class of methods that are highly relevant to our
proposal are {\it boosting} procedures.  Boosting algorithms are
iterative in form, and we typically think of them as tracing out
a sequence of estimates, just like our general stagewise algorithm
(and unlike the iterative algorithms described above, e.g., steepest
descent and Frank-Wolfe, which we tend to think of as culminating in 
a single estimate).  The literature on boosting is vast; see, e.g.,
\citet{esl} or \citet{boostreview} for a nice review.  Among boosting
methods, {\it gradient boosting} \citep{gradboost} most closely
parallels forward stagewise fitting.  Consider a setup in which 
our weak learners are the individual predictor variables $X_1,\ldots
X_p$, and the loss function is $L(X\beta)=f(\beta)$.  The gradient
boosting updates, using a shrinkage factor $\epsilon$, are most
commonly expressed in terms of the fitted values, as in
\begin{gather}
\label{eq:gbup}
X\beta^{(k)} = X\beta^{(k-1)} + \epsilon \cdot \alpha_i X_i, \\
\label{eq:gbsearch}
\text{where} \;\,
\alpha_i \in \argmin_{\alpha \in \R} \, 
L(X\beta^{(k-1)} + \alpha X_i), \\
\label{eq:gbdir}
\text{and}\;\,
i \in \argmin_{j=1,\ldots p} \, \bigg( \min_{\alpha \in \R} \, 
\|-\nabla L (X\beta^{(k-1)}) - \alpha X_j\|_2^2 \bigg).
\end{gather}
The step \eqref{eq:gbdir} selects the weak learner $X_i$ that
best matches the negative gradient,
$-\nabla L (X\beta^{(k-1)})$, in a least squares 
sense; the step \eqref{eq:gbsearch} chooses the coefficient
$\alpha_i$ of $X_i$ via line search. If we assume that the variables
have been scaled to have unit norm,  $\|X_j\|_2=1$ for $j=1,\ldots p$,
then it is easy to see that \eqref{eq:gbdir} is equivalent to  
\begin{equation*}
i \in \argmax_{j=1,\ldots p} \, 
|X_j^T \nabla L(X\beta^{(k-1)}) | = 
\argmax_{j=1,\ldots p} \,
|\nabla_j f(\beta^{(k-1)})|,
\end{equation*}
which is exactly the same selection criterion used by forward
stagewise under the loss function $f$, as expressed in
\eqref{eq:fsdir}.  Therefore, at a given iteration, gradient boosting
and forward stagewise choose the next variable $i$ in the
same manner, and only differ in their choice of the coefficient of
$X_i$ in the constructed additive model. The gradient boosting update
in \eqref{eq:gbup} adds $\epsilon \cdot \alpha_i X_i$ to the current
model, where $\alpha_i$ is chosen by line search in
\eqref{eq:gbsearch}; meanwhile, the forward stagewise
update in \eqref{eq:fsup} can be expressed as
\begin{equation}
\label{eq:fsup3}
X\beta^{(k)} = X\beta^{(k-1)} + \epsilon \cdot s_i X_i,
\end{equation}
where $s_i = -\sign(\nabla_i f(\beta^{(k-1)})$, a simple choice
of coefficient compared to $\alpha_i$.  Because $\alpha_i$ is chosen
by minimizing the loss function along the direction defined by $X_i$ 
(anchored at $X\beta^{(k-1)}$), gradient boosting is even more greedy
than forward stagewise, but practically there is not a big difference
between the two, especially when $\epsilon$ is small.
In fact, the distinction between \eqref{eq:gbup} and \eqref{eq:fsup3}
is slight enough that several authors refer to forward stagewise
as a boosting procedure, e.g., \citet{boostpath},
\citet{stagewiselasso}, and \citet{boostreview} refer to forward 
stagewise as {\it $\epsilon$-boosting}.   

The tie between boosting and forward stagewise suggests that we might
be able to look at our general stagewise proposal through the lens of
boosting, as well. Above we compared boosting and forward
stagewise for the problem of sparse estimation; in this problem,
deciding on the universe of weak learners for gradient boosting is
more or less straightforward, as we can use the variables $X_1,\ldots X_p$
themselves (or, e.g., smooth marginal transformations of these
variables for sparse nonparametric
estimation). This works because each iteration of gradient 
boosting adds a single weak learner to the fitted model, so the 
model is sparse in the early stages of the algorithm, and becomes  
increasingly dense as the algorithm proceeds.  However, for more
complex problems (beyond sparse estimation), specifying a universe of
weak learners is not as straightforward.  Consider, e.g., matrix
completion or image denoising---what kind of weak learners would be
appropriate here? At a broad level, our 
general stagewise procedure offers a prescription for a
class of weak learners based on the regularizer $g$,
through the definition of $\Delta$ in \eqref{eq:stagedir}.  Such weak
learners seem intuitively reasonable in various problem settings:
they end up being groups of variables for group-structured estimation
problems (see Section \ref{sec:groupreg}), rank 1 matrices for matrix 
completion (Section \ref{sec:tracereg}), and pixel
contrasts for image denoising (Section \ref{sec:genlasso}).  This may 
lead to an interesting perspective on gradient boosting with an
arbitrary regularization scheme, though we do not explore it further. 

Finally, the form of the update $\Delta$ in \eqref{eq:stagedir} sets our 
work apart from other general path tracing
procedures. \citet{stagewiselasso} and \citet{gps} propose approximate
path following methods for optimization problems whose regularizers
extend beyond the $\ell_1$ norm, but their algorithms only update 
one component of the estimate at a time (which corresponds to utilizing
individual variables as weak learners, in the boosting perspective);
on the other hand, our general stagewise procedure specifically adapts 
its updates to the regularizer of concern $g$.  
We note that, in certain special cases (i.e., for certain regularizers
$g$), our proposed algorithm bears similarities to existing algorithms 
in the literature: for ridge regularization, our proposal is similar
to gradient-directed path following, as studied in 
\citet{graddir} and \citet{parflow}, and for $\ell_1/\ell_2$ 
multitask learning, our stagewise algorithm is similar to the
block-wise path following method of \citet{multitask2}. 








\section{Applications of the general stagewise framework} 
\label{sec:applications}

\subsection{Group-structured regularization}
\label{sec:groupreg}

We begin by considering the group-structured regularization problem 
\begin{equation}
\label{eq:groupreg}
\hbeta(t) \in \argmin_{\beta\in\R^p} \,f(\beta)
\;\,\st\;\, \sum_{j=1}^G w_j \|\beta_{\cI_j}\|_2 \leq t,
\end{equation}
where the index set $\{1,\ldots p\}$ has been partitioned into $G$ 
groups $\cI_1,\ldots \cI_G$, $\beta_{\cI_j} \in \R^{p_j}$
denotes the components of $\beta\in\R^p$ for the $j$th group, and  
$w_1,\ldots w_G \geq 0$ are fixed weights.  The loss $f$ is kept as a 
generic differentiable convex function---this is because, as
explained in Section \ref{sec:basicprops}, the stagewise updates are  
invariant around $\nabla f$, in terms of their computational form.  

Note that the {\it group lasso} problem \citep{bakinphd,grouplasso} is
a special case of \eqref{eq:groupreg}.  In the typical group lasso
regression setup, we observe an outcome $y\in\R^n$ and predictors  
$X\in\R^{n\times p}$, and the predictor variables admit some natural
grouping $\cI_1,\ldots \cI_G$. To perform group-wise variable
selection, one can use the group lasso estimator, defined as in   
\eqref{eq:groupreg} with  
\begin{equation*}
f(\beta) = 
\half \Big\|y-\sum_{j=1}^G X_{\cI_j} \beta_{\cI_j} \Big\|_2^2 
\;\;\;\text{and}\;\;\; 
w_j = \sqrt{p_j}, \;\, j=1,\ldots G,
\end{equation*}
where $X_{\cI_j}\in\R^{n\times p_j}$ is the predictor matrix for group
$j$, and $p_j=|\cI_j|$ is the size of the group $j$.  The same
idea clearly applies outside of the linear regression setting
(e.g., see \citet{grouplassolog} for a study of the group lasso
regularization in logistic regression).

A related yet distinct problem is that of {\it multitask
learning}.  In this setting we consider not one but multiple learning
problems, or tasks, and we want to select a common set of 
variables that are important across all tasks.  A popular estimator 
for this purpose is based on $\ell_1/\ell_2$ regularization
\citep{multitask1,multitask2}, and
also fits into the framework \eqref{eq:groupreg}: the loss function
$f$ becomes the sum of the losses across the
tasks, and the groups $\cI_1,\ldots \cI_G$ collect the coefficients
corresponding to the same variables across tasks.  For example, in 
multitask linear regression, we write $y^{(i)} \in \R^n$ for the
outcome, $X^{(i)} \in \R^{n\times m}$ for the predictors, and
$\beta^{(i)}$ the coefficients for the $i$th task, 
$i=1,\ldots r$.  We form a global coefficient vector 
$\beta=(\beta^{(1)},\ldots \beta^{(m)}) \in \R^p$, where $p=m\cdot r$,
and form groups $\cI_1,\ldots \cI_m$, where $\cI_j$ collects the
coefficients of predictor variable $j$ across the tasks.  The 
$\ell_1/\ell_2$ regularized multitask learning estimator is then
defined as in \eqref{eq:groupreg} with
\begin{equation*}
f(\beta) = \half \sum_{i=1}^r \|y^{(i)} - X^{(i)} \beta^{(i)} \|_2^2
\;\;\;\text{and}\;\;\; w_j = 1, \;\, j=1,\ldots m,
\end{equation*}
where the default is to set all of the weights to 1, in the lack of any 
prior information about variable importance (note that the groups
$\cI_1,\ldots \cI_m$ are all the same size here).

The general stagewise algorithm, Algorithm \ref{alg:genstage}, does
not make any distinction between cases such as the group lasso and 
multitask learning problems; it only requires $f$ to be a convex
and smooth function. To initialize the algorithm for the group
regularized problem \eqref{eq:groupreg}, we can take $t_0=0$
and $\beta^{(0)}=0$. The next lemma shows how to calculate the 
appropriate update direction $\Delta$ in \eqref{eq:stagedir}. 

\begin{lemma}
\label{lem:groupdir}
For \smash{$g(\beta)=\sum_{j=1}^G w_j \|\beta_{\cI_j}\|_2$}, the
general stagewise procedure in Algorithm \ref{alg:genstage} repeats
the updates $\beta^{(k)}= \beta^{(k-1)}+\Delta$, where $\Delta$ can be 
computed as follows: first find $i$ such that
\begin{equation}
\label{eq:groupdir0}
\frac{\|(\nabla f)_{\cI_i}\|_2}{w_i} = \max_{j=1,\ldots G} \,
\frac{\|(\nabla f)_{\cI_j}\|_2}{w_j},
\end{equation}
where we abbreviate $\nabla f = \nabla f(\beta^{(k-1)})$, then let
\begin{align}
\label{eq:groupdir1}
\Delta_{\cI_j} &= 0 \;\;\;\text{for all}\;\, j\not=i, \\
\label{eq:groupdir2}
\Delta_{\cI_i} &= \frac{-\epsilon \cdot (\nabla f)_{\cI_i}}
{w_i \|(\nabla f)_{\cI_i}\|_2}.
\end{align}
\end{lemma}


We omit the proof; it follows straight from the KKT conditions for  
\eqref{eq:stagedir}, with $g$ as defined in the lemma.
Computation of $\Delta$ in \eqref{eq:groupdir0}, 
\eqref{eq:groupdir1}, \eqref{eq:groupdir2} is very cheap, 
and requires $O(p)$ operations. To rephrase:
at the $k$th iteration, we simply find the group $i$ such that the
corresponding 
block of the gradient $\nabla f(\beta^{(k-1)})$ has the largest $\ell_2$
norm (after scaling appropriately by the weights). We then
move the coefficients for group $i$ in a direction opposite to this 
gradient value; for all other groups, we leave their coefficients untouched
(note that, if a group has not been visited by past update steps, then this
means leaving its coefficients identically equal to zero).
The outputs of the stagewise algorithm therefore match our intuition
about the role of the constraint in \eqref{eq:groupreg}---for 
some select groups, all coefficients are set to nonzero values, and for other 
groups, all coefficients are set to zero. That the actual solution in
\eqref{eq:groupreg} satisfies this intuitive property can be verified
by examining its own KKT conditions.


Looking back at Figure \ref{fig:examples}, the first row compares the
exact solution and stagewise paths for a group lasso regression
problem.  The stagewise path was computed using 300 steps with 
$\epsilon=0.01$, and shows strong similarities to the exact group
lasso path.  In other problem instances, say, when the predictors
across different groups are highly correlated, the group lasso
coefficient paths can behave wildly with $t$, and yet the stagewise
paths can appear much less wild and more stable. Later, in Section
\ref{sec:bigexamples}, we consider larger examples and give more
thorough empirical comparisons.

\subsection{Group-structured regularization with arbitrary norms}
\label{sec:groupregarb}

Several authors have considered group-based regularization using the
$\ell_\infty$ norm in place of the usual $\ell_2$ norm (e.g., see
\citet{simulvar} for such an approach in multitask learning).
To accomodate this and other general group-structured 
regularization approaches, we consider the problem
\begin{equation}
\label{eq:groupreg2}
\hbeta(t) \in \argmin_{\beta\in\R^p} \,f(\beta)
\;\,\st\;\, \sum_{j=1}^G w_j h_j(\beta_{\cI_j}) \leq t,
\end{equation}
where each $h_j$ is an arbitrary norm.  Let $h_j^*$ denote the  
dual norm of $h_j$; e.g., if $h_j(x)=\|x\|_{q_j}$, then 
$h_j^*(x)=\|x\|_{r_j}$, where $1/q_j+1/r_j = 1$.  
Similar to the result in Lemma
\ref{lem:groupdir}, the stagewise updates for problem
\eqref{eq:groupreg2} take a simple group-based form. 

\begin{lemma}
\label{lem:group2dir}
For \smash{$g(\beta)=\sum_{j=1}^G w_j h_j(\beta_{\cI_j})$}, the
general stagewise procedure in Algorithm \ref{alg:genstage} repeats
the updates $\beta^{(k)}= \beta^{(k-1)}+\Delta$, where $\Delta$ can be 
computed as follows: first find $i$ such that
\begin{equation*}
\frac{h_i^*\big((\nabla f)_{\cI_i}\big)}{w_i} = \max_{j=1,\ldots G} \,
\frac{h_j^*\big((\nabla f)_{\cI_j}\big)}{w_j},
\end{equation*}
where we abbreviate $\nabla f = \nabla f(\beta^{(k-1)})$, then let
\begin{align*}
\Delta_{\cI_j} &= 0 \;\;\;\text{for all}\;\, j\not=i, \\
\Delta_{\cI_i} &\in -\frac{\epsilon}{w_i} \cdot 
\partial h_i^* \big( (\nabla f)_{\cI_i} \big).
\end{align*}
\end{lemma}

Again we omit the proof; it follows from the KKT conditions for
\eqref{eq:stagedir}. Indeed, Lemma \ref{lem:group2dir} covers  
Lemma \ref{lem:groupdir} as a special case, recalling that the
$\ell_2$ norm is self-dual.  Also, recalling that the $\ell_\infty$
and $\ell_1$ norms are dual, Lemma \ref{lem:group2dir} says that the 
stagewise algorithm for $g(\beta)=\sum_{j=1}^G w_j
\|\beta_{\cI_j}\|_\infty$ first finds $i$ such that
\begin{equation*}
\frac{\|(\nabla f)_{\cI_i}\|_1}{w_i} = \max_{j=1,\ldots G} \,
\frac{\|(\nabla f)_{\cI_j}\|_1}{w_j},
\end{equation*}
and then defines the update direction $\Delta$ by
\begin{align*}
\Delta_{\cI_j} &= 0 \;\;\;\text{for all}\;\, j\not=i, \\
\Delta_\ell &= -\frac{\epsilon}{w_i} \cdot
\begin{cases}
0 & \text{for}\;\, \ell \in \cI_i, \,(\nabla f)_\ell = 0 \\
\sign\big((\nabla f)_\ell\big) & \text{for}\;\, \ell \in \cI_i, \, 
(\nabla f)_\ell \not= 0. 
\end{cases}
\end{align*}
More broadly, Lemma \ref{lem:group2dir} provides a general
prescription for deriving the stagewise updates for
regularizers that are block-wise sums of norms, as long as we can
compute subgradients of the dual norms.  For example, the norms
in consideration could be a mix of $\ell_p$ norms, matrix norms, 
etc.  

\subsection{Trace norm regularization}
\label{sec:tracereg}

Consider a class of optimization problems over matrices,
\begin{equation}
\label{eq:tracereg}
\hat{B}(t) \in \argmin_{B \in\R^{m\times n}} \,f(B)
\;\,\st\;\, \|B\|_* \leq t,
\end{equation}
where $\|B\|_*$ denotes the trace norm (also called the nuclear norm)
of a matrix $B$, i.e., the sum of its singular values.  
Perhaps the most well-known example of trace norm regularization comes
from the problem of {\it matrix completion} (e.g., see
\citet{exactmat}, \citet{nearopt2},
\citet{softimpute}).  Here the setup is 
that we only partially observe entries of a matrix $Y \in \R^{m\times 
n}$---say, we observe all entries $(i,j) \in \Omega$---and we seek
to estimate the missing entries.  A natural estimator for this purpose
(studied by, e.g., \citet{softimpute}) is defined as in
\eqref{eq:tracereg} with  
\begin{equation*}
f(B) = \half\sum_{(i,j) \in \Omega} (Y_{ij} - B_{ij})^2.
\end{equation*}
The trace norm also appears in interesting examples beyond
matrix completion.  For example, \citet{nptrace} consider
regularization with the trace norm in multiple nonparametric
regression, and \citet{imtrace} consider it in large-scale
image classification.   

The general stagewise algorithm applied to the trace norm
regularization problem \eqref{eq:tracereg} can be initialized with
$t_0=0$ and $B^{(0)}=0$, and the update direction in \eqref{eq:stagedir}
is now simple and efficient.

\begin{lemma}
\label{lem:tracedir}
For $g(B)=\|B\|_*$, the general stagewise procedure in Algorithm
\ref{alg:genstage} repeats the updates
$\beta^{(k)}=\beta^{(k-1)}+\Delta$, where 
\begin{equation}
\label{eq:tracedir}
\Delta = -\epsilon \cdot uv^T,
\end{equation}
with $u,v$ being leading left and right singular vectors,
respectively, of $\nabla f (B^{(k-1)})$.
\end{lemma}


The proof relies on the fact that the dual of the trace norm
$g(B)=\|B\|_*$ is the spectral norm $g^*(B) = \|B\|_2$, and then
invokes the representation \eqref{eq:dualdir} for stagewise
estimates. 
For the stagewise update direction \eqref{eq:tracedir}, we need 
to compute the leading left and right singular vectors $u,v$ of the
$m\times n$ matrix $\nabla f(B^{(k-1)})$---these are the left and
right singular vectors corresponding to the top singular value of 
$\nabla f(B^{(k-1)})$.  Assuming that $\nabla f(B^{(k-1)})$ has a
distinct largest singular value, this can be done, e.g., using the
power method: letting $A=\nabla f(B^{(k-1)})$, we first run 
the power method on the $m\times m$ matrix $AA^T$, or the $n\times n$
matrix $A^T A$, depending on whichever is smaller.  This gives us
either $u$ or $v$; to recover the other, we then simply use matrix 
multiplication: $v=A^T u / \|A^T u\|_2$ or $u = Av / \|Av\|_2$.  
The power method is especially efficient if $A = \nabla
f(B^{(k-1)})$ is sparse (each iteration being faster), or has a large
spectral gap (fewer iterations required until convergence).
Of course, alternatives to the power method can be used
for computing the leading singular vectors of $\nabla f(B^{(k-1)})$,
such as methods based on inverse iterations, Rayleigh quotients, 
or QR iterations; see, e.g., \citet{gvl}.

In the second row of Figure \ref{fig:examples}, the exact and
stagewise paths for are shown matrix completion problem, where the
stagewise paths were computed using 500 steps with $\epsilon=0.05$.
While the two sets of paths appear fairly similar, we note that it is
harder to judge the degree of similarity between the two in the matrix
completion context. Here, 
the coordinate paths correspond to entries in the estimated matrix 
\smash{$\hat{B}$}, and their roles are not as clear as they are in,
say, in a regression setting, where the coordinate paths correspond to
the coefficients of individual variables. In other words, it is
difficult to interpret the slight differences between the exact and 
stagewise paths in the second row of Figure \ref{fig:examples}, which
present themselves as the trace norm grows large.  Therefore, to get a 
sense for the effect of these differences, we might compare
the mean squared error curves generated by the exact and stagewise 
estimates.  This is done in depth in Section \ref{sec:bigexamples}.

\subsection{Quadratic regularization}
\label{sec:quadreg}

Consider problems of the form
\begin{equation}
\label{eq:quadreg}
\hbeta(t) \in \argmin_{\beta \in\R^p} \, f(\beta) 
\;\,\st\;\, \beta^T Q \beta \leq t,
\end{equation}
where $Q \succeq 0$, a positive semidefinite matrix.
The quadratic regularizer in \eqref{eq:quadreg} encompasses several
common statistical tasks.  When $Q=I$, the regularization 
term $\beta^T \beta = \|\beta\|_2^2$ is well-known as {\it ridge} 
\citep{ridge}, or {\it Tikhonov} regularization
\citep{tikhonov}. This regularizer shrinks the components of the
solution \smash{$\hbeta$} towards zero.  In a (generalized)
linear model setting with many predictor variables, such
shrinkage helps control the variance of the estimated coefficients.
Beyond this simple ridge case, roughness regularization in
nonparametric regression often fits into the form \eqref{eq:quadreg},
with $Q$ not just the identity.  For example, 
{\it smoothing splines} \citep{wahbasplines,greensilver} 
and {\it P-splines} \citep{psplines} can both be expressed as in
\eqref{eq:quadreg}. 
To see this, suppose that $y_1,\ldots y_n \in \R$ are observed
across input points $x_1,\ldots x_n \in \R$,    
and let $b_1,\ldots b_p$ denote the B-spline basis (of, say, 
cubic order) with knots at locations 
$z_1,\ldots z_p \in \R$. Smoothing splines use the inputs as
knots, $z_1=x_1, \ldots z_p=x_n$ (so that $p=n$); P-splines 
typically employ a (much) smaller number of knots across the range of
$x_1,\ldots x_n \in \R$.  Both estimators solve problem
\eqref{eq:quadreg}, with a loss function
\smash{$f(\beta)=\half\|y-B\beta\|_2^2$}, 
and $B \in  \R^{n\times p}$ having entries $B_{ij}=b_j(x_i)$, but the
two use a different definition for $Q$: its entries are given by
\smash{$Q_{ij} =  \int b''_i(x) b''_j(x) \, dx$} in the case of
smoothing splines, while $Q=D^T D$ in the case of P-splines, where $D$ 
is the discrete difference operator of a given (fixed) integral
order.  Both estimators can be extended to the logistic or Poisson 
regression settings, just by setting 
$f$ to be the logistic or Poisson loss, with natural parameter
$\eta=B\beta$ \citep{greensilver,psplines}.  

When $Q$ is positive definite, the general stagewise
algorithm, applied to \eqref{eq:quadreg}, can be initialized with
$t_0=0$ and $\beta^{(0)}=0$. 
The update direction $\Delta$ in \eqref{eq:stagedir} is
described by the following lemma.

\begin{lemma}
\label{lem:quadreg}
For $g(\beta)=\beta^T Q \beta$, with $Q$ a
positive definite matrix, the general stagewise procedure 
in Algorithm \ref{alg:genstage} repeats the updates
$\beta^{(k)}=\beta^{(k-1)}+\Delta$, where  
\begin{equation}
\label{eq:quaddir}
\Delta = -\sqrt{\epsilon} \cdot 
\frac{Q^{-1} \nabla f}{\sqrt{(\nabla f)^T Q^{-1} \nabla f}},
\end{equation}
and $\nabla f$ is an abbreviation for $\nabla f(\beta^{(k-1)})$. 
\end{lemma}


The proof follows by checking the KKT conditions for
\eqref{eq:stagedir}.  When $Q=I$, the update step \eqref{eq:quaddir}
of the general stagewise procedure for quadratic regularization is
computationally trivial, reducing to 
\begin{equation*}
\Delta = -\sqrt{\epsilon} \cdot 
\frac{\nabla f}{\|\nabla f\|_2}.
\end{equation*}
This yields fast, simple updates for ridge regularized estimators.
For a general matrix $Q$, computing the update direction in
\eqref{eq:quaddir} boils down to solving the linear equation  
\begin{equation}
\label{eq:q}
Q v = \nabla f(\beta^{(k-1)})
\end{equation}
in $v$. This is expensive for an arbitrary, dense $Q$; a
single solve of the linear system \eqref{eq:q} generally requires 
$O(p^3)$ operations.  
Of course, since the systems across all iterations involve the
same linear operator $Q$, we could initially compute a Cholesky
decomposition of $Q$ (or a related factorization), requiring
$O(p^3)$ operations, and then use this factorization to solve
\eqref{eq:q} at each iteration, requiring only $O(p^2)$
operations.  While certainly more efficient than the naive strategy of
separately solving each instance of \eqref{eq:q}, this is still not
entirely desirable for large problems.  

On the other hand, for several cases in which $Q$ is structured or 
sparse, the linear system \eqref{eq:q} can be solved efficiently.
For example, if $Q$ is banded with bandwidth $d$, then we can solve  
\eqref{eq:q} in $O(pd^2)$ operations (actually, an
initial Cholesky decomposition takes $O(pd^2)$ operations, and
each successive solve with this decomposition then takes $O(pd)$
operations).  

Importantly, the matrix $Q$ is banded in both the smoothing
spline and P-spline regularization cases: for smoothing splines, $Q$
is banded because the B-spline basis functions have local support;
for P-splines, $Q$ is banded because the discrete difference operator
is.  However, some care must be taken in applying the stagewise
updates in these cases, as $Q$ is singular, i.e., positive
semidefinite but not strictly positive definite.     
The stagewise algorithm needs to be modified, albeit only slightly, to
deal with this issue---this modification was discussed in
\eqref{eq:stageupmod}, \eqref{eq:stagedirmod} in Section
\ref{sec:basicprops}, and here we summarize the implications for
problem \eqref{eq:quadreg}.  First we compute the initial iterate to
lie in $\nul(Q)$, the null space of $Q$,   
\begin{equation}
\label{eq:quadinit2}
\beta^{(0)} \in \argmin_{\beta\in\nul(Q)} \, f(\beta).
\end{equation}
For, e.g., P-splines with $Q=D^T D$, and $D$ the discrete
difference operator of order $k$, the space $\nul(Q)$ is
$k$-dimensional and contains (the evaluations of) all polynomial
functions of order $k-1$.   The stagewise algorithm is then
initialized at such a point $\beta^{(0)}$ in \eqref{eq:quadinit2}, and
$t_0=0$.  For future iterations, note that when $\nabla
f(\beta^{(k)})$ has a nontrivial projection onto 
$\nul(Q)$, the stagewise update in \eqref{eq:stagedir} is
undefined, since $\langle \nabla f(\beta^{(k)}), z \rangle$ can be
made arbitrarily small along a direction $z$ such that $z^T Q z = 0$.
Therefore, we must further constrain the stagewise update to lie in
the orthocomplement $\nul(Q)^\perp=\row(Q)$, the row space of $Q$, 
as in 
\begin{equation*}
\Delta \in \argmin_{z \in \row(Q)} \,\,
\langle \nabla f(\beta^{(k-1)}), z \rangle
\;\,\st\;\, z^T Q z \leq \epsilon.
\end{equation*}
It is not hard to check that, instead of \eqref{eq:quaddir}, the
update now becomes
\begin{equation}
\label{eq:quaddir2}
\Delta = -\sqrt{\epsilon} \cdot 
\frac{Q^+\nabla f}
{\sqrt{(\nabla f)^T Q^+ \nabla f}}, 
\end{equation}
with $Q^+$ denoting the (Moore-Penrose) generalized inverse of $Q$.  

From a computational perspective, the stagewise update in
\eqref{eq:quaddir2} for the rank deficient case does not represent
more much work than that in \eqref{eq:quaddir} for the full rank
case.  With P-splines, e.g., we have $Q=D^T D$ where $D \in \R^{(n-k)
  \times n}$ is a banded matrix of full row rank.
A short calculation shows that in this case
\begin{equation*}
(D^T D)^+ = D^T (DD^T)^{-2} D,
\end{equation*}
i.e., applying $Q^+$ is computationally equivalent to two banded
linear system solves and two banded matrix multiplications.  Hence one
stagewise update for P-spline regularization problems takes $O(p)$
operations (the bandwidth of $D$ is a constant, $d=k+1$), excluding
computation of the gradient.

The third row of Figure \ref{fig:examples} shows an example of
logistic regression with ridge regularization, and displays the
grossly similar exact solution and stagewise paths. 
Notably, the stagewise path here was
constructed using only {\it 15 steps}, with an effective step size  
$\sqrt{\epsilon}=0.1$. This is a surprisingly small number of 
steps, especially compared to the numbers 
needed by stagewise in the examples (both small and large) from  
other regularization settings covered in this paper.  
As far as we can tell, 
this rough scaling appears to hold for ridge
regularization problems in general---for such problems, the stagewise
algorithm can be run with relatively large step sizes for small
numbers of steps, and it will still produce statistically appealing
paths. Unfortunately, this trend does not persist uniformly across all
quadratic regularization problems; it seems that the ridge case 
($Q=I$) is really a special one.    

For a second example, we consider P-spline regularization, using both 
continuous and binomial outcomes.  The left panel of Figure
\ref{fig:pspline} displays an array of stagewise
estimates, computed under P-spline regularization and a
Gaussian regression loss.  We generated $n=100$ noisy
observations $y_1,\ldots y_{100}$ from an underlying sinusoidal
curve, sampled at input locations $x_1,\ldots x_{100}$ drawn uniformly
over $[0,1]$.  The P-splines were defined using 30 equally spaced
knots across $[0,1]$, and the stagewise algorithm was run
for 300 steps with $\sqrt{\epsilon}=0.005$.  The figure shows the
spline approximations delivered by the stagewise estimates (from every 
15th step along the path, for visibility) and the true
sinusoidal curve overlayed as a thick dotted black line. We note that
in this particular setting, the stagewise algorithm is not so
interesting computationally, because each update step solves a banded
linear system, and yet the exact solution can itself be computed at
the same cost, at any regularization parameter value.
The example is instead meant to portray that the stagewise algorithm
can produce smooth and visually reasonable estimates of the underlying 
curve.  

The right panel of Figure \ref{fig:pspline} displays an
analogous example using $n=100$ binary observations,
$y_1,\ldots y_{100}$, generated according to the probabilities
\smash{$p_i^*=1/(1+e^{-\mu(x_i)})$},  
$i=1,\ldots 100$, where the inputs $x_1,\ldots x_{100}$ were sampled 
uniformly from $[0,1]$, and $\mu$ is a smooth function.  The
probability curve \smash{$p^*(x)=1/(1+e^{-\mu(x)})$} is drawn as a
thick dotted black line.  We ran the stagewise algorithm
under a logistic loss, with $\sqrt{\epsilon}=0.005$, and for 300 
steps; the figure plots the probability curves associated with the
stagewise estimates (from every 15th step along the path, for
visibility).  Again, we can see that the fitted curves are 
smooth and visually reasonable.  Computationally, the difficulty of
the stagewise algorithm in this logistic setting is essentially the
same  as that in the previous Gaussian setting; all that changes is
the computation of the gradient, which is an easy task. 
The exact solution, however, is more difficult to compute in this
setting than the previous, and requires the use of iterative algorithm
like Newton's method.  This kind of computational invariance around
the loss function, recall, is an advantage of the stagewise framework.  

\begin{figure}[htb]
\centering
\includegraphics[width=0.475\textwidth]{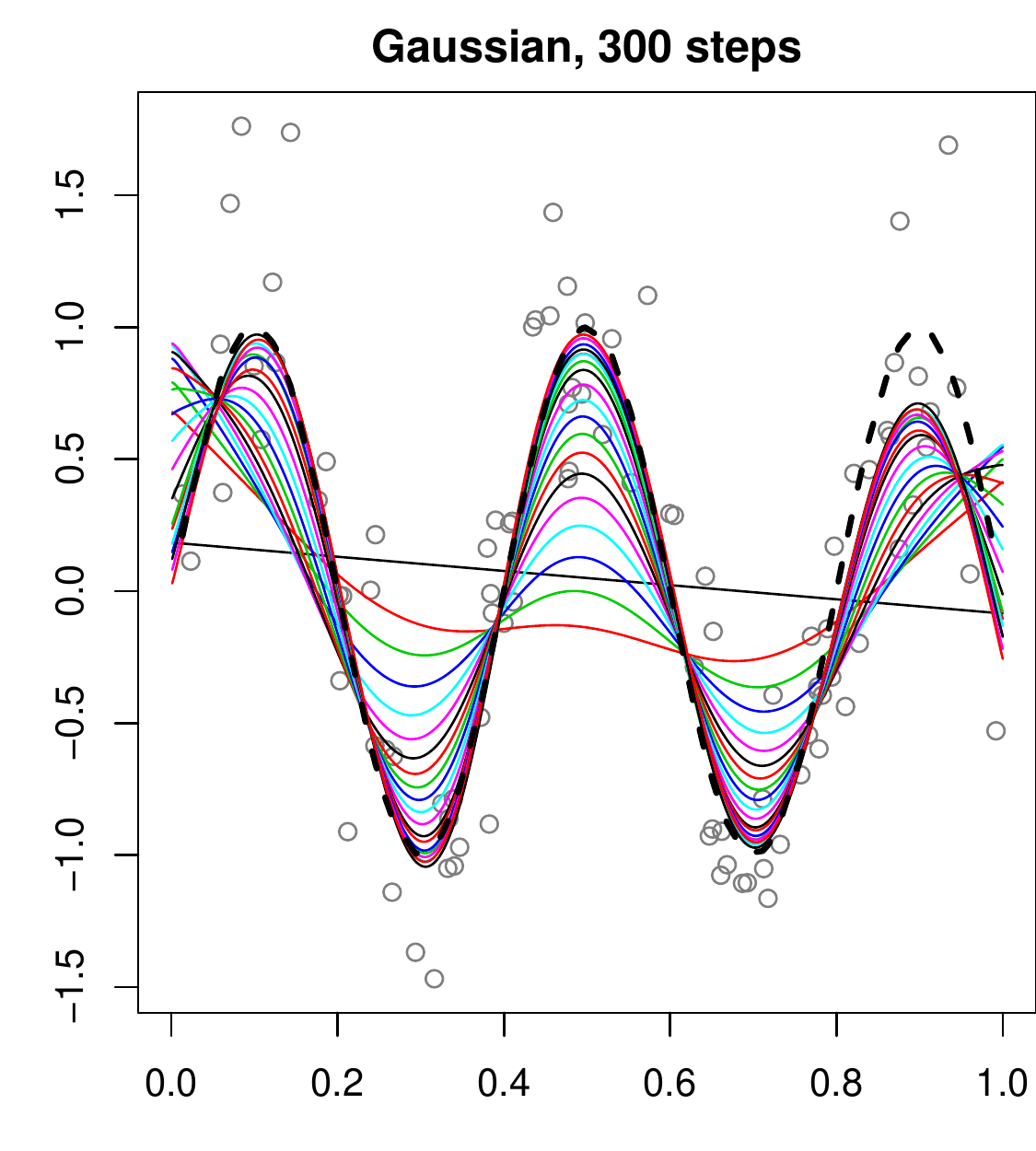} 
\includegraphics[width=0.475\textwidth]{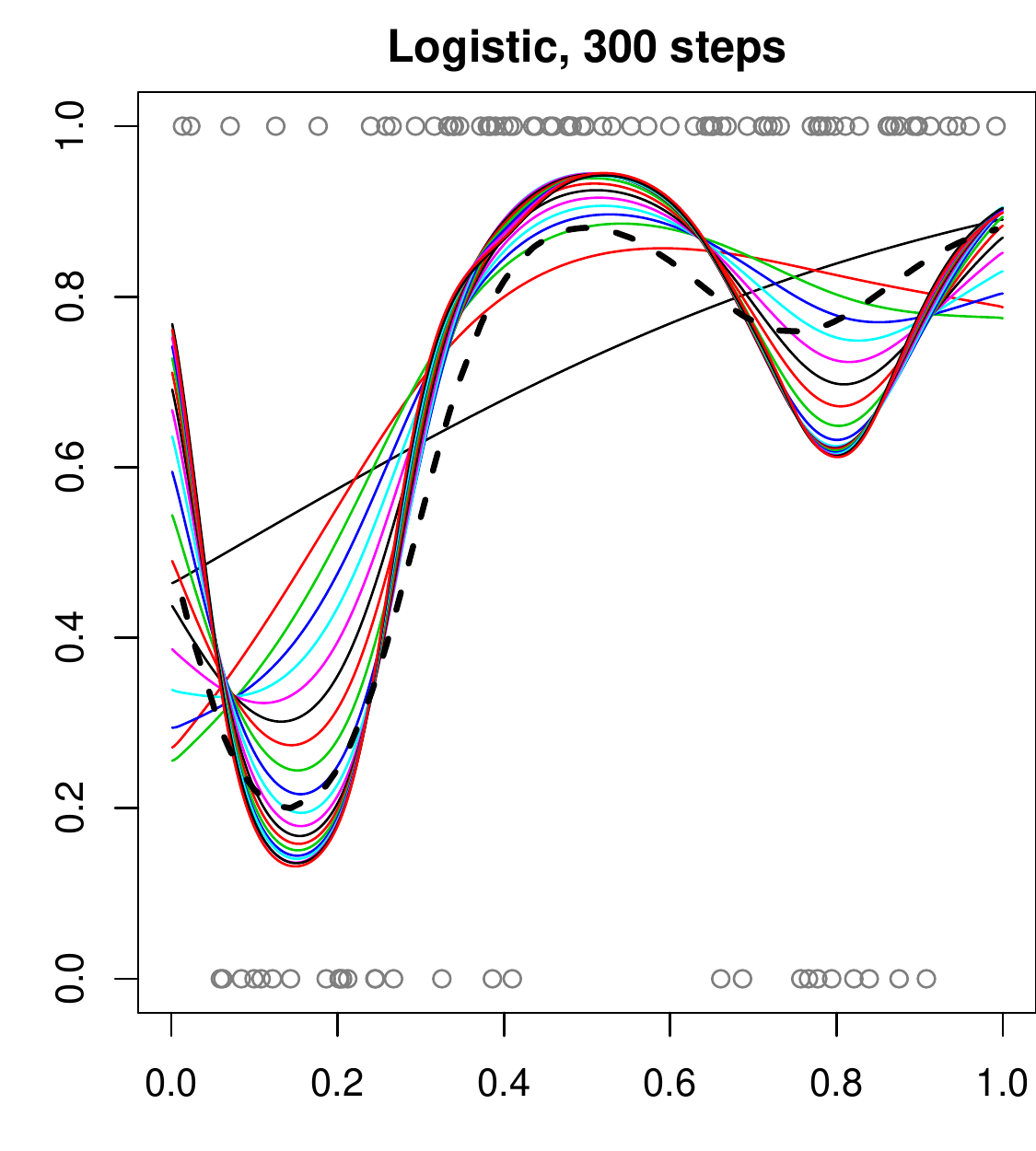} 
\caption{\it Snapshots of the stagewise path for P-spline
  regularization problems, with continuous data in the left panel, and
  binary data in the right panel.  In both examples, we use
  $n=100$ points, and the true data generating curve is displayed as
  a thick dotted black line.  The colored curves show the stagewise
  estimates over the first 300 path steps (plotted are every 15th
  estimate, for visibility).
}
\label{fig:pspline}
\end{figure}

\subsection{Generalized lasso regularization}
\label{sec:genlasso}

In this last application, we study generalized $\ell_1$ regularization 
problems,  
\begin{equation}
\label{eq:genlasso}
\hbeta(t) \in \argmin_{\beta\in\R^p}\,
f(\beta) \;\,\st\;\, \|D\beta\|_1 \leq t,
\end{equation}
where $D$ is a given matrix (it need not be square).
The regularization term above is
also called {\it generalized lasso} regularization, since it includes
lasso regularization as a special case, with $D=I$, but also covers a 
number of other regularization forms \citep{genlasso}.  For
example, {\it fused lasso} regularization is encompassed by
\eqref{eq:genlasso}, with $D$ chosen to be the edge incidence matrix
of some graph $G$, having nodes $V=\{1,\ldots p\}$ and
edges $E=\{e_1,\ldots e_m\}$.  In the special case of the chain graph,
wherein $E=\{\{1,2\},\{2,3\},\ldots \{p-1,p\}\}$, we have 
\begin{equation*}
D = \left[\begin{array}{rrrrrr}
-1 & 1 & 0 & \ldots & 0 & 0 \\
0 & -1 & 1 & \ldots & 0 & 0 \\
\vdots & & & & & \\
0 & 0 & 0 & \ldots & -1 & 1
\end{array}\right],
\end{equation*}
so that 
\smash{$\|D\beta\|_1=\sum_{j=1}^{p-1} |\beta_j-\beta_{j+1}|$}. This
regularization term encourages the ordered components of $\beta$ to be
piecewise constant, and problem \eqref{eq:genlasso} with this
particular choice of $D$ is usually called the 
{\it 1-dimensional fused lasso} in the statistics literature
\citep{fuse}, or {\it 1-dimensional total variation denoising} in
signal processing \citep{tv}.  In general, the edge incidence matrix  
$D \in \R^{m\times p}$ has rows corresponding to edges in $E$, and its
$\ell$th row is  
\begin{equation*}
D_\ell = (0, \ldots \underset{\substack{\;\;\uparrow \\ \;\;i}}{-1},
\ldots \underset{\substack{\uparrow \\ j}}{1}, \ldots 0) \in \R^p,
\end{equation*}
provided that the $\ell$th edge is $e_\ell=\{i,j\}$.  Hence
\smash{$\|D\beta\|_1=\sum_{\{i,j\} \in E} |\beta_i-\beta_j|$}, a
regularization term that encourages the components of $\beta$ to be
piecewise constant with respect to the structure defined by the graph
$G$. Higher degrees of
smoothness can be regularized in this framework as well, using {\it
  trend filtering} methods; see \citet{l1tf} or \citet{trendfilter}
for the 1-dimensional case, and \citet{graphtf} for the more general
case over arbitrary graphs.  

Unfortunately the stagewise update in \eqref{eq:stagedir}, under the
regularizer $g(\beta)=\|D\beta\|_1$, is not computationally
tractable.  Computing this update is the same as
solving a linear program, absent of any special structure in the
presence of a generic matrix $D$.  But we can make progress by
studying the generalized lasso from the perspective of convex
duality.  Our jumping point for the dual is actually the Lagrange form
of problem \eqref{eq:genlasso}, namely 
\begin{equation}
\label{eq:genlasso2}
\hbeta(\lambda) \in \argmin_{\beta\in\R^p}\,
f(\beta) + \lambda \|D\beta\|_1,
\end{equation}
with $\lambda \geq 0$ now being the regularization parameter.
The switch from \eqref{eq:genlasso} to \eqref{eq:genlasso2} is 
justified because the two parametrizations admit identical solution
paths.  Following standard arguments in convex analysis, the dual
problem of \eqref{eq:genlasso2} can be written as 
\begin{equation}
\label{eq:genlassodual}
\hu(\lambda) \in \argmin_{u \in \R^m} \, f^*(-D^T u) \;\, \st \;\,
\|u\|_\infty \leq \lambda,
\end{equation} 
with $f^*$ denoting the convex conjugate of $f$.  The primal and dual
solutions satisfy the relationship 
\begin{equation}
\label{eq:genlassopd}
\nabla f (\hbeta(\lambda)) + D^T \hu(\lambda) = 0.
\end{equation}
The general strategy is now to apply the stagewise algorithm to the
dual problem \eqref{eq:genlassodual} to produce an approximate
dual solution path, and then convert this into an approximate primal 
solution path via \eqref{eq:genlassopd}.
The stagewise procedure for \eqref{eq:genlassodual} can be initialized
with $\lambda_0=0$ and $u^{(0)}=0$, and the form of the updates is 
described next.  We assume that the conjugate function $f^*$ is
differentiable, which holds if $f$ is strictly convex.    

\begin{lemma}
\label{lem:genlasso}
Applied to the problem \eqref{eq:genlassodual}, the
general stagewise procedure in Algorithm \ref{alg:genstage} repeats
the updates $u^{(k)}=u^{(k-1)}+\Delta$, where    
\begin{equation}
\label{eq:genlassodir}
\Delta_i = -\epsilon \cdot
\begin{cases}
1 & \big[D \nabla f^*(-D^T u^{(k-1)})\big]_i < 0 \\
-1 & \big[D \nabla f^*(-D^T u^{(k-1)})\big]_i > 0\\
0 & \big[D \nabla f^*(-D^T u^{(k-1)})\big]_i = 0
\end{cases}
\;\;\;\text{for}\;\, i=1,\ldots m.
\end{equation}
\end{lemma}

The proof follows from the duality of the $\ell_\infty$ and $\ell_1$
norms, and the alternative representation in \eqref{eq:dualdir} for
stagewise updates. Computation of $\Delta$ in
\eqref{eq:genlassodir}, aside from evaluating the gradient $\nabla
f^*$, reduces to two matrix multiplications: one by $D$ and one by
$D^T$. In many cases (e.g., fused lasso and trend filtering problems),
the matrix $D$ is sparse, which makes this update step very cheap.    
To reiterate the dual strategy: we compute the dual estimates
$u^{(k)}$, $k=1,2,3,\ldots$ using the stagewise updates outlined
above, and we compute primal estimates $\beta^{(k)}$,
$k=1,2,3,\ldots$ by solving for $\beta^{(k)}$ in the stationarity
condition      
\begin{equation}
\label{eq:genlassopd2}
\nabla f (\beta^{(k)}) + D^T u^{(k)} = 0,
\end{equation}
for each $k$.  The $k$th dual iterate $u^{(k)}$ is viewed as an
approximate solution in \eqref{eq:genlassodual} at
$\lambda=\|u^{(k)}\|_\infty$, and the $k$th primal iterate
$\beta^{(k)}$ an approximate solution in \eqref{eq:genlasso} at
$t=\|D\beta^{(k)}\|_1$.  

As pointed out by a referee of this paper, there is a key relationship
between $f$ and its conjugate $f^*$ that simplifies the
update direction in \eqref{eq:genlassodir} considerably.  At step $k$,
observe that  
\begin{equation*}
\nabla f^*(-D^T u^{(k-1)}) = \nabla f^*(\nabla f(\beta^{(k-1)})) =
\beta^{(k-1)}. 
\end{equation*}
The first equality comes from the primal-dual relationship
\eqref{eq:genlassopd} at step $k-1$, and the second is due to
the fact that $x = \nabla f^*(z) \iff z = \nabla f(x)$.  
As a result, the dual update $u^{(k)}=u^{(k-1)}+\Delta$ with $\Delta$
as in \eqref{eq:genlassodir} can be written more succinctly as 
\begin{equation}
\label{eq:genlassoup}
u^{(k)} = u^{(k-1)} - \epsilon \cdot \sign(D\beta^{(k-1)}),
\end{equation}
where $\sign(\cdot)$ is to be interpreted componentwise (with the
convention $\sign(0)=0$).  Therefore, one can think of the dual
stagewise strategy as alternating between computing a dual estimate
$u^{(k)}$ as in \eqref{eq:genlassoup}, and computing a primal
estimate $\beta^{(k)}$ by solving \eqref{eq:genlassopd2}. 

We note that, since the stagewise algorithm is being run through the
dual, the estimates $\beta^{(k)}$, $k=1,2,3,\ldots$ for generalized  
lasso problems differ from those in the other stagewise
implementations encountered thus far, in that $\beta^{(k)}$,
$k=1,2,3,\ldots$ correspond to approximate solutions at 
{\it increasing} levels of regularization, 
as $k$ increases.  That is, the stagewise algorithm
for problem \eqref{eq:genlasso} begins at the unregularized end of the
path and iterates towards the fully regularized end, which is opposite
to its usual direction.  

A special case worth noting is that of Gaussian signal
approximator problems, where the loss is
\smash{$f(\beta)=\half\|y-\beta\|_2^2$}. 
For such problems, the primal-dual relationship in
\eqref{eq:genlassopd2} reduces to
\begin{equation*}
\beta^{(k)} = y-D^T u^{(k)},
\end{equation*}
for each $k$.  This means that the initialization $u^{(0)}=0$ and
$\lambda_0=0$ in the dual is the same as $\beta^{(0)}=y$ and  
$t_0=\|Dy\|_1$ in the primal.  Furthermore, it means that the dual 
updates in \eqref{eq:genlassoup} lead to primal updates that can be
expressed directly as
\begin{equation}
\label{eq:genlassogup}
\beta^{(k)} = \beta^{(k-1)} - \epsilon \cdot D^T \sign(D\beta^{(k-1)}).
\end{equation}
From the pure primal perspective, therefore,
the stagewise algorithm begins   
with the trivial unregularized estimate $\beta^{(0)}=y$, and
to fit subsequent estimates in \eqref{eq:genlassogup},
it iteratively shrinks along directions opposite to the active rows of
$D$. That is, if $D_\ell \beta^{(k-1)}>0$ (where $D_\ell$ is the
$\ell$th row of $D$), then the algorithm adds $D_\ell^T$ to
$\beta^{(k-1)}$ in forming $\beta^{(k)}$, which shrinks 
$D_\ell \beta^{(k)}$ towards zero, as $D_\ell D_\ell^T >0$ (recall
that $D_\ell$ is a row vector).  The case $D_\ell
\beta^{(k-1)}<0$ is similar. If $D_\ell \beta^{(k-1)}=0$, then no 
shrinkage is applied along $D_\ell$.    

This story can be made more concrete for fused lasso
problems, where $D$ is the edge incidence matrix of a graph:  
here the update in \eqref{eq:genlassogup} evaluates the differences 
across neighboring components of $\beta^{(k-1)}$, and for any nonzero 
difference, it shrinks the associated components towards each other   
to build $\beta^{(k)}$.  The level of shrinkage is uniform across 
all active differences, as any two neighboring components move a 
constant amount $\epsilon$ towards each 
other.\footnote{This is 
  assuming that $D$ is the edge incidence matrix of an unweighted
  graph; with edge weights, the rows of $D$ scale accordingly, and so
  the effective amounts of shrinkage in the stagewise algorithm
  scale accordingly too.}  This is a simple and natural
iterative procedure for fitting piecewise constant estimates over
graphs.  For small examples using 1d and 2d grid graphs, 
see Appendix \ref{app:fusedlasso}.

\section{Large-scale examples and practical considerations}   
\label{sec:bigexamples}

We compare the proposed general stagewise procedure to various
alternatives, with respect to both computational and statistical
performance, across the three of the four major regularization
settings seen so far.  The fourth setting is moved to the 
appendix, for space; see Appendix \ref{app:biglog}. The current
section specifically investigates 
large examples, at least relative to the small examples presented in
Sections \ref{sec:introduction}--\ref{sec:applications}.   Of course,
one can surely find room to criticize our comparisons, e.g., with
respect to a different tuning of the algorithm that computes exact
solutions, a coarser grid of regularization parameter values over
which it computes solutions, a different choice of algorithm
completely, etc.  We have tried to conduct fair comparisons in each
problem setting, but we  recognize that perfectly fair and exhaustive
comparisons are near impossible. The message that we hope to convey is
not that the stagewise algorithm is computationally superior to other   
algorithms in the problems we consider, but rather, that the
stagewise algorithm is computationally competitive with the others,
yet it is very simple, and capable of producing estimates of
high statistical quality. 

\subsection{Group lasso regression}
\label{sec:biggroup}

\noindent
{\bf Overview.}
We examine two simulated high-dimensional group lasso regression
problems. To compute group lasso solution paths, we used the   
{\tt SGL} R package, available on the CRAN repository.  This package 
implements a block coordinate descent algorithm for solving the group
lasso problem, where each block update itself applies
accelerated proximal gradient descent \citep{sgl}.  This idea 
is not complicated, but an efficient implementation of this
algorithm requires care and attention to detail, such as backtracking
line search for the proximal gradient step sizes.  The stagewise
algorithm, on the other hand, is very simple---in C++, the 
implementation is only about 50 lines of code.  Refer to
Section \ref{sec:groupreg} for a description of the stagewise update
steps. The algorithmics of the {\tt SGL} package are also written in
C++. 

\bigskip
\noindent
{\bf Examples and comparisons.}  In both problem setups, we used 
$n=200$ observations, $p=4000$ predictors, and $G=100$
equal-sized groups (of size 40).  The true coefficient vector
$\beta^* \in \R^{4000}$ was defined to be group sparse, supported on
only 4 groups, and the nonzero components were drawn 
independently from
$N(0,1)$.  We generated observations $y \in \R^{200}$ by adding
independent  $N(0,\tau^2)$ noise to $X\beta^*$, where the predictor
matrix  $X \in \R^{200 \times 4000}$ and noise level $\tau$ were
chosen under two different setups.  In the first, the entries of $X$
were drawn independently from $N(0,1)$, so that the predictors were
uncorrelated (in the population); we also let $\tau=6$.  In the
second, each row of $X$ was drawn independently from a $N(0,\Sigma)$
distribution, where $\Sigma$ had a block correlation structure.  The
covariance matrix $\Sigma$ was defined so that each predictor variable 
had unit (population) variance, but (population) correlation
$\rho=0.85$ with 99 other predictors, each from a different group.
Further, in this second setup, we used an elevated noise level $\tau=10$. 

Figure \ref{fig:big_grouplasso} shows a comparison of the group
lasso and stagewise paths, from both computational and statistical 
perspectives.  We fit group lasso solutions over 100 regularization
parameter values (the {\tt SGL} package started at the
regularized end, and used warm starts).  We also ran the stagewise
algorithm in two modes: for 250 steps with $\epsilon=1$, and for 25
steps with $\epsilon=10$.  The top row of Figure
\ref{fig:big_grouplasso} asserts that, in both the uncorrelated and
correlated problem setups, the mean squared errors of the stagewise
fits $X\beta^{(k)}$ to the underlying mean $X\beta^*$ are quite
competitive with those of the exact fits \smash{$X\hbeta(t)$}.  
In both plots, the red and black error curves, corresponding to the
stagewise fits with $\epsilon=1$ and the exact fits, respectively,
lie directly on top of each other. It took less than 1
second to compute these stagewise fits, in either problem setup;
meanwhile, it took about  
10 times this long to compute the group lasso fits in the uncorrelated
setup, and 100 times this long in the correlated setup.
The stagewise algorithm with $\epsilon=10$ took 
less than 0.1 seconds to compute a total of 25 estimates, and offers
a slightly degraded but still surprisingly competitive mean squared 
error curve, in both the correlated and uncorrelated problem setups.
Exact timings can be found in the middle row of Figure
\ref{fig:big_grouplasso}.  The error curves and
timings were all averaged over 10 draws of observations $y$ from
the uncorrelated or correlated simulation models (for fixed
$X,\beta^*$); the timings were made on a desktop personal
computer. 

\begin{figure}[htbp]
\vspace{-20pt}
\centering
\includegraphics[width=0.45\textwidth]{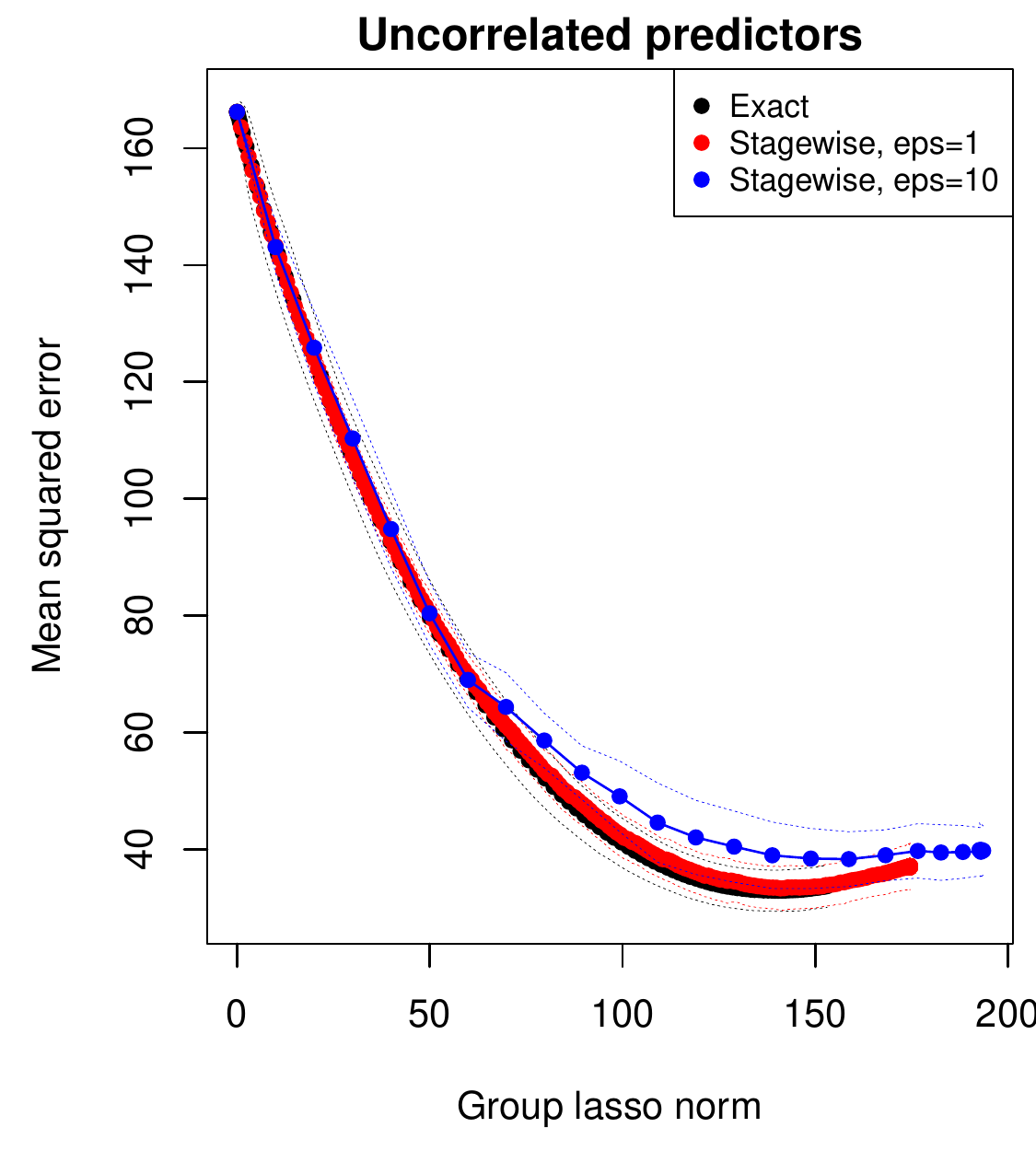}
\includegraphics[width=0.45\textwidth]{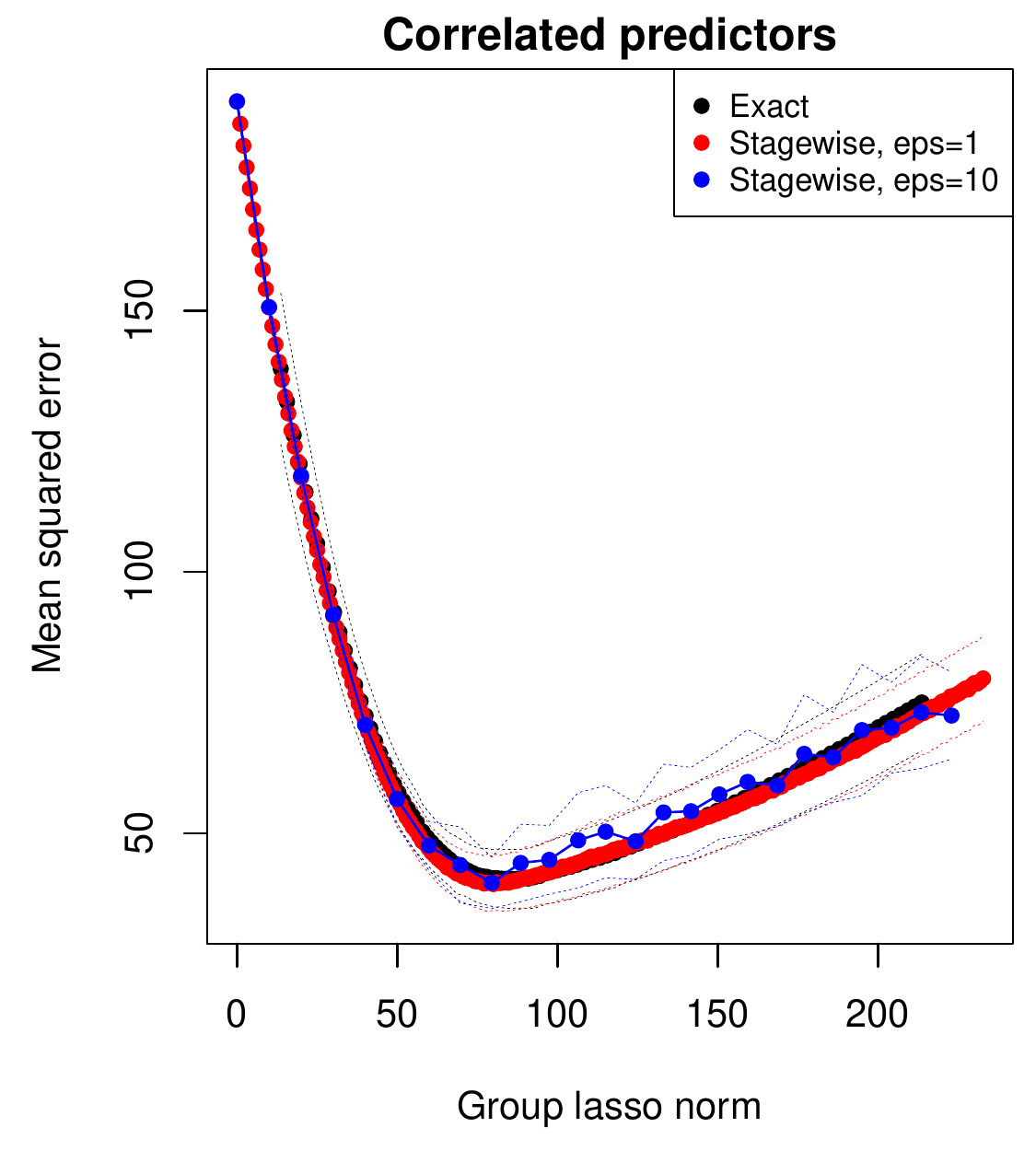} 
\vspace{10pt} \\
{\small
\begin{tabular}{|c|r|r|}
\hline
\multicolumn{3}{|c|}{Algorithm timings} \\
\hline
\hline
Method & Uncorrelated case & Correlated case \\
\hline
Exact: coordinate descent, 100 solutions & 9.08 (1.06) & 78.64 (17.92)\\
Stagewise: $\epsilon=1$, 250 estimates & 0.93 (0.00) & 0.94 (0.01) \\
Stagewise: $\epsilon=10$, 25 estimates & 0.09 (0.00) & 0.10 (0.01) \\
\hline
Frank-Wolfe: within 1\% of criterion value & 67.73 (10.37) & 92.91 (8.37) \\
Frank-Wolfe: within 1\% of mean squared error & 1.30 (0.56) & 13.17 (26.26) \\
\hline
\end{tabular}}
\vspace{20pt} \\
\includegraphics[width=0.425\textwidth]{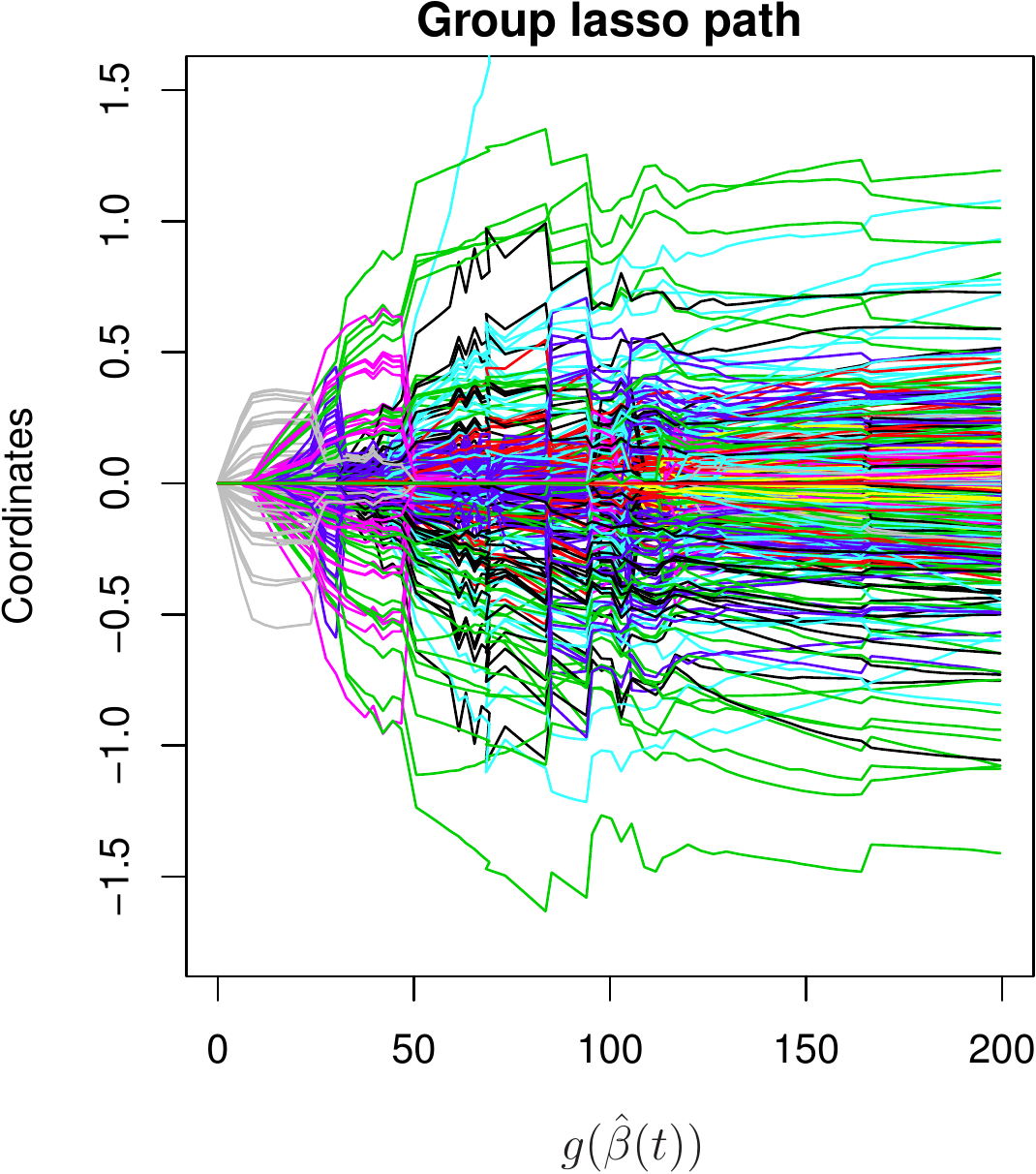}
\includegraphics[width=0.425\textwidth]{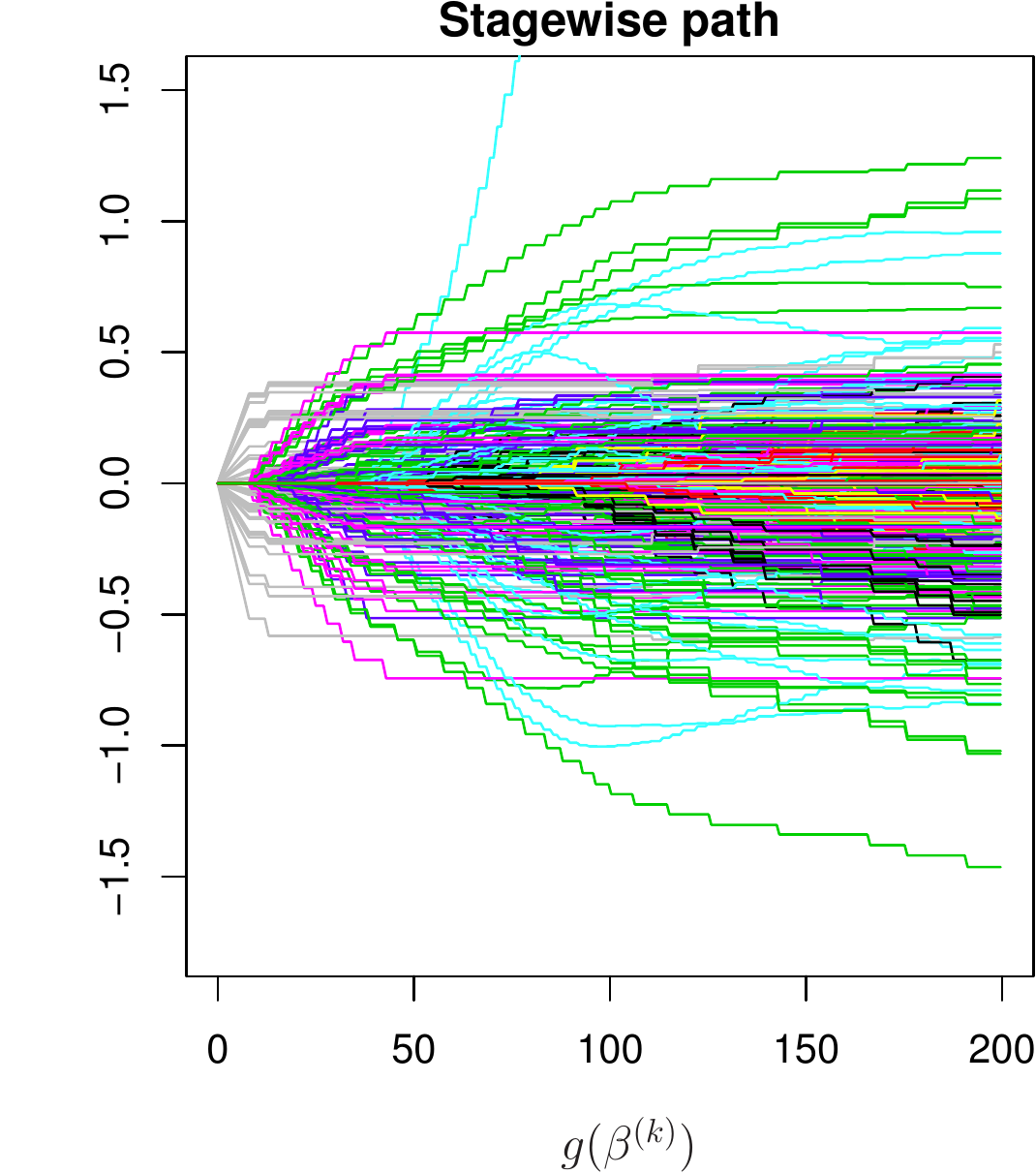} 
\caption{\it Statistical and computational comparisons between group
  lasso solutions and corresponding estimates produced by the
  stagewise approach, when $n=200$, $p=4000$.
  The top row shows that stagewise estimates
  can achieve competitive mean squared errors to that of group lasso
  solutions, as computed by coordinate descent, under two different
  setups for the predictors in group lasso regression:
  uncorrelated and block correlated.  (The curves were averaged over
  10 simulations, with standard deviations denoted by dotted
  lines.)  The middle table reports runtimes in seconds 
  (averaged over 10 simulations, with standard deviations in
  parentheses) for the various  
  algorithms considered, and shows that the stagewise algorithm
  represents a computationally attractive alternative to the {\tt SGL}
   coordinate descent approach and the Frank-Wolfe algorithm.
   Lastly, the bottom row contrasts the group lasso and
   stagewise component paths, for one draw from the
   correlated predictors setup.}     
\label{fig:big_grouplasso}
\end{figure}

Though the exact and stagewise component paths typically appear
quite similar in the uncorrelated problem setup, the same is not true
for the correlated setup. The bottom row of
Figure~\ref{fig:big_grouplasso} displays an example of the two sets of  
component paths for one simulated draw of observations, under the
correlated predictor model. The component paths of the group lasso
solution, on the left, vary wildly with the regularization
parameter; the stagewise paths, on the right, are much more stable. 
It is interesting to see that such different estimates can yield
similar mean squared errors (as, recall, shown in the top row of
Figure \ref{fig:big_grouplasso}) but this is the nature of using  
correlated predictors in a regression problem.

\bigskip
\noindent
{\bf Frank-Wolfe.} We include a comparison to the Frank-Wolfe
algorithm for computing group lasso solutions, across the same 100
regularization parameter values considered by the 
coordinate descent method.  Recall that the updates from Frank-Wolfe
share the same computational underpinnings as the stagewise ones, but 
are combined in a different manner; refer to
Appendix \ref{app:frankwolfe} for details.  We implemented the
Frank-Wolfe method for group lasso regression in C++, which starts 
at the largest regularization parameter value, and uses warm starts
along the parameter sequence.  The middle row of Figure 
\ref{fig:big_grouplasso} 
reports the Frank-Wolfe timings, averaged over 10 draws from the
uncorrelated and correlated simulation models.  We considered two
schemes for termination of the algorithm, at each regularization
parameter value $t$: the first terminates when 
\begin{equation}
\label{eq:train}
\|y-X\tbeta(t)\|_2^2 \leq 1.01 \cdot \|y-X\hbeta(t)\|_2^2,
\end{equation}
where \smash{$\tbeta(t)$} is the Frank-Wolfe iterate at $t$, and
\smash{$\hbeta(t)$} is the computed coordinate descent solution at
$t$; the second terminates when
\begin{equation}
\label{eq:test}
\|X\beta^*-X\tbeta(t)\|_2^2 \leq 1.01 \cdot \max \Big\{
\|X\beta^*-X\hbeta(t)\|_2^2, \|X\beta^*-X\beta^{(k_t)}\|_2^2 \Big\}, 
\end{equation}
where $\beta^{(k_t)}$ is the imputed stagewise estimate at the
parameter value $t$ (computed by linear interpolation of the
appropriate neighboring stagewise estimates).  In other words, the
first rule \eqref{eq:train} stops when the Frank-Wolfe iterate is
within 1\% of the criterion value achieved by the coordinate descent
solution, and the second rule \eqref{eq:test} stops when the
Frank-Wolfe iterate is within 1\% of the mean squared error of either
of the coordinate descent or stagewise fits.  Using the first rule, the
Frank-Wolfe algorithm took about 68 seconds to compute 100
solutions in the uncorrelated problem setup, and 93 seconds in
the correlated problem setup.  In terms of the total iteration count,
this meant 18,627 Frank-Wolfe iterations in the
uncorrelated case, and 25,579 in the correlated case; these numbers are
meaningful, because, recall, one Frank-Wolfe iteration is
(essentially) computationally equivalent to one stagewise iteration.
We can see that Frank-Wolfe struggles here
to compute solutions that match the accuracy of coordinate descent
solutions, especially for large values of $t$---in fact, when we
changed the factor of 1.01 to 1 in the stopping rule \eqref{eq:train},
the Frank-Wolfe algorithm converged far, far more slowly.  (For this
part, the coordinate descent solutions themselves were only computed
to moderate accuracy; we used the default convergence threshold in the
{\tt SGL} package.)  The results are  more optimistic under the second
stopping rule.  Under this rule, the 
Frank-Wolfe algorithm ran in just over 1 second (274 
iterations) in the uncorrelated setup, and about 13 seconds (3592 
iterations) in the correlated setup.  But this stopping rule
represents an idealistic situation for Frank-Wolfe, and moreover,
it cannot be realistically applied in practice, since it relies on
the underlying mean $X\beta^*$. 

\subsection{Matrix completion}
\label{sec:bigmat}

\noindent
{\bf Overview.}
We consider two matrix completion examples, one simulated and one
using real data.  To compute solutions of the matrix
completion problem, under trace norm regularization, we used the 
{\tt softImpute} R package from CRAN, which 
implements proximal gradient descent 
\citep{softimpute}. The proximal operator here
requires a truncated singular value decomposition (SVD) of a matrix
the same dimensions as the input (partially observed) matrix $Y$. SVD 
calculations are generally very expensive, but for this problem a
partial SVD can be efficiently
computed with clever schemes based on bidiagonalization 
or alternating least squares.  The {\tt softImpute}
package uses the latter scheme to compute a truncated SVD, and
though this does provide a substantial improvement over the naive
method of computing a full SVD, it is still far from cheap. 
The partial SVD computation via alternating least squares 
scales roughly quadratically with the rank of the sought solution, and
this must be repeated for every iteration taken by the algorithm until 
convergence.   

In comparison, the stagewise steps for the matrix completion 
problem require only the top left and right singular vectors of a 
matrix the same size as 
the input $Y$.  Refer back to Section \ref{sec:tracereg} for an 
explanation. To emphasize the differences between the two methods: 
the proximal gradient descent algorithm of {\tt softImpute}, at each 
regularization parameter value $t$ of interest, must iteratively
compute a partial SVD until converging on the desired solution; the
stagewise algorithm computes a single pair of left and right singular
vectors, to form one estimate at one parameter value $t$, and then
moves on to the next value of $t$.  
For the following examples, we used a simple
R implementation of the stagewise algorithm; the computational core of
the {\tt softImpute} package is also written in R.  

\bigskip
\noindent
{\bf Examples and comparisons.}
In the first example, we simulated an underlying low-rank matrix  
$B^* \in \R^{500 \times 500}$, of rank 50, by letting $B^*=UU^T$,
where $U \in \R^{500 \times 50}$ had independent $N(0,1)$ entries. 
We then added $N(0,20)$ noise, and discarded 40\% of the
entries, to form the input matrix $Y \in \R^{500 \times 500}$ 
(so that $Y$ was 60\% observed). We ran {\tt softImpute} at 
100 regularization parameter values (starting at the
regularized end, and using warm starts), and we ran two different
versions of the stagewise algorithm: one with $\epsilon=50$, for 500
steps, and 
one with $\epsilon=250$, for 100 steps.  The left plot in Figure  
\ref{fig:big_matcomp} shows the mean squared error curves 
of the resulting estimates, averaged over 10 draws of the input matrix   
$Y$ from the above prescription (with $B^*$ fixed). We can see that
the stagewise estimates, with $\epsilon=50$, trace out an essentially 
identical mean squared error curve to that from the exact solutions.
We can also see that, curiously, the larger step size
$\epsilon=250$ leads to suboptimal performance in stagewise
estimation, as measured by mean squared error.  This is unlike
the previous group lasso setting, in which a larger step size still
yielded basically the same performance (albeit slightly noisier mean
squared error curves).  

\begin{figure}[htb!]
\centering
\includegraphics[width=0.475\textwidth]{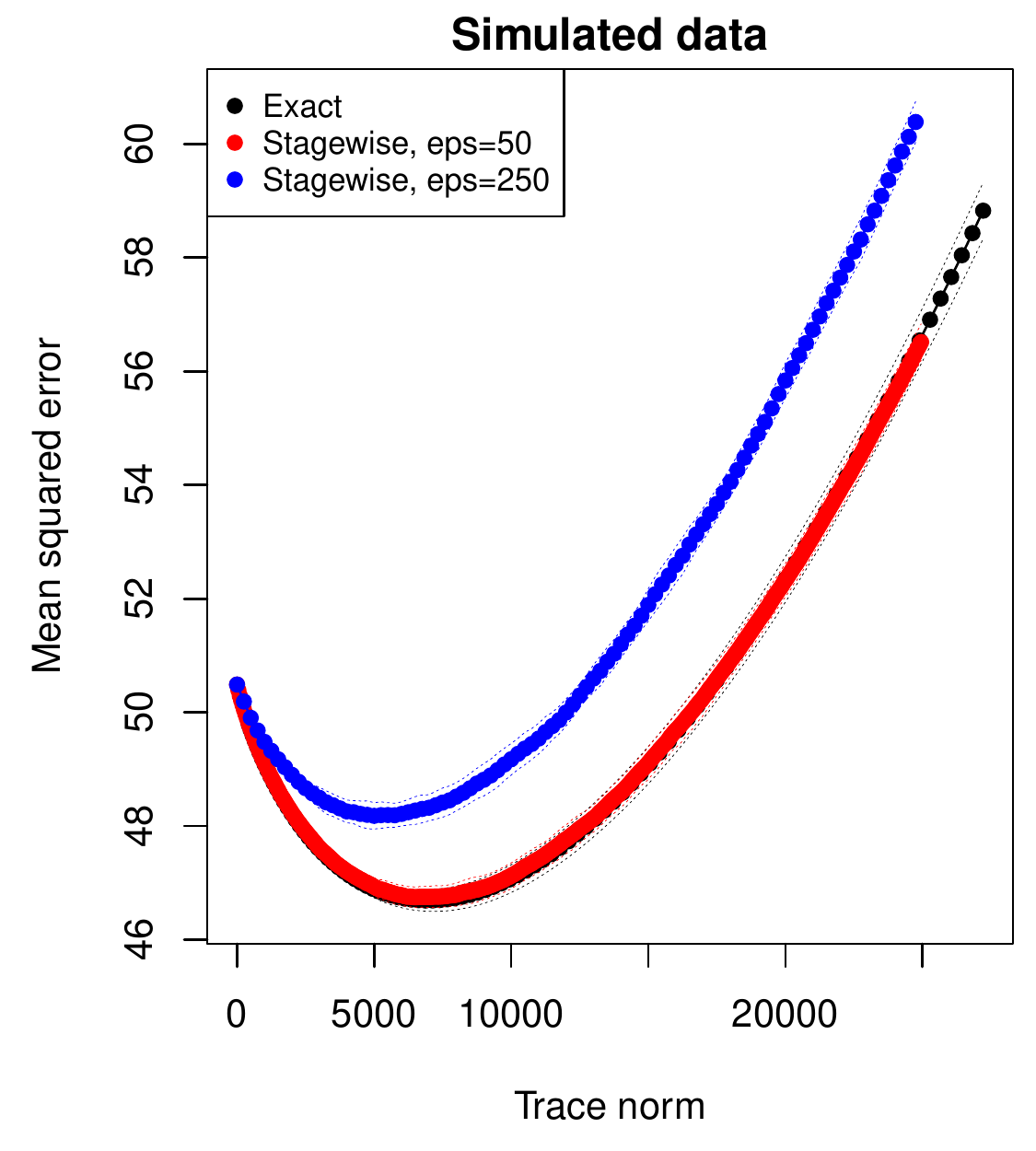}
\includegraphics[width=0.475\textwidth]{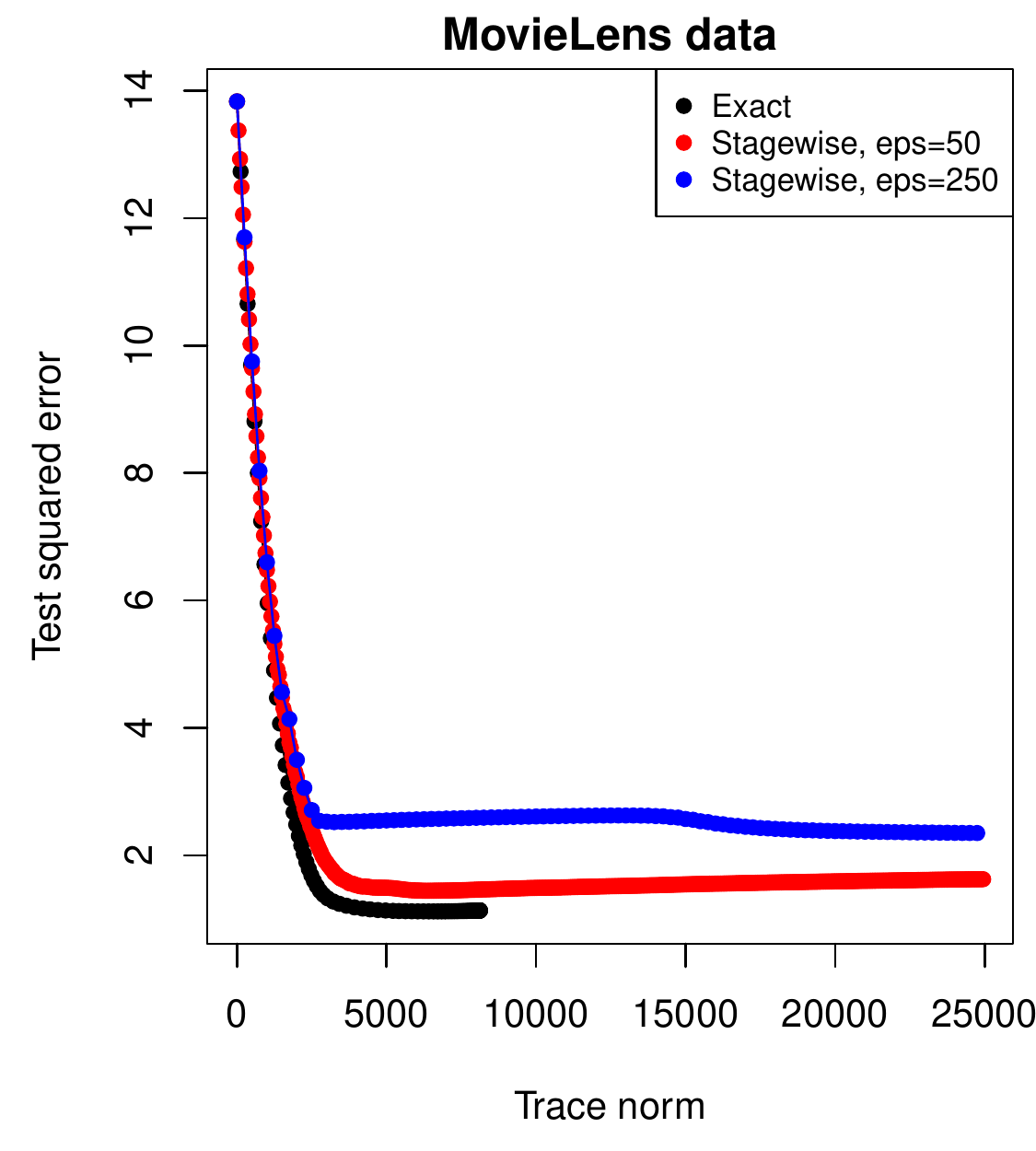}
\vspace{5pt} \\
{\small
\begin{tabular}{|c|r|r|}
\hline
\multicolumn{3}{|c|}{Algorithm timings} \\
\hline
\hline
Method & Simulated data & MovieLens data \\
\hline
Exact: proximal gradient, 100 solutions & 60.20 (1.45) & 334.67 \\
Stagewise: $\epsilon=50$, 500 estimates & 92.92 (2.42) & 107.66 \\
Stagewise: $\epsilon=250$, 100 estimates & 18.26 (0.98) & 21.22 \\
\hline
Frank-Wolfe: within 1\% of criterion value & 989.77 (19.88) & - \\
Frank-Wolfe: within 1\% of mean squared error & 154.06 (10.76) & - \\
\hline
\end{tabular}}
\caption{\it Comparisons between exact and stagewise estimates for
matrix completion problems.  The top left plot shows mean
squared error curves for a simulated example of a 40\% observed, 
$500 \times 500$ input matrix, and the right shows the same for the 
MovieLens data, where the input is 6\% observed and 
$943 \times 1682$. (The  
error curves in the left plot were averaged over 10
repetitions, and standard deviations are drawn as dotted lines.)
The stagewise estimates with $\epsilon=50$ are competitive in both  
cases. The bottom table gives the runtimes of {\tt softImpute}
proximal gradient descent, 
stagewise, and the Frank-Wolfe algorithm.
(Timings for the simulated case were averaged over 10
repetitions, with standard deviations in parentheses; Frank-Wolfe
was not run on the MovieLens example.)} 
\label{fig:big_matcomp} 
\end{figure}

In this simulated example, the proximal gradient descent method
implemented by {\tt softImpute} took an average of 206 iterations to
compute 100 solutions across 100 values of the regularization
parameter (averaged over the 10 repetitions of the observation matrix
$Y$).   This means an average of just 2.06 iterations per
solution---quite rapid convergence behavior for a first-order method  
like proximal gradient descent. (Note: we used the default convergence
threshold for {\tt softImpute}, which is only moderately small.) The
stagewise algorithms, using step sizes  
$\epsilon=50$ and $\epsilon=250$, ran for 500 and 100
iterations, respectively.  As explained, the two types of iterations
here are different in nature. Each iteration of proximal
gradient descent computes a truncated SVD, which is of roughly
quadratic complexity in the rank of current solution, and therefore
becomes more expensive as we progress down the regularization path;
each stagewise iteration computes a single pair of left and right
singular vectors, which has 
the same cost throughout the path, independent of the rank of the
current estimate.  The bottom row of Figure \ref{fig:big_matcomp} is a
table containing the running times of these two methods (averaged over
10 draws of $Y$, and recorded on a desktop computer).  We see that
proximal gradient descent spent an average of about 60 seconds to 
compute 100 solutions, i.e., 0.6 seconds per solution.  The
stagewise algorithm with $\epsilon=50$ took an average of about 93
seconds for 500 steps, and the algorithm with $\epsilon=250$ an
average of 18 seconds for 100 steps, with both translate into about
0.18 seconds per estimate. The speedy time of 0.6 seconds per
estimate of {\tt softImpute} is explained by two factors: fast
iterations (using the impressive, custom alternating least squares
routine developed by the package authors to compute partial SVDs), and
few iterations needed per solution (recall, only an average of 2.06
per solution in this example).  
The 0.18 seconds per stagewise iteration reflects the
runtime of computing leading left and right singular vectors with R's
standard {\tt svd} function, as our implementation somewhat naively
does (it does not take advantage of sparsity in any
way). This naive stagewise implementation works just fine for moderate  
matrix sizes, as in the current example.  But for larger matrix
sizes (and higher levels of missingness), we see significant 
improvements when we use a more specialized routine for computing the
top singular vectors. We also see a bigger separation in the costs per
estimate with stagewise and proximal gradient descent. This is
discussed next.   

The second example is based on the MovieLens data set (collected by
the GroupLens Research Project at the University of Minnesota, 
see \url{http://grouplens.org/datasets/movielens/}). 
We examined a 
subset of the full data set, with 100,000 ratings from $m = 943$ users
on $n=1682$ movies (hence the input 
matrix $Y \in \R^{943 \times 1682}$ was approximately 6\% observed).
We used an 80\%/20\% split of these ratings for training and testing,
respectively; i.e., we computed matrix completion estimates using the
first 80\% of the ratings, and evaluated test errors on the held out
20\% of the ratings.  
For the estimates, we ran {\tt softImpute} over 100 values of the
regularization parameter (starting at the regularized end, using warm
starts), and stagewise with $\epsilon=50$ for 500 steps, as well as
with $\epsilon=250$ for 100 steps. The right plot of Figure 
\ref{fig:big_matcomp} shows the test error curves from each of
these methods.  The stagewise estimates computed with $\epsilon=50$
and the exact solutions perform quite similarly, with the exact
solutions having a slight advantage as the trace norm exceeds about
2500.  The stagewise error curve when $\epsilon=250$ begins by
dropping off strongly just like the other two curves, but then it
flattens out too early, while the other two continue descending.
(We note that, for step sizes larger than $\epsilon=250$, the test
error curve stops decreasing even earlier, and for step sizes smaller
than  $\epsilon=50$, the error curve reaches a slightly lower minimum,
in line with that of the exact solution.  This type of behavior is 
reminiscent of boosting algorithms.)  

In terms of computation, the proximal gradient descent algorithm used
a total of 1220 iterations to compute 100 solutions in the MovieLens
example, or an average of 122 iterations per solution.  This is much
more than the 2.06 seconds per iteration as in the previous
simulated example, and it explains the longer total runtime of about
335 seconds, i.e., the longer total time of 33.5 seconds per
solution. 
The stagewise algorithms ran, by construction, for 500 and 100 steps
and took about 108 and 21 seconds, respectively, i.e., an average of
0.21 seconds per estimate.  To compute the leading left and right
singular vectors in each stagewise step here, we used the 
{\tt rARPACK} R package from CRAN, which accomodates sparse 
matrices.  This was highly beneficial because the gradient 
$\nabla f(B^{(k-1)})$ at each stagewise step was very sparse (about 
6\% entries of its were nonzero, since $Y$ was about 6\% observed).    

\bigskip
\noindent
{\bf Frank-Wolfe.} We now compare the Frank-Wolfe algorithm for 
computing matrix completion solutions, over the same 100
regularization parameter values used by {\tt softImpute}.  Each
Frank-Wolfe iteration computes a single pair of left and right top
singular vectors, just like stagewise iterations; 
see Appendix \ref{app:frankwolfe} for a general description of the
Frank-Wolfe method (or \citet{fwtrace} for a study of Frank-Wolfe
for trace norm regularization problems in particular).  We implemented 
the Frank-Wolfe algorithm for  
matrix completion in R, which starts at the regularized end of the
path, and uses warm starts at each regularization parameter value.
The timings for the Frank-Wolfe method, run on the simulated
example, are given in the table in Figure \ref{fig:big_matcomp} (we
did not run it on the MovieLens example).  As before, in the group
lasso setting, we considered two different stopping rules for
Frank-Wolfe, applied at each regularization parameter value $t$: the
first stops when the achieved criterion value is within 1\% of that
achieved by the proximal gradient descent approach in 
{\tt softImpute}, and the second stops when the achieved mean squared
error is within 1\% of either that of {\tt softImpute} or stagewise.
In either case, we cap the maximum number of iterations at 100, at
each parameter value $t$.

Under the first stopping rule, the Frank-Wolfe algorithm required an 
average of 5847 iterations to compute 100 solutions (averaged
over 10 draws of the input matrix $Y$); furthermore, this total
was calculated under the limit of 100 maximum iterations per solution,
and the algorithm met this limit at each one of the largest 50 
regularization parameter values $t$.  Recall that each one of these
Frank-Wolfe iterations is computationally equivalent to a stagewise
iteration. Accordingly, 500 steps of the stagewise algorithm, with 
$\epsilon=50$, ran in about an order of magnitude less time---93 
seconds versus 990 seconds.  The message is that the Frank-Wolfe
algorithm experiences serious difficulty in producing solutions at a
level of accuracy close to that of proximal gradient descent,
especially for lower levels of regularization. Using the
second stopping rule, Frank-Wolfe ran much faster, and
computed 100 solutions in about 997 iterations, or 154 seconds.
However, there are two important points to stress.
First, this rule is not generally
available in practice, as it depends on performance measured
with respect to the true matrix $B^*$.  Second, the termination
behavior under this rule is actually somewhat misleading, because once
the mean squared error curve begins to rise (in the left plot of
Figure \ref{fig:big_matcomp}, after about $t=7000$ in trace norm), the 
second rule will always cause Frank-Wolfe with warm starts to
trivially terminate in 1 iteration.  Indeed, in the simulated data
example, the Frank-Wolfe algorithm using this rule took about 22 
iterations per solution  before $t=7000$, and trivially 1 iteration
per solution after this point.   

\subsection{Image denoising}
\label{sec:bigimg}

{\bf Overview.} We study the image denoising
problem, cast as a generalized lasso problem
with Gaussian signal approximator loss, and 2d fused lasso or 2d total
variation regularization (meaning that the underlying graph is a 2d grid).
To compute exact solutions of this problem, we applied
a direct (noniterative) algorithm of \citet{fuseflow}, that
reduces this problem to sequence of maximum flow problems. The  
``parametric'' maximum flow approach taken by these authors is both
very elegant and highly specialized.  To the best of our knowledge,
their algorithm is one of the fastest existing algorithms for 2d fused
lasso problems (more generally, fused lasso problems over graphs).
For the simulations in this 
section we relied on a fast C++ implementation provided by the
authors (see
\url{http://www.cmap.polytechnique.fr/~antonin/software/}), which 
totals close to 1000 lines of code. The stagewise
algorithm is almost trivially simple in comparison, as our own C++
implementation requires only about 50 lines of code.  For the 2d
fused lasso regularizer, the stagewise update steps reduce to sparse
matrix multiplications; refer to Section \ref{sec:genlasso} for details. 

\bigskip
\noindent
{\bf Examples and comparisons.} We inspect two image denoising
examples. For the first, we constructed a $300 \times 200$ image
to have piecewise constant levels, and added independent $N(0,1)$
noise to the level of each pixel.  Both this true underlying image and
its noisy version are displayed in Figure \ref{fig:big_image}.  We
then ran the parametric max flow approach of \citet{fuseflow}, to
compute exact 2d fused lasso solutions, at 100 values of the
regularization parameter.  (This algorithm is direct and does not take
warm starts, so each instance was solved separately.) We also ran the
stagewise method in two modes: for 6000 steps with
$\epsilon=0.0005$, and for 500 steps with $\epsilon=0.005$.  The mean
squared error curves for each method are shown in the top left corner
of Figure \ref{fig:big_image}, and timings are given in the bottom
table. (All results here have been averaged over 10 draws of the noisy
image, and the timings were recorded on a desktop computer.)  We can
see that the stagewise estimates, both with $\epsilon=0.0005$ and
$\epsilon=0.005$, perform comparably to the exact solutions in terms of
mean squared error, though the estimates under the smaller step size
fare slightly better towards the more regularized end of the path.
The 6000 stagewise estimates using $\epsilon=0.0005$ took about 15
seconds to compute, and the 500 stagewise estimates using
$\epsilon=0.005$ took roughly 1.5 seconds.  The max flow approach
required an average of about 110 seconds to compute 100 solutions,
with the majority of computation time spent on solutions at higher
levels of regularization (which, here, correspond to lower mean
squared errors).  Finally, 
the estimate from each method that minimized mean squared error is
also plotted in Figure \ref{fig:big_image}; all look very similar and
do a visually reasonable job of 
recovering the underlying image.  That the stagewise approach can
deliver such high-quality denoised images with simple, cheap
iterations is both fortuituous and surprising.

\begin{figure}[htbp]
\vspace{-20pt}
\centering
{\renewcommand*{\arraystretch}{2.5}
\begin{tabular}{ccc}
& True image & Noisy image \\
\includegraphics[width=0.4\textwidth]{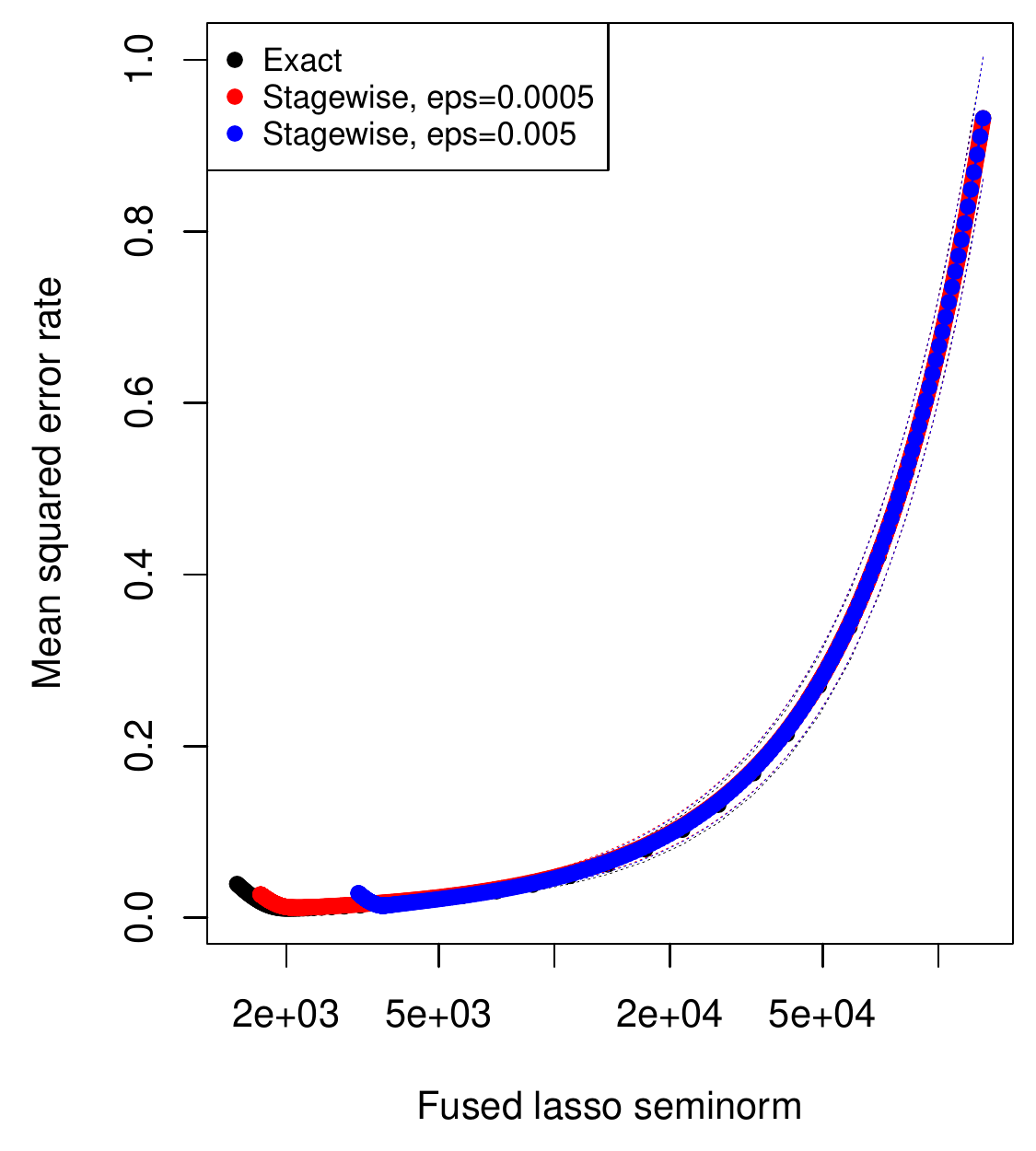} 
\hspace{5pt} &
\includegraphics[width=0.27\textwidth]{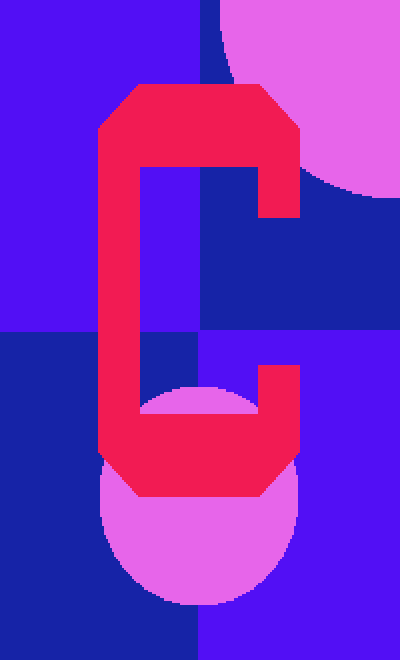} 
\hspace{5pt} &
\includegraphics[width=0.27\textwidth]{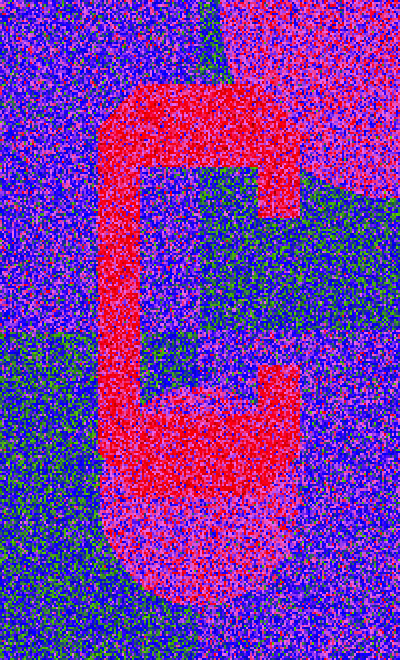} \\
\hfill
\parbox{0.27\textwidth}{\centering Exact, $t=2055.9$} & 
\parbox{0.27\textwidth}{\centering
Stagewise, $\epsilon=0.0005$, \\ 2323 steps} &
\parbox{0.27\textwidth}{\centering
Stagewise, $\epsilon=0.005$, \\ 211 steps} \\
\hfill
\includegraphics[height=0.45\textwidth]{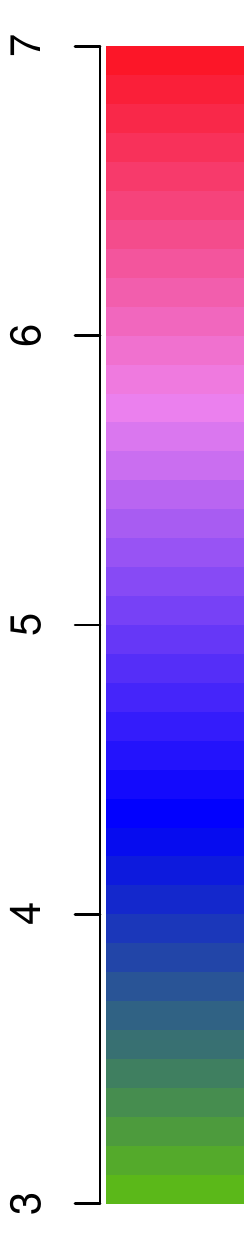}
\hspace{10pt}
\includegraphics[width=0.27\textwidth]{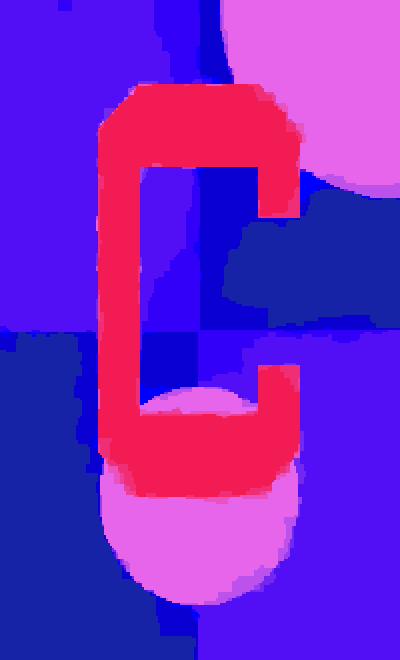} 
\hspace{5pt} &
\includegraphics[width=0.27\textwidth]{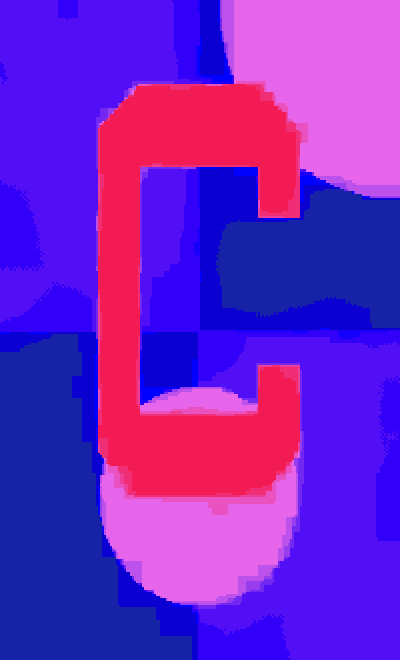} 
\hspace{5pt} &
\includegraphics[width=0.27\textwidth]{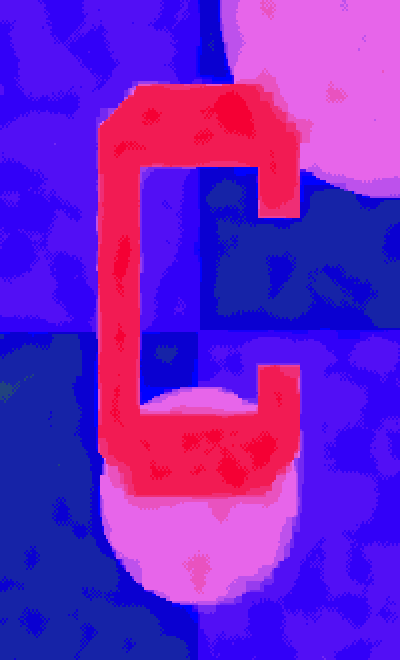} 
\end{tabular}}
\vspace{5pt} \\
{\small
\begin{tabular}{|c|r|}
\hline
\multicolumn{2}{|c|}{Algorithm timings} \\
\hline
\hline
Method & Runtime \\
\hline
Exact: maximum flow, 100 solutions & 109.04 (6.21) \\  
Stagewise: $\epsilon=0.0025$, 6000 estimates & 15.11 (0.18) \\
Stagewise: $\epsilon=0.25$, 500 estimates & 1.26 (0.02) \\
\hline
\end{tabular}}
\caption{\it Comparisons between exact 2d fused lasso solutions and
  stagewise estimates on a synthetic image denoising example. The true
  underlying $300 \times 200$ image is displayed in the middle of the
  top row.  (A color scale is applied for visualization purposes, see
  the left end of the bottom row.)  Over 10 noisy pertubations of this
  underlying 
  image, with one such example shown in the right plot of the top row,
  we compare averaged mean squared errors of the exact solutions and
  stagewise estimates, in the left plot of the top row.  Average
  timings for these methods are given in the bottom table.  (Standard
  deviations are denoted by dotted lines in the error plots, and are in 
  parentheses in the table.)  The stagewise estimates have competitive
  mean squared errors and are fast to compute.  The bottom row of plots
  shows the optimal image (i.e., that minimizing mean squared error)
  from each method, based on the single noisy image in the top right.} 
\label{fig:big_image}
\end{figure}

\begin{figure}[htbp]
\vspace{-20pt}
\centering
\begin{tabular}{cc}
\parbox{1.5in}{Original image:} &
\parbox{0.6\textwidth}{\includegraphics[width=0.6\textwidth]{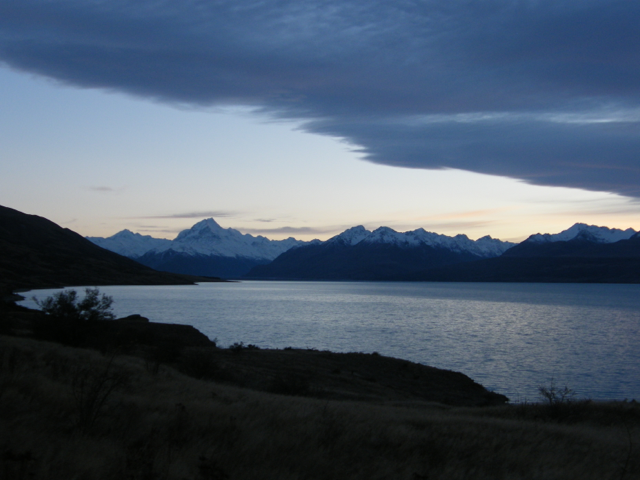}}
\\ &\\
\parbox{1.5in}{Noisy version:} &
\parbox{0.6\textwidth}{\includegraphics[width=0.6\textwidth]{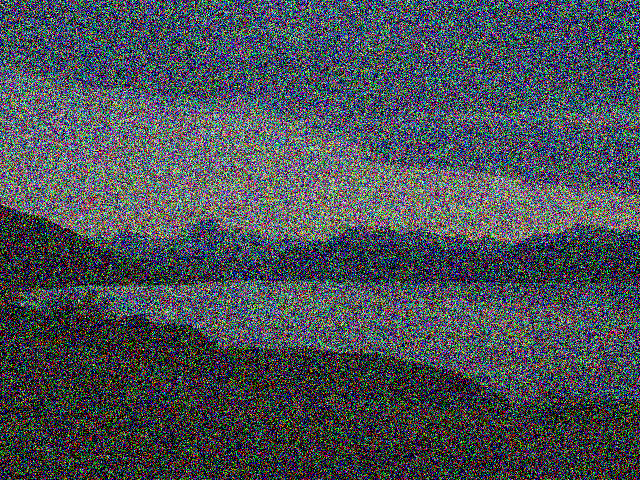}}
\\ & \\
\parbox{1.5in}{Stagewise, $\epsilon=0.001$, \\ 650 steps: 
\vspace{10pt} \\ (computed in \\ 21.34 seconds)} &
\parbox{0.6\textwidth}{\includegraphics[width=0.6\textwidth]{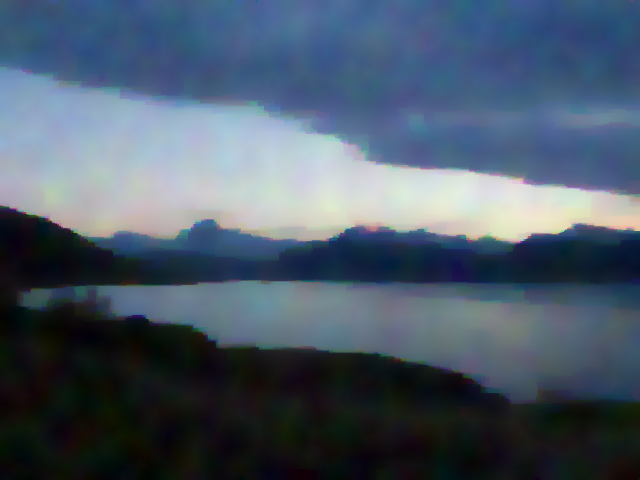}}
\end{tabular}
\vspace{5pt} \\
\caption{\it A more realistic image denoising example using 
  stagewise.  We began with a $640 \times 480$ photograph
  of Lake 
  Pukaki and Mount Cook, in New Zealand, shown at the top.  Working
  with each color channel separately, we added noise to form the
  middle noisy image, and ran the stagewise algorithm to eventually
  yield the bottom image, a nice reconstruction.}  
\label{fig:big_image_cook} 
\end{figure}

The second example considers the stagewise algorithm for a larger-scale 
image denoising task, based on a real $640 \times 480$ image, of Lake 
Pukaki in front of Mount Cook, New Zealand.  We worked with each 
color channel---red, green, 
blue---separately, and the pixel values were scaled to lie between 0 
and 1.  For each of these three images, we added independent
$N(0,0.5)$ noise to the pixel values, and ran the stagewise algorithm
with $\epsilon=0.005$ for 650 steps.  We chose this number of steps
because the achieved mean squared error (averaged over the three color 
channels) roughly began to rise after this point.  We then recombined
the three denoised images---on the red, green, blue color
channels---to form a single image.  See Figure
\ref{fig:big_image_cook}.  Visually, the reconstructed image is
remarkably close to the original one, especially considering
the input noisy image on which it is computed. The stagewise algorithm 
took a total of around 21 seconds to produce this result; recall,
though, that in this time it actually produced $650 \times 3 = 1950$
fused lasso estimates (650 steps in three different image denoising
tasks, one for each color).  

\subsection{Choice of step size}
\label{sec:stepsize}

We discuss a main practical issue when running the stagewise
algorithm: choice of the step size $\epsilon$.  Of course, when
$\epsilon$ is too small, the algorithm is less efficient, and when
$\epsilon$ is too large, the stagewise estimates can fail to span
the full regularization path (or a sizeable portion of it).  Our
heuristic suggestion therefore is to start with a large step size
$\epsilon$, and plot the progress of the achieved loss $f(x^{(k)})$
and regularizer $g(x^{(k)})$ function values across steps
$k=1,2,3,\ldots$ of the algorithm.  With a proper choice of
$\epsilon$, note that we should see $f(x^{(k)})$
monotone decreasing with $k$, as well as $g(x^{(k)})$ 
monotone increasing with $k$ (this is true of
\smash{$f(\hx(t))$} and \smash{$g(\hx(t))$} as we
increase the regularization 
parameter $t$, in the exact solution computation).  If $\epsilon$ is
too large, then it seems to be the tendency in practice that the
achieved values $f(x^{(k)})$ and $g(x^{(k)})$, $k=1,2,3,\ldots$ stop   
their monotone progress at some point, and alternate back and forth. 
Figure \ref{fig:toobig} illustrates this behavior.  Once encountered,
an appropriate response would be decrease $\epsilon$ (say, halve it), 
and continue the stagewise algorithm from the last step before this
alternating pattern surfaced.    

\begin{figure}[htb]
\centering
\includegraphics[width=0.475\textwidth]{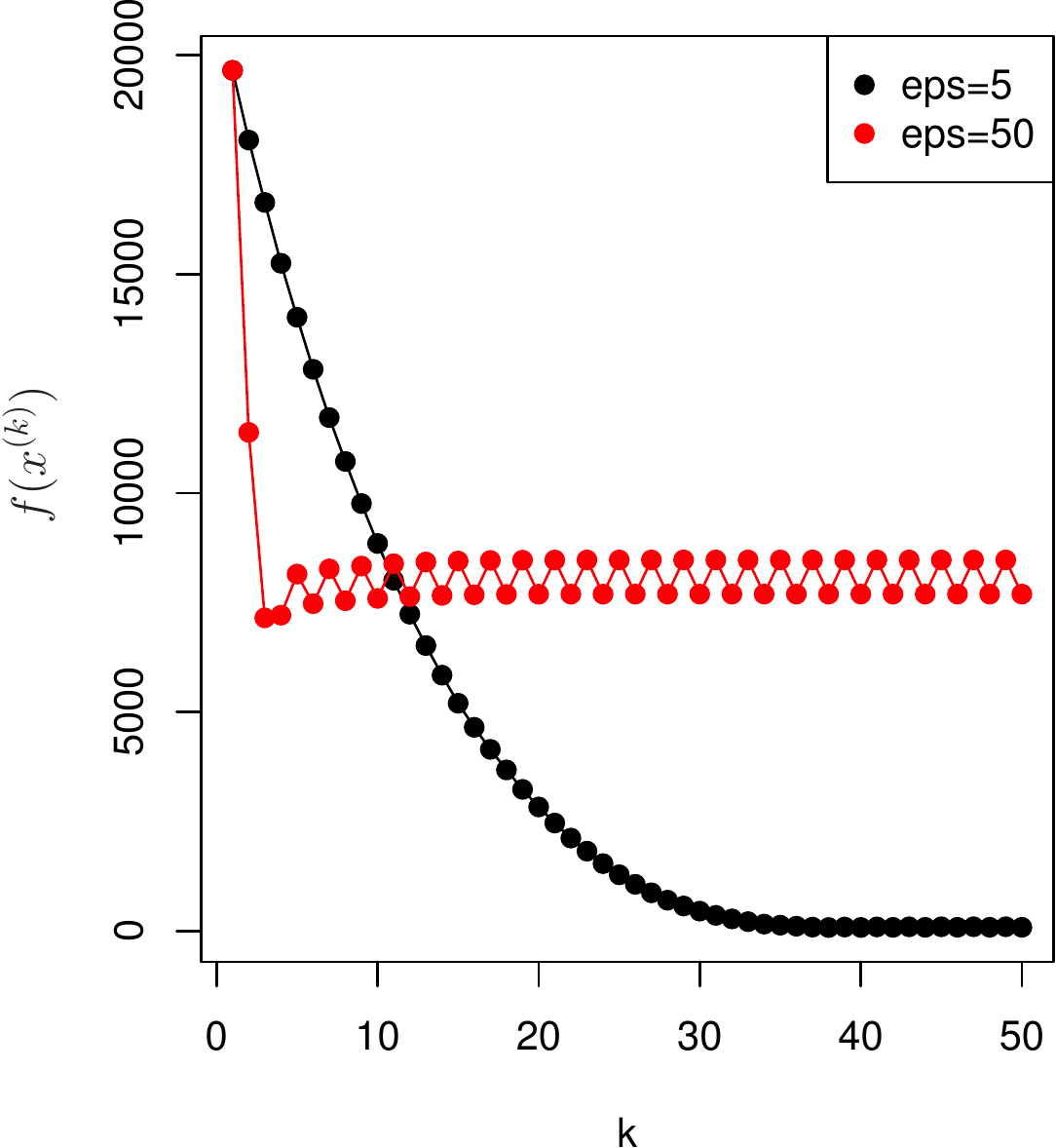}
\hspace{3pt}
\includegraphics[width=0.475\textwidth]{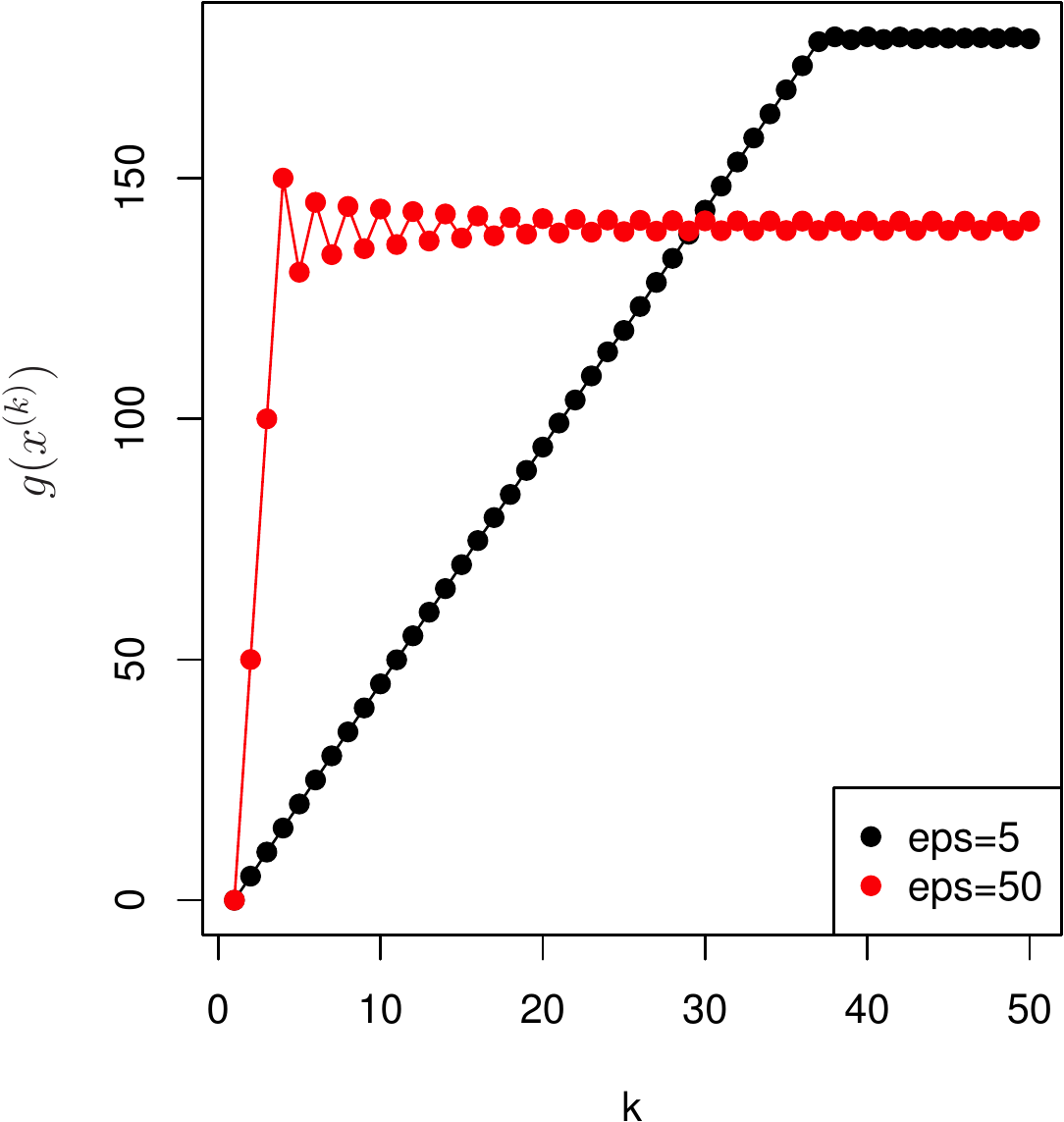}
\caption{\it An example displaying a common tendency of stagewise 
  estimates under a choice of step size $\epsilon$ that is too large.
  We used the group lasso regression data setup from Figure
\ref{fig:big_grouplasso} (uncorrelated case). Both the achieved loss
$f(x^{(k)})$ (left plot) and regularizer $g(x^{(k)})$ (right
plot) function values should be monotonic across steps
$k=1,2,3,\ldots$.  We see that for the larger step size $\epsilon=50$
(in red), progress halts and an alternating pattern begins, with both
sequences; for the smaller step size $\epsilon=5$ (in black), progress 
continues all the way until the end of the path.}   
\label{fig:toobig}
\end{figure}

The heuristic guideline above attempts to produce the largest step 
size $\epsilon$ that still produces an expansive regularization path
of stagewise estimates.  This ignores the subtlety that a larger choice
$\epsilon$ may offer suboptimal statistical performance,
even if the corresponding estimates span the full path.  This was seen
in some examples of Section \ref{sec:bigexamples} (e.g., matrix
completion, in Figure \ref{fig:big_matcomp}), but not in others (e.g.,
group lasso regression, in Figure \ref{fig:big_grouplasso}).  The
issue of tuning $\epsilon$ for optimal statistical performance is 
more complex and problem dependent. 
Although it is clearly important, we do not
study this task in the current
paper.  We mention the (somewhat obvious) point that strategies like  
cross-validation (if applicable, in the given problem setting) could
be helpful here.

\section{Suboptimality bounds for stagewise estimates} 
\label{sec:subopt}

\subsection{General stagewise suboptimality}
\label{sec:stagesub}

We present a suboptimality bound for estimates produced by the
general stagewise algorithm, restricting our attention to a norm 
regularizer $g$. 
The following result makes use of the dual norm $g^*$ of $g$ which, 
recall, is defined as $g^*(x)=\max_{g(z) \leq 1} x^T z$.  Its proof is
based on recursively tracking a duality gap for the general problem 
\eqref{eq:genprob}, and is deferred until Appendix \ref{app:stagethm}.  

\begin{theorem}
\label{thm:stagewise}
Consider the general problem \eqref{eq:genprob}, assuming that
$f$ is differentiable and convex, and $g$ is a norm.  Assume also
that $\nabla f$ is Lipschitz with respect to the pair $g^*,g$ with
constant $L$, i.e.,
\begin{equation*}
g^* \big( \nabla f(x) - \nabla f(y) \big) \leq L \cdot g(x-y), \;\;\; 
\text{all}\;\, x,y.
\end{equation*}
Fix a regularization parameter value $t$ of interest, and consider
running the general stagewise algorithm, Algorithm \ref{alg:genstage}, 
from $x^{(0)}=\hx(t_0)$, a solution in \eqref{eq:genprob} at  
a parameter value $t_0 \leq t$.  Suppose that we run the
algorithm for $k$ steps, with step size $\epsilon$, such that
$t_k = t_0 + k\epsilon = t$.  The resulting stagewise estimate 
$x^{(k)}$ satisfies
\begin{equation*}
f(x^{(k)}) - f(\hx(t)) \leq L(t^2-t_0^2) + L(t-t_0)\epsilon.
\end{equation*}
Therefore, if we consider the limiting stagewise estimate 
at the parameter value $t$, denoted 
by $\tilde{x}(t)$, as the step size $\epsilon \rightarrow 0$, then
such an estimate satisfies 
\begin{equation*}
f(\tx(t)) - f(\hx(t)) \leq L(t^2-t_0^2).
\end{equation*}
\end{theorem}

\noindent
{\it Remark 1.} In the theorem, the $k$th stagewise estimate $x^{(k)}$
is taken to be an approximate solution at the static regularization  
parameter value $t_k=t_0+k\epsilon$, not at the dynamic
value $t_k=g(x^{(k)})$, as we have been considering so far.  It is
easy to see that with the static choice $t_k=t_0+k\epsilon$, we have 
$g(x^{(k)}) \leq t_k$, so that $x^{(k)}$ is still feasible at the
parameter $t_k$. Furthermore, this choice simplifies the analysis, and 
would also simplify running the algorithm in practice (when
$g$ is expensive to compute, e.g., in the trace norm setting).   

\smallskip\smallskip
\noindent
{\it Remark 2.} The assumptions that $f$ is differentiable and that its
gradient $\nabla f$ is Lipschitz continuous are fairly standard
in the analysis of optimization algorithms; usually the Lipschitz
assumption is made with respect to a prespecified pair of primal and
dual norms, but here instead we utilize the pair naturally suggested
by the problem \eqref{eq:genprob}, namely, $g,g^*$.  For example, in
the least squares setting, \smash{$f(\beta)=\half\|y-X\beta\|_2^2$},
with an arbitrary norm $g$ as the regularizer, the Lipschitz constant
of $\nabla f$ is
\begin{equation*}
L = \max_{u \not=0} \; \frac{g^*(X^T X u)}{g(u)},
\end{equation*}
which we might write as $L=\|X^T X\|_{g,g^*}$ in the spirit of matrix
norms. 

\smallskip\smallskip
\noindent
{\it Remark 3.} The theorem can be extended to the case when $g$ is a
seminorm regularizer. As written, the Lipschitz constant $L$ would be 
infinite if $g$ has a nontrivial null space $N_g$ that overlaps with
$\nabla f$, as made precise in \eqref{eq:nullg}.  However, we could
$g^*$ redefine as    
\begin{equation*}
g^*(x) = \max_{z \in N_g^\perp, \, g(z) \leq 1} \, x^T z,
\end{equation*}
and one can then check that, under the same conditions, the proof of 
Theorem \ref{thm:stagewise} goes through just as before, but now
the bounds apply to the modified stagewise estimates in
\eqref{eq:stageupmod}, \eqref{eq:stagedirmod}.  

\subsection{Shrunken stagewise framework}
\label{sec:shrunkstage}

For reasons that will become apparent, we introduce a shrunken version
of the stagewise estimates. 

\begin{algorithm}[\textbf{Shrunken stagewise procedure}]
\label{alg:shrunkstage}
\hfill\par
\smallskip\smallskip
\noindent
Fix $\epsilon>0$, $\alpha \in (0,1)$, $t_0\in\R$. Set 
$x^{(0)}=\hx(t_0)$, a solution in \eqref{eq:genprob} at
$t=t_0$. Repeat, for $k=1,2,3,\ldots$, 
\begin{gather}
\label{eq:shrunkup}
x^{(k)} = \alpha x^{(k-1)} + \Delta, \\ 
\label{eq:shrunkdir}
\text{where}\;\,
\Delta \in \argmin_{z\in\R^n} \,\,
\langle \nabla f(x^{(k-1)}), z \rangle
\;\,\st\;\, g(z) \leq \epsilon.
\end{gather}
\end{algorithm}

The only difference between Algorithm \ref{alg:shrunkstage} and the
existing stagewise proposal in Algorithm \ref{alg:genstage} is that
the update step in \eqref{eq:shrunkup} shrinks the current iterate
$x^{(k-1)}$ by a constant amount $\alpha < 1$, before adding the
direction $\Delta$.  Note that in the case of unbounded
stagewise updates, we would replace \eqref{eq:shrunkdir} by the
subspace constrained version \eqref{eq:stagedirmod}, as explained 
in Section \ref{sec:basicprops}.  

Before we give examples or theory, we motivate the study of the
shrunken stagewise algorithm from a conceptual point of view.  It 
helps to think about lasso regression in particular, with
$f(\beta)=\half\|y-X\beta\|_2^2$ and $g(\beta)=\|\beta\|_1$. Recall
that in this case, the general stagewise procedure reduces to
classical 
forward stagewise regression, in Algorithm \ref{alg:fsr}.  A step $k$, 
forward stagewise updates the component $i$ of the estimate
\smash{$\beta^{(k-1)}$} such that the variable $X_i$ has the largest 
absolute inner product with 
the residual \smash{$y-X\beta^{(k-1)}$}; further, it moves
\smash{$\beta_i^{(k-1)}$} in a direction given by the sign of this
inner product.  It is intuitively clear why such a procedure 
generally yields monotone component paths: if $X_i$ has a   
large positive inner product with the residual, and we add a small 
amount $\epsilon$ to the $i$th coefficient, then in the next step,
$X_i$ will still have a large positive inner product with the
residual. This inner product will have been slightly decremented due
to the change in $i$th coefficient, but we will continue to increment
the $i$th coefficient by $\epsilon$ (decrement the $i$th inner
product) until another variable attains a comparable inner product
with the residual. In other words, the $i$th component path computed
by forward stagewise will increase monotonically, and eventually
flatten out.   

So how does nonmonoticity occur in stagewise paths?  Keeping with the
above thought experiment, in order for the $i$th coefficient path to
decrease at some point, the variable $X_i$ must achieve a 
{\it negative} inner product with the residual, and this must be largest in 
magnitude compared 
to the inner products from all other variables.  Given that $X_i$ had
a large positive inner product with the residual in previous 
iterations,
this seems highly unlikely, 
especially in a high-dimensional setting with many variables in total.   
But we know from many examples that the components of the exact lasso
solution path can exhibit many nonmonoticities, even very early on in
the regularization path, and even in high-dimensional
settings. To recover the exact path with a stagewise-like
algorithm, therefore, some change needs to be made to counteract the
momentum gathered over successive updates.  \citet{stagewiselasso} 
do just this, as discussed in the introduction, by adding an
explicit backward step to the stagewise routine in which coefficients
are driven towards zero as long as this decreases the loss by a
significant amount.     

An arguably simpler way to achieve a roughly similar effect is to
shrink all coefficients towards zero at each step.  This is
what is done by the shrunken stagewise method, in Algorithm
\ref{alg:shrunkstage}, via the parameter   
$\alpha<1$.  In shrunken stagewise for lasso regression, the
importance of each variable wanes over steps of the algorithm.
Thus, in the absence of attention from the stagewise update mechanism,
a coefficient path slides towards zero, 
instead of leveling off; for a coefficient path to depart from zero, or
even remain at a constant level, it
must regain the attention of the update mechanism by repeatedly  
achieving the maximal absolute inner product.  This
actually represents a fairly different philosophy from the pure
stagewise approach (with $\alpha=1$) and the two can be crudely
constrasted as follows:  
pure stagewise keeps coefficients at constant levels, unless there is 
good reason to move them away from zero; shrunken stagewise drives
coefficients to zero, unless there is good reason to keep them
on  their current trajectories. 

We give a small example of shrunken stagewise applied to lasso
regression, with $n=20$ observations and $p=10$ variables.
The rows of the predictor matrix 
$X \in \R^{20\times 10}$ were drawn independently from a Gaussian 
distribution with mean zero, and a covariance matrix having unit
diagonals and constant off-diagonals $\rho=0.8$.  The underlying
coefficient vector $\beta^* \in \R^{10}$ had dense support, with all
entries drawn from $N(0,1)$, and the observations $y$ were formed by
adding independent $N(0,1)$ noise to $X\beta^*$.  Figure
\ref{fig:shrunkstage123} shows the exact lasso solution path on the
left panel, the stagewise path in the middle panel, and the shrunkage 
stagewise path on the right.  We can see that, at various points,
components of the exact lasso path become nonmonotone, and as
expected, the corresponding the stagewise component paths ignore this
trend and level out.   The shrunken stagewise component paths
pick up on the nonmonoticities and actually mimick the exact ones
quite closely.  We note that the stagewise and shrunken stagewise
algorithms were not run here for efficiency, but were run at fine
resolution to reveal their limiting behaviors; both used a small step
size $\epsilon=0.0001$, and the latter used a shrinkage factor
$\alpha=\epsilon/10$.  The two required 100,000 and 500,000 steps,
respectively.

\begin{figure}[htb]
\centering
\includegraphics[width=0.325\textwidth]{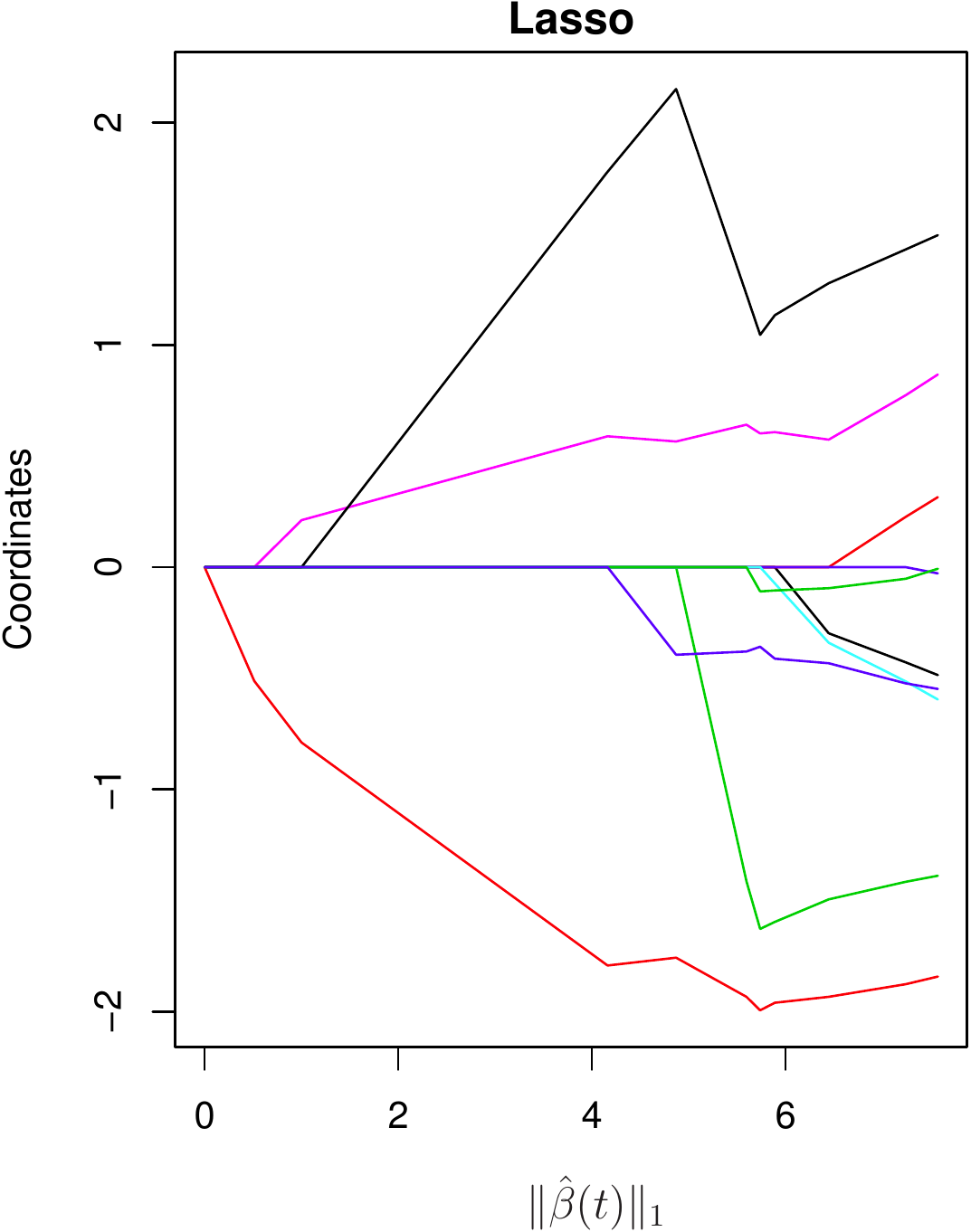}
\includegraphics[width=0.325\textwidth]{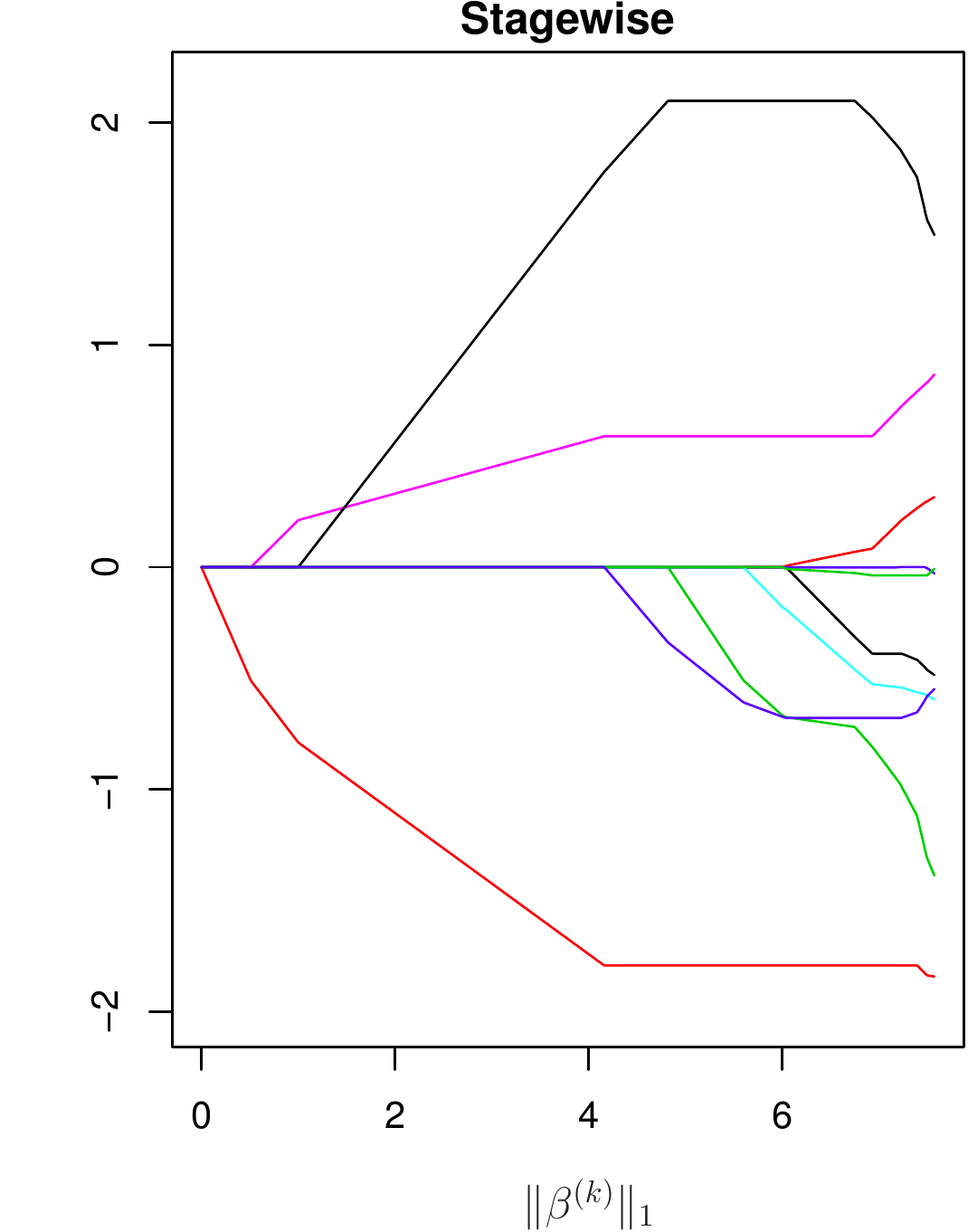}
\includegraphics[width=0.325\textwidth]{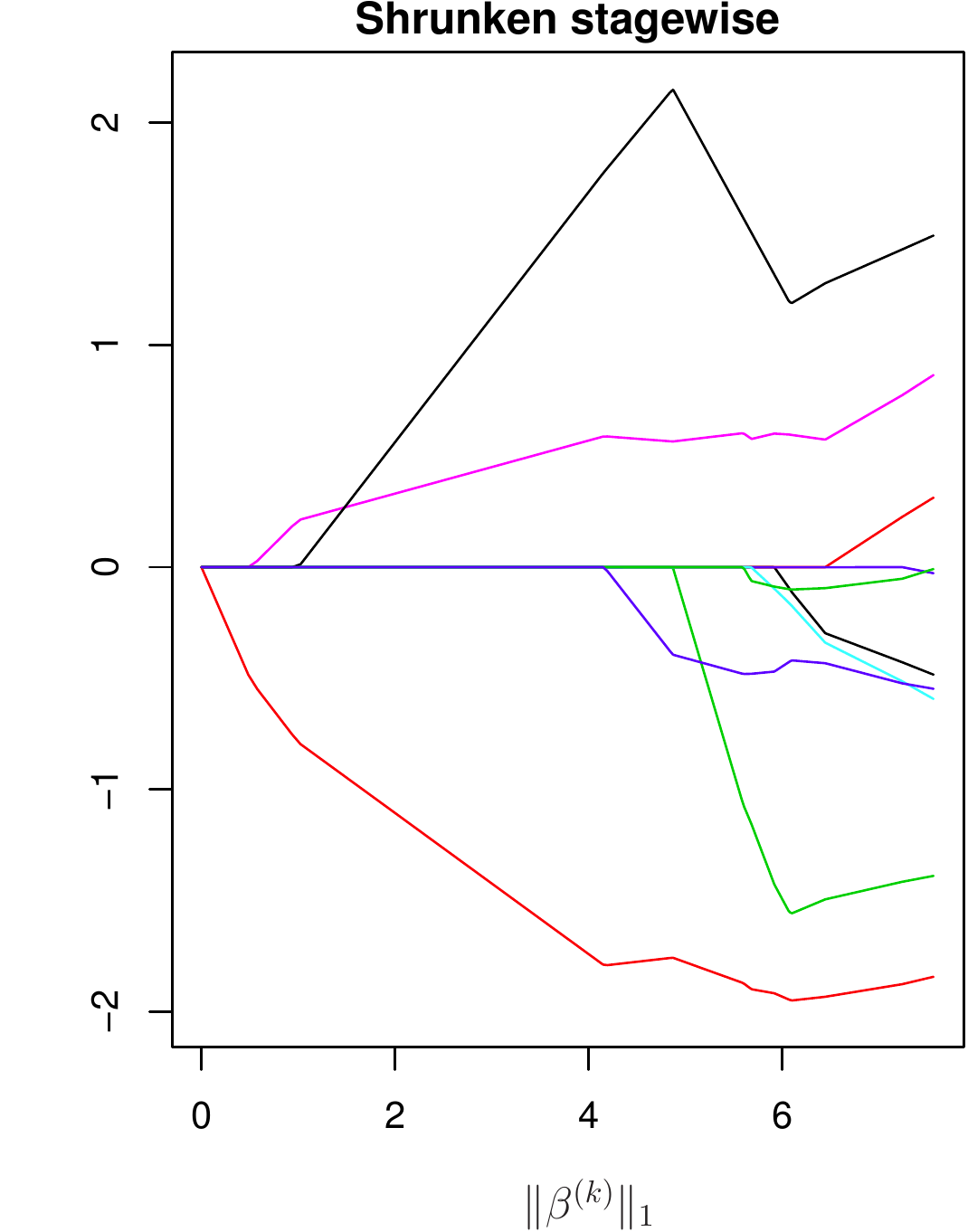}
\caption{\it Exact, stagewise, and shrunken stagewise paths for a
  small lasso regression problem with $n=20$ observations, and $p=10$ 
  correlated predictors.  When components of the lasso solution path
  become nonmonotone (e.g., top black path, and bottom red path), the 
  corresponding stagewise ones are more stable and remain at a 
  constant level, but shrunken stagewise matches the
  nonmonotonicities.} 
\label{fig:shrunkstage123}
\end{figure}

To be upfront, we remark that the shrunken stagewise method is
{\it not computationally efficient approach}, and we do not advocate
its use in practice.   
The stagewise algorithm in the above example could have been run, 
e.g., with $\epsilon=0.01$ and for 100 steps, and this would have
yielded a sequence of estimates with effectively the same pattern.
But to capture the nonmonotonicities present in the exact
solution path, larger step sizes do not suffice for shrunken
stagewise, and the algorithm needs to be run with $\epsilon=0.0001$
and for 500,000 steps---this is clearly not desirable for such a small
example with $n=20$ and $p=10$, and it does not bode well for
scalability. 
We will see in what follows that the shrunken stagewise estimates
provide a bridge between pure stagewise estimates and exact solutions 
in the general convex regularization problem \eqref{eq:genprob}.
Hence we view the shrunken stagewise estimates as interesting and 
worthwhile {\it because} they provides this connection. 

The main reason we choose to study the shrinkage strategy in
Algorithm \ref{alg:shrunkstage}, as opposed to, say, backward steps,
is that the shrinkage approach applies outside of the lasso
regularization setting; as far as we can tell, there is 
no natural analog of backwards steps beyond the sparse setting.  In
fact, in the general problem setup, the shrinkage factor $\alpha$ in
Algorithm \ref{alg:shrunkstage} somewhat roughly mirrors what is done
by Frank-Wolfe (this is really a different strategy, but still, it is one
that computes exact solutions; compare equations \eqref{eq:shrunkdir}
and \eqref{eq:fwdir} from Appendix \ref{app:frankwolfe}).  A general
interpretation of the shrinkage operation in \eqref{eq:shrunkdir} is
that it lessens the dependence of the stagewise estimates on the
computed history, i.e., decreases the stability of the computed
stagewise component paths, and implicitly allows for more weight to be
placed on the local update directions.  Empirical examples with, e.g., group
lasso regression or matrix completion confirm that shrunken
stagewise estimates can be tuned to track the exact solution path even
when the pure stagewise path deviates from it.  We do not examine
these cases here but instead turn to theoretical development.

\subsection{Shrunken stagewise suboptimality}

As in Section \ref{sec:stagesub}, we assume that $g$ is a norm, and  
write $g^*$ for its dual norm.  We also consider the $k$th
shrunken stagewise estimate $x^{(k)}$ as an approximate solution in
the general problem \eqref{eq:genprob} at a static value of the
regularization parameter, defined recursively as $t_k=\alpha t_{k-1} +
\epsilon$.  A straightforward inductive argument shows that
$g(x^{(k)}) \leq t_k$, i.e., the estimate $x^{(k)}$ is feasible for
the problem \eqref{eq:genprob} at $t=t_k$.  
Under this setup, the same limiting suboptimality bound as in Theorem
\ref{thm:stagewise} can be established for the shrunken stagewise
estimates.  For the sake of space, we do not present this result.
Instead we show that, under additional conditions, the shrunken
stagewise estimates overcome the stability inherent to
stagewise, and achieve the idealized behavior suggested by Figure
\ref{fig:shrunkstage123}, i.e., they converge to exact solutions along
the path.  See Appendix \ref{app:shrunkthm} for the proof.   

\begin{theorem}
\label{thm:shrunkstage}
Consider the general problem \eqref{eq:genprob}. 
Assume, as in Theorem \ref{thm:stagewise}, that the loss function $f$  
is differentiable and convex, the regularizer $g$ is a norm, and
$\nabla f$ is Lipschitz with respect to $g^*,g$, having Lipschitz
constant $L$. Fix a parameter value $t$, and consider running
the shrunken stagewise algorithm, Algorithm \ref{alg:shrunkstage},
from $x^{(0)}=\hx(t_0)$, a solution in \eqref{eq:genprob} at a
parameter value $t_0 \leq t$.  Consider the limiting estimate $\tx(t)$
at the parameter value $t$, as both $\epsilon \rightarrow 0$ and
$\alpha \rightarrow 1$.  Suppose that 
\begin{equation*}
\frac{1-\alpha}{\epsilon} \rightarrow 0 
\;\;\; \text{and} \;\;\;
\frac{1-\alpha}{\epsilon^2} \rightarrow \infty.
\end{equation*}
Let $k=k(\epsilon,\alpha)$ denote the number of steps taken by the  
shrunken stagewise algorithm to reach the parameter value
$t_k=t$; note that $k \rightarrow \infty$ as 
$\epsilon \rightarrow 0$, $\alpha \rightarrow 1$.  Define the  
effective Lagrange parameters 
\smash{$\lambda_i = g^*(\nabla f(x^{(i)}))$}, 
$i=1,\ldots k$, and assume that these parameters exhibit a weak 
type of decay: 
\begin{equation}
\label{eq:decay}
\begin{aligned}
\lambda_i/t_i &\geq CL, \;\;\; i=1,\ldots r-1, \\
\lambda_r/t_r &\leq \frac{(C+1) \theta^2 - 2}{2}L,
\end{aligned}
\end{equation}
for some $r < k$, with $r/k \rightarrow \theta \in (0,1)$, and some
constant $C$. Then the limiting shrunken stagewise 
estimate $\tx(t)$ at the parameter value $t$, as 
$\epsilon \rightarrow 0$ and $\alpha \rightarrow 1$, satisfies 
\begin{equation*}
f(\tx(t)) = f(\hx(t)),  
\end{equation*}
i.e., $\tx(t)$ is a solution in \eqref{eq:genprob} at the parameter
value $t$. 
\end{theorem}

\noindent
{\it Remark 1.} The result above can be extended to the case when $g$
is a seminorm.  We simply need to redefine $g^*$ and the updates in
order to accomodate the possibly nontrivial null space $N_g$ of $g$,
as discussed in the third remark following Theorem
\ref{thm:stagewise}.

\smallskip\smallskip
\noindent
{\it Remark 2.} The assumption in \eqref{eq:decay} of Theorem
\ref{thm:shrunkstage} stands out as technical assumption that is hard
to interpret.  This condition is used in the proof to control 
a term in the duality gap expansion that involves
differences of \smash{$g^*(\nabla f(x^{(i)}))$} across successive
iterations $i,i+1$.  The theorem refers to such a quantity, 
\smash{$\lambda_i = g^*(\nabla f(x^{(i)}))$}, as the ``effective
Lagrange parameter'' at $x^{(i)}$. To explain this, consider the
stationarity condition for  
the problem \eqref{eq:genprob}, 
\begin{equation*}
\nabla f(x) + \lambda v = 0,
\end{equation*}
where $v \in \partial g(x) = \argmax_{g^*(z) \leq 1} \,
x^T z$.  This implies that $\nabla f(x) = -\lambda v$, or
$g^*(\nabla f(x)) = \lambda g^*(v) = \lambda$, which gives an
expression for the Lagrange parameter associated with a solution  
of the constrained problem \eqref{eq:genprob}.  As $x^{(i)}$
is not a solution, but an approximate one, we call
\smash{$\lambda_i=g^*(\nabla f(x^{(i)}))$} its effective Lagrange
parameter. 

The condition \eqref{eq:decay} says that until some number of steps
$r$ along the path, the ratio of effective Lagrange parameters
$\lambda_i$ to bound parameters $t_i$ must not be too small, and then
at step $r$ it must not be too large. This is a formulation of a type
of weak decay of $\lambda_i/t_i$, $i=1,2,3,\ldots$.  It is not
intuitively clear to us when (i.e., in what 
kinds of problems) we should expect this condition to be satisfied.  
We can, however, inspect it emprically.  For the example lasso problem 
in Figure \ref{fig:shrunkstage123} (where, recall, the
shrunken stagewise path appears to approach the exact solution path), 
we plot the ratio $\lambda_i/t_i$, $i=1,2,3,\ldots$ in Figure
\ref{fig:shrunkstage4}.  This ratio displays a sharp decay across
steps of the algorithm, and so, at least empirically, the assumption
\eqref{eq:decay} seems reasonable.  We suspect that in general, the 
two hard bounds in \eqref{eq:decay} can be replaced by a more natural
decay condition, and furthermore, there are characterizable problem
classes with sharp decays of the Lagrange to bound parameter ratios.
These are topics for future work.

\begin{figure}[htb]
\centering
\includegraphics[width=0.475\textwidth]{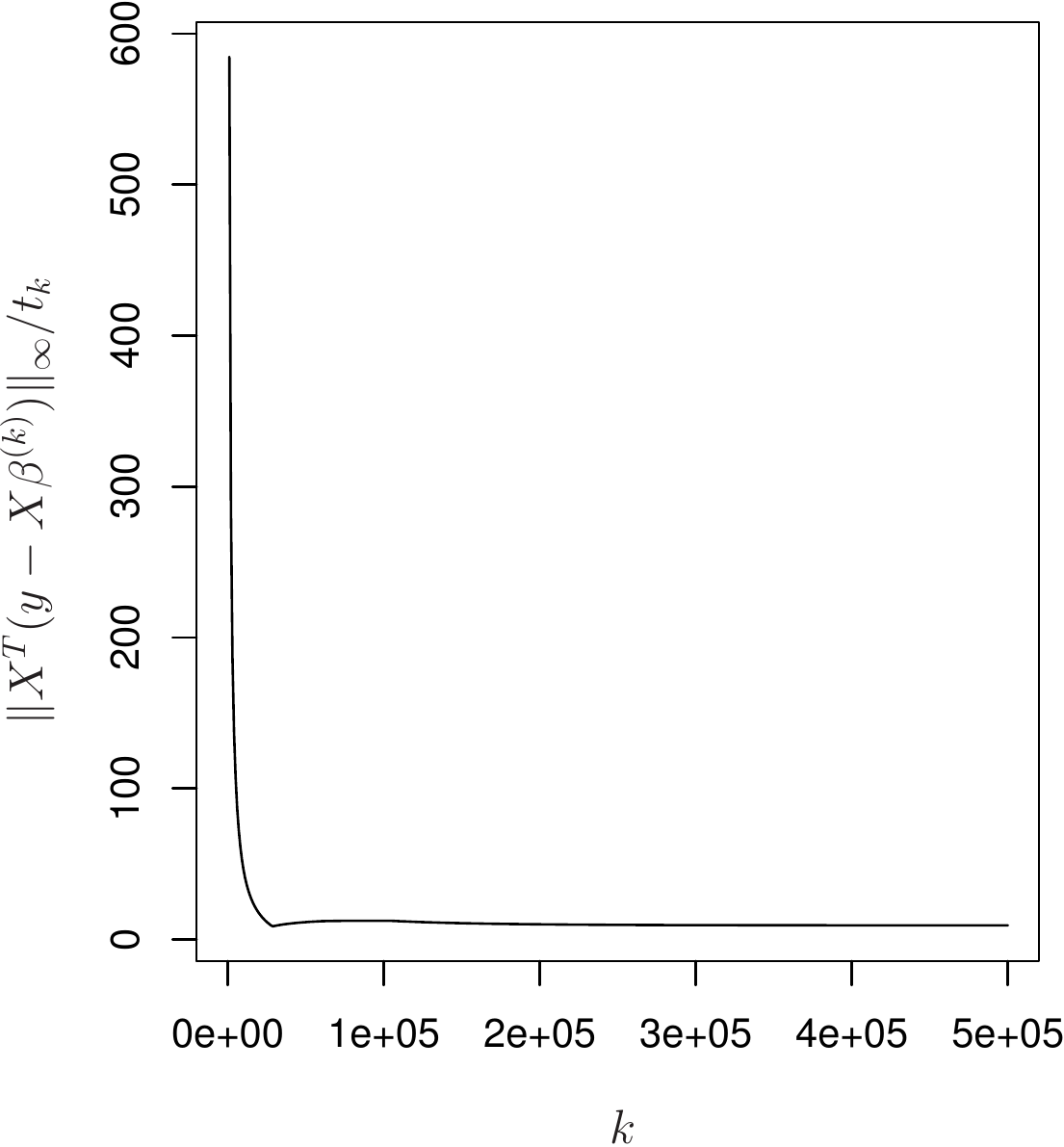}
\caption{\it A plot of
  \smash{$\lambda_k/t_k=\|X^T(y-X\beta^{(k)})\|_\infty/t_k$} across
  steps $k$ of the shrunken stagewise algorithm, for the lasso data
  set of Figure \ref{fig:shrunkstage123}.  This decay 
  roughly verifies the condition \eqref{eq:decay}
  of Theorem~\ref{thm:shrunkstage}, needed to ensure the
  convergence of shrunken stagewise estimates to exact solutions.} 
\label{fig:shrunkstage4}
\end{figure}




\section{Discussion}
\label{sec:discussion}

We presented a framework for computing incremental stagewise
paths in a general regularized estimation setting, defined by
minimizing a differentiable convex loss function subject to a convex
constraint.  The stagewise estimates are explicitly and efficiently
computable for a wide variety of problems, and they
provide an approximate solution path for the underlying convex
problem of interest, but exhibit generally more stability as the
regularization parameter changes.  In some situations this
approximation (i.e., the discrepancy between stagewise estimates and
solutions) appears empirically to be quite tight, and in others it
does not.  All in all, however, we have found that the
stagewise estimates essentially always offer competitive statistical
performance (as measured, e.g., by test error) with that of exact
solutions. This suggests that they should be a point of study, even
apart from their ability to approximate solution paths of convex
problems, and a rigorous (theoretical) characterization of the 
statistical properties of stagewise estimates is an important
direction to pursue in the future.  There are many other potential
topics for future work, as alluded to throughout the paper.  It is our
hope that other researchers will take an interest too, and that this
paper marks the beginning of a deeper understanding of stagewise 
capabilities. 



\subsubsection*{Acknowledgements}

This work was motivated by an attempt to explain the intuitive
connection between forward stagewise regression and the lasso, in
preparing lectures for a graduate class on optimization at Carnegie
Mellon University. We thank co-teacher Geoff
Gordon and the students of this class for early inspiring
conversations.  We also thank Rob Tibshirani, Jerry
Friedman, Jonathan Taylor, Jacob Bien, and Lester 
Mackey for their helpful feedback.  We are
grateful to Jacob Bien for his understanding and patience
throughout our (unusually slow) writing process, and to Lester
Mackey for enlightening discussion on the Frank-Wolfe connection.
Lastly, we would like to thank the editors and referees who reviewed
this paper, as they provided extremely helpful and constructive
reports. 

\newpage
\appendix
\allowdisplaybreaks

\section{Appendix}

\subsection{Comparison to Frank-Wolfe}
\label{app:frankwolfe}

We compare our general stagewise procedure
to the Frank-Wolfe algorithm for the general convex minimization
problem \eqref{eq:genprob}.
At any fixed value of $t$, the Frank-Wolfe algorithm begins with
$\tx^{(0)} \in \R^n$ such that $g(\tx^{(0)}) \leq t$, and repeats the
following steps \citep{frankwolfe,jaggi}:   
\begin{gather}
\label{eq:fwup}
\tx^{(k)} = (1-\gamma) \tx^{(k-1)} + \gamma \tilde{\Delta}, \\
\text{where}\;\,
\label{eq:fwdir}
\tilde{\Delta} \in \argmin_{z \in \R^n} \,
 \langle \nabla f(\tx^{(k-1)}), z \rangle
\;\,\st\;\, g(z) \leq t, \\
\label{eq:fwshrink}
\text{and}\;\, \gamma = 2/(k+1),
\end{gather}
for $k=1,2,3,\ldots$.
The Frank-Wolfe steps can be seen as iteratively minimizing local
linear approximations of the loss function $f$ over the constraint
set $\{x : g(x) \leq t\}$, as is done in \eqref{eq:fwdir}.  The actual
updates performed in \eqref{eq:fwup} take successively smaller and 
smaller steps in the direction of these local minimizers.
Under fairly weak conditions, the Frank-Wolfe iterates satisfy
$f(\tx^{(k)}) \rightarrow f(\hx(t))$ as $k\rightarrow \infty$; in
fact, as shown in, e.g., \citet{jaggi}, the error
$f(\tx^{(k)})-f(\hx(t))$ is $O(1/k)$. \citet{jaggi} also shows how to
use the Frank-Wolfe iterates to easily 
compute a duality gap for the problem \eqref{eq:genprob}, so in 
practice we could stop iterating when this duality gap is sufficiently
small. 

At face value, the Frank-Wolfe steps
\eqref{eq:fwup}, \eqref{eq:fwdir}, \eqref{eq:fwshrink} and the
stagewise steps \eqref{eq:stageup}, \eqref{eq:stagedir} appear 
very similar.  One apparent difference is that the
former steps are iterated to ultimately yield a single estimate at
a given value of the regularization parameter $t$, whereas the
latter steps are iterated to yield several estimates (one 
per iteration) that form a regularization path. We make 
more substantial and informative comparisons between the two 
methods below.

First, consider a setting in which we run the Frank-Wolfe
algorithm multiple times, in order to compute estimates at multiple
values of the regularization parameter $t$; a typical strategy
would be to run the Frank-Wolfe algorithm until convergence at each
desired value of $t$, using ``warm starts'' (i.e., at the end of 
each run, we would use the newly computed estimate as the initial
guess $\tx^{(0)}$ in the Frank-Wolfe algorithm at the next 
parameter value). With this in mind, it may be
tempting to compare our stagewise algorithm to something like a 
1-step Frank-Wolfe algorithm, where at each regularization parameter
value, we perform a single Frank-Wolfe update to 
construct our estimate, rather than iterating the algorithm until
convergence.  But a more careful examination shows that these two
approaches, the stagewise and 1-step Frank-Wolfe approaches, are
actually quite different.  To make the comparison as direct as
possible, assume that the 1-step Frank-Wolfe procedure starts with
$x^{(0)}=\hx(t_0)$, a solution in \eqref{eq:genprob} at
$t=t_0$. It would then compute estimates $x^{(k)}$, $k=1,2,3,\ldots$
at the regularization parameter values $t_k=t_{k-1}+\epsilon$,
$k=1,2,3,\ldots$ via 
\begin{gather}
\label{eq:fwup2}
x^{(k)} = \tilde{\Delta}, \\
\label{eq:fwdir2}
\text{where}\;\, \tilde{\Delta} \in \argmin_{z\in\R^n} \, 
\langle \nabla f(x^{(k-1)}), z \rangle \;\, \st \;\, g(z) \leq t_k,
\end{gather}
which is just a single step of the Frank-Wolfe algorithm at $t=t_k$,
taking as the initial guess $x^{(k-1)}$. We can see that  
both the 1-step Frank-Wolfe \eqref{eq:fwup2}, \eqref{eq:fwdir2} 
and stagewise \eqref{eq:stageup}, \eqref{eq:stagedir} updates utilize
a local linearization of $f$ around previous estimate, and minimize
this linear function over a sublevel set of $g$, but they do so in
subtly different ways.  
The 1-step Frank-Wolfe approach takes $x^{(k)}$ to be a minimizer of 
$\langle \nabla f(x^{(k-1)}), z \rangle$ over the full constraint set
$\{z : g(z) \leq t_k\}$; the stagewise approach computes a minimizer
of $\langle \nabla f(x^{(k-1)}), z \rangle$ over a highly restricted
constraint set $\{z : g(z) \leq \epsilon\}$, and adds this to the
last estimate $x^{(k-1)}$ to form $x^{(k)}$.  In both cases, the
estimate $x^{(k)}$ is a feasible point for problem
\eqref{eq:genprob} at $t=t_k$. See Figure \ref{fig:fw} for an
illustration.   

\begin{figure}[htb]
\centering
\includegraphics[height=0.485\textwidth]{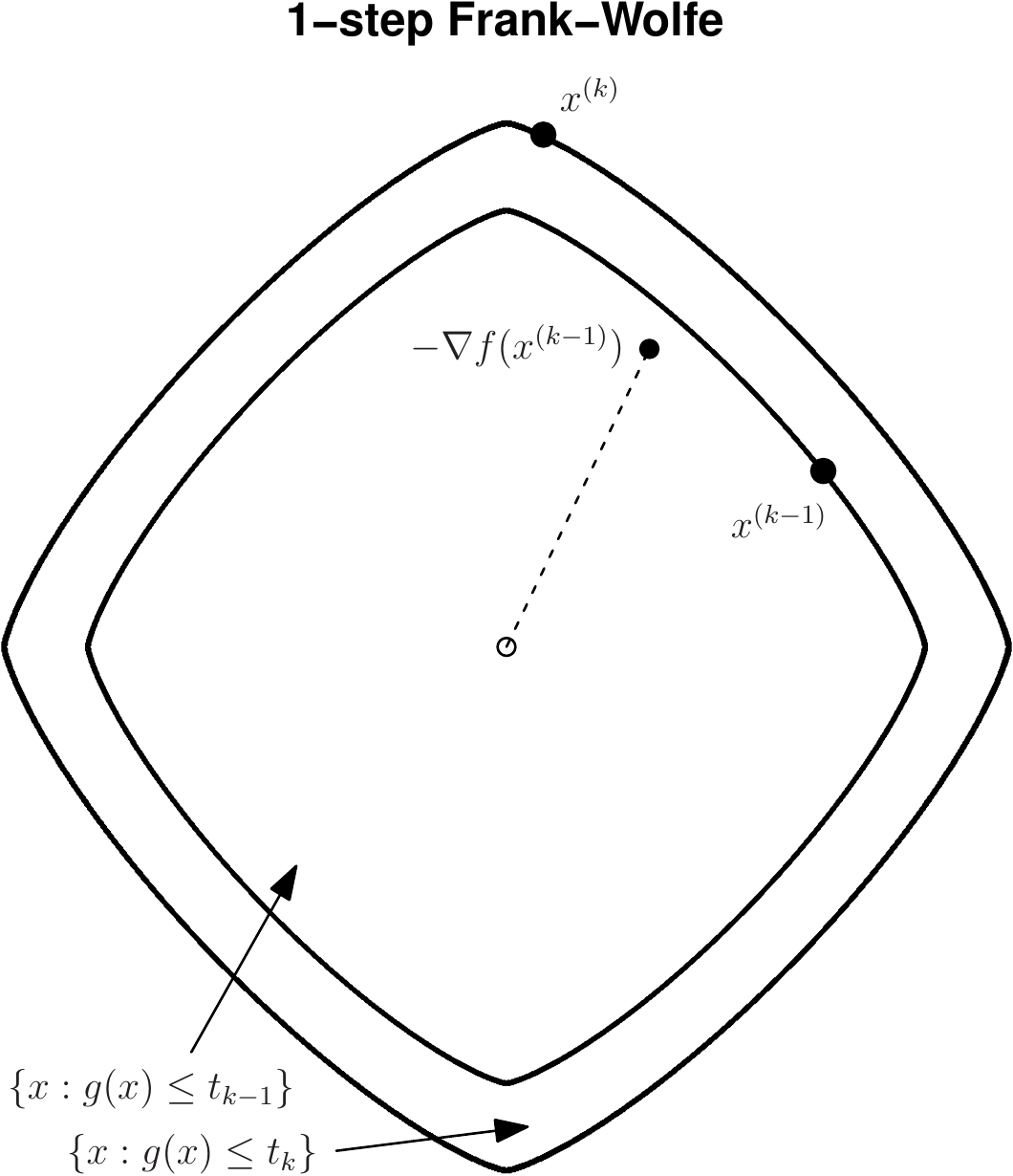}
\includegraphics[height=0.485\textwidth]{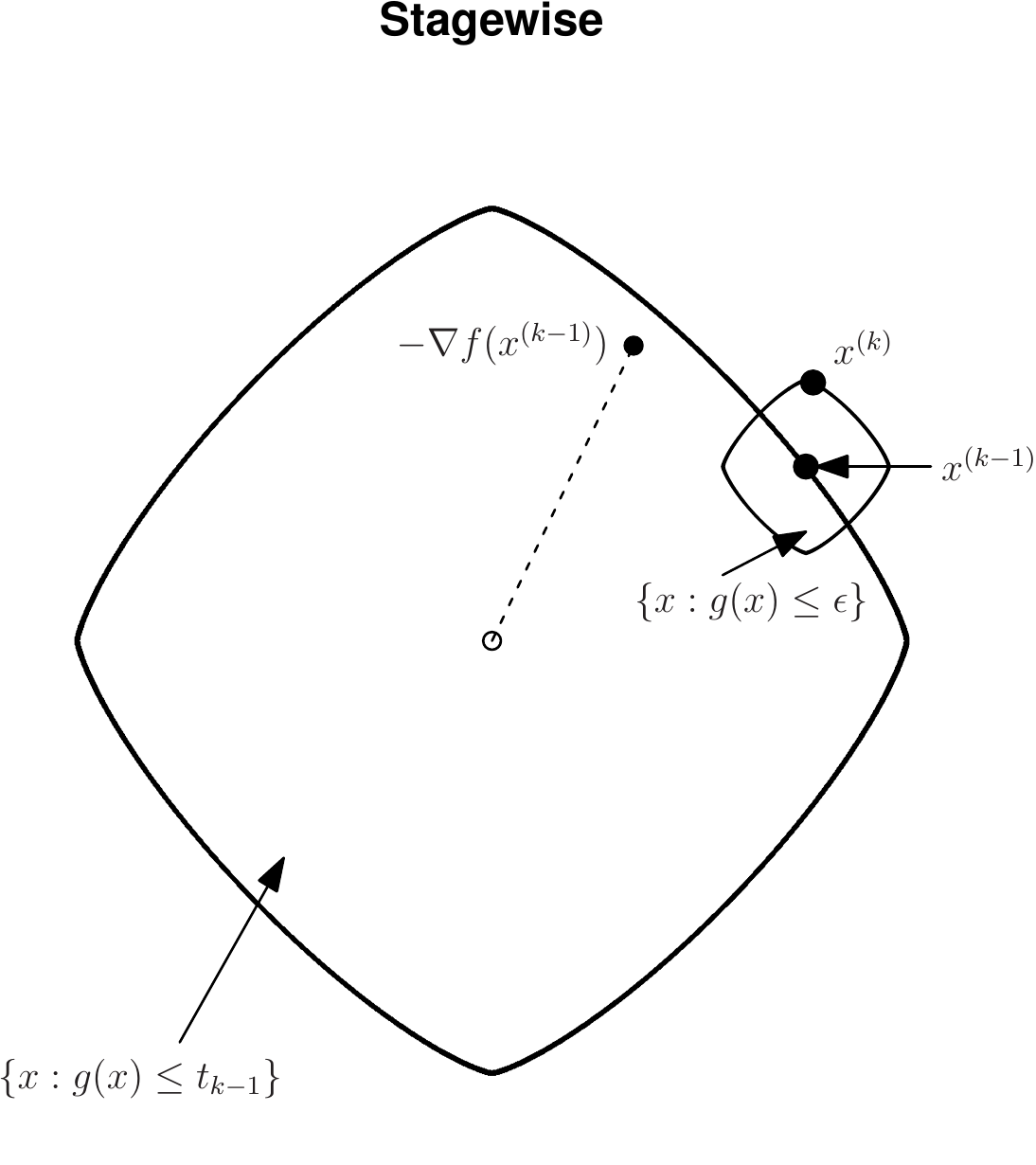}
\caption{\it Illustration of the 1-step Frank-Wolfe and stagewise
  methods.  Each starts with an estimate $x^{(k-1)}$ at a
  regularization parameter value $t_{k-1}$, satisfying $g(x^{(k-1)})
  \leq t_{k-1}$, i.e., a feasible point for the
  problem \eqref{eq:genprob} (but not necessarily optimal).  At
  a larger parameter value $t_k=t_{k-1}+\epsilon$, the 1-step
  Frank-Wolfe strategy inflates the constraint set to $\{x: g(x) \leq
  t_k\}$, and chooses its estimate $x^{(k)}$ to be the point 
  most aligned with $-\nabla f(x^{(k-1)})$ over this new constraint
  set.  Note that this means $x^{(k)}$ can be far away from the
  previous estimate $x^{(k-1)}$ at $t_{k-1}$.  (Note also that the 
  typical Frank-Wolfe strategy, as opposed to the 1-step strategy,
  would not settle for such a point $x^{(k)}$ as its estimate at
  $t_k$, but would continue iterating from $x^{(k)}$ by repeatedly
  minimizing linear approximations of $f$ over 
$\{x: g(x) \leq t_k\}$ until convergence.)
  The stagewise strategy instead builds
  a shrunken constraint set $\{x: g(x) \leq \epsilon\}$ around 
  $x^{(k-1)}$, and considers only the points in this small region as
  candidates for its next estimate.  It then constructs $x^{(k)}$
  using the same logic as above, by 
  finding the point maximally aligned with $-\nabla f(x^{(k-1)})$ over
  the new constraint region. Such differences (between the
  Frank-Wolfe and stagewise strategies) may not seem drastic, 
  but they have big implications.} 
\label{fig:fw}
\end{figure}

Though seemingly similar, these two strategies
result in entirely different paths of estimates.  Generally
speaking, the 1-step Frank-Wolfe strategy \eqref{eq:fwup2},
\eqref{eq:fwdir2} is not very useful,
since its update steps discard too much information from previous 
estimates.  Consider, e.g., the $\ell_1$ regularization setting, 
where $g(x)=\|x\|_1$: here each estimate from the 1-step
Frank-Wolfe algorithm would have only one nonzero component,
corresponding to the maximum absolute entry of the gradient vector
evaluated at the previous estimate.\footnote{Strictly speaking, if
there are ties between the absolute components of the gradient vector 
at $t_k$, then the estimate can be taken to be any convex combination 
$t_k \cdot \sum_{i \in \cI} \alpha_i e_i$, where $\cI$
is the set of maximizing indices, and each $\alpha_i \geq 0$ with  
$\sum_{i \in \cI} \alpha_i=1$.  We do not maintain this distinction
throughout our discussion in this section.}
Hence, instead of producing a sequence of models that become
progressively more and more dense as the regularization
parameter increases, as with the
stagewise algorithm, the 1-step Frank-Wolfe algorithm produces a
sequence of trivial models, each with just one active variable.
Similar conclusions can be drawn by looking at settings like
group-structured regularization, trace norm regularization, etc.

We would likely never use the 1-step Frank-Wolfe procedure in practice
to compute an (approximate) regularization path, but the insights
gained from studying this algorithm carry over to
the more common use case introduced initially: the typical
Frank-Wolfe strategy, in which we run the Frank-Wolfe algorithm until 
convergence across a sequence of regularization parameter values $t_k$,
$k=1,2,3,\ldots$ with warm starts, discards a lot of information
about previously computed estimates.   
At a parameter value $t_k$, the only information used by
the Frank-Wolfe algorithm about the previously
computed estimate $x^{(k-1)}$ is the gradient of $f$ at
$x^{(k-1)}$.  In particular, in its first step at $t_k$, it chooses
the first iterate to minimize the inner product with $\nabla
f(x^{(k-1)})$ over all feasible points.  
If this minimizer is far 
from $x^{(k-1)}$, then, assuming that the solutions at $t_{k-1}$ and
$t_k$ are close, the Frank-Wolfe algorithm basically wastes
iterations bringing itself back to where it was at the end
of its run for $t_{k-1}$. 
Interestingly, the iterations {\it within} a run of Frank-Wolfe 
at a fixed parameter value $t_k$ prevent the algorithm from deviating  
too far from previous iterates, by means of the shrinkage factor
$\gamma$ in \eqref{eq:fwshrink}; however, no such control 
takes place {\it between} runs of the Frank-Wolfe algorithm at
successive parameter values, $t_{k-1}$ and $t_k$, using warm 
starts. The stagewise algorithm \eqref{eq:stageup},
\eqref{eq:stagedir}, on the other hand, shares a great deal of
information between estimates at successive iterations (recall that by
definition, the stagewise estimates
$x^{(k)}$ and $x^{(k-1)}$ differ by an amount $\Delta$, 
controlled to be small under $g$), and in this sense, it makes a much
more efficient use of its history.   

It helps to think about an example.
Returning to the $\ell_1$ regularization setting, suppose that
we have computed an estimate $x^{(k-1)}$ with, say, 50 nonzero
components out of 1000 at some value of the regularization parameter
$t_{k-1}$.  At a slightly larger parameter value $t_k$,  
the stagewise algorithm retains essentially all of the information in 
$x^{(k-1)}$---information about which variables are active,
and the values of their coefficients---and
increments (or decrements) another component of $x^{(k-1)}$
in order to form $x^{(k)}$.  By comparison, the
Frank-Wolfe algorithm uses $x^{(k-1)}$ as a warm start for its run at
$t_k$, meaning that for its first step, it constructs an iterate 
with only one nonzero component, corresponding to the maximal
entry of $\nabla f(x^{(k-1)})$ in absolute value.  In subsequent
steps, only one component of the iterate is adjusted at a time.  Said
in words, the Frank-Wolfe algorithm at $t_k$
has to ``relearn'' the entire set of active variables (and their 
coefficient values), starting from the empty set.  This seems like a
markedly inefficient use of its computational history, certainly in 
comparison to the strategy taken by the stagewise
algorithm.\footnote{For completeness, we should also mention a variant
of the Frank-Wolfe algorithm proposed by \citet{jaggi}, in which the 
shrinkage parameter $\gamma$ in the update step \eqref{eq:fwup}    
is chosen by exact line search, as opposed to the default
(nonadaptive) value given in \eqref{eq:fwshrink}.  This version of 
Frank-Wolfe has the potential to use more of its history,
depending on how large it sets $\gamma$ (especially in its first
step).  Still, a key distinction remains: the adaptive Frank-Wolfe
strategy can choose to use more or less of its history, but for
the stagewise algorithm, relying strongly on the computed history is a 
set decision, not one that is adaptively made over its course.}

Of course, a crucial difference to note 
is that the estimates $x^{(k)}$, $k=1,2,3,\ldots$ from the Frank-Wolfe
strategy are guaranteed to be solutions in \eqref{eq:genprob}   
(up to an arbitrarily small level of tolerance) at $t=t_k$,
$k=1,2,3,\ldots$, but the stagewise estimates $x^{(k)}$,
$k=1,2,3,\ldots$ are not, even as the spacings between the
parameter values $t_k$, $k=1,2,3,\ldots$ goes to zero.
One can also argue that the Frank-Wolfe algorithm was not 
designed to be a path following method, and so comparing to stagewise
by simply applying it sequentially with warm starts is unfair.  Some
authors have in fact considered a specialized Frank-Wolfe strategy
for path following \citep{fwpath1,fwpath2,fwpath3}.  The general
goal of this work is to construct an approximate solution 
path in \eqref{eq:genprob} with a provable approximation  
guarantee (in terms of the achieved criterion value); this is done by 
continuously controlling a duality gap for problem
\eqref{eq:genprob} as the parameter $t$ varies, a strategy 
that does not depend on the Frank-Wolfe algorithm per se,
but can be easily combined with the Frank-Wolfe algorithm because its
iterates readily admit such a duality gap.  

An implementation of this idea is described in Appendix
\ref{app:fwpath}, as its details are not important for the current
discussion.  This path following algorithm can be setup to ensure a 
$\gamma$-suboptimal regularization path, for any given $\gamma>0$, and
operationally it boils down to running Frank-Wolfe with warm starts
over a sequence of adaptively chosen parameter values $t_k$, 
$k=1,2,3,\ldots$ (rather than a given fixed sequence).
This adaptive sequence tends to be quite dense for reasonably small 
choices of $\gamma$ (much more dense than a typical fixed
sequence of parameter values); for larger values of $\gamma$, the
adaptive sequence is more spread out, but then it takes many
iterations at each parameter value to converge 
(especially towards the unregularized end of the path). 
Altogether, the previous comparison between the two methods can be  
drawn here: the Frank-Wolfe path following strategy does not utilize 
its history nearly as efficiently as the stagewise algorithm.

The arguments in this subsection were based on high level
reasoning, but they are empirically supported by the examples in 
Section \ref{sec:bigexamples} and Appendix \ref{app:fwpath}, where we 
run stagewise and Frank-Wolfe across a variety of
scenarios.  We can summarize the comparisons drawn, as follows:   
\begin{itemize}
\item the Frank-Wolfe algorithm, run over a (fixed or
  adaptively chosen) sequence of regularization parameter values 
  with warm starts, does not make an efficient use of its
  computational history (i.e., the information contained in
  previously computed estimates), however, it is guaranteed to produce 
  solutions in \eqref{eq:genprob};
\item the stagewise algorithm is comparatively much more efficient at
  using its history of estimates, but is not guaranteed to produce
  solutions in \eqref{eq:genprob}.
\end{itemize}
The fact that the Frank-Wolfe algorithm relinquishes so much
information about previously computed estimates may actually be
the reason, roughly speaking, that it is able to produce solutions in   
\eqref{eq:genprob}.  After all, the solution
path of the convex regularization problem \eqref{eq:genprob} can 
be highly variable (e.g., in a high-dimensional 
lasso problem with correlated predictors, the components of
the solution path can be very wiggly, as predictors can enter and
leave the active set many times), and therefore, by not constraining 
itself to adhere strongly to its computational past, the Frank-Wolfe
algorithm gives itself the freedom to fit each individual estimate
(along a sequence of parameter values $t_k$,
$k=1,2,3,\ldots$) as appropriate.  In contrast, the stagewise
algorithm is constrained to closely follow its path of previously
computed estimates, by construction.  One can even look at this
constrained nature of fitting as an additional type of regularization.    
Except in special circumstances (e.g., monotone component paths in the
lasso problem), the stagewise algorithm does not produce exact
solutions in \eqref{eq:genprob}, a seemingly necessary feature of any
estimation method that follows its computational history so carefully.  
But this is not the end of the story; recall that a main theme of this
paper (the third point in Section \ref{sec:summary}) is that the
stagewise estimates are statistically useful in their own right, in
spite of their (sometimes extreme) differences 
to solutions in \eqref{eq:genprob}.  If we view the momentum that the
stagewise method places on past estimates as an added level of
regularization, then such a claim is perhaps not too surprising.

\subsection{Path following with Frank-Wolfe}
\label{app:fwpath}

Assume that $g$ is a norm, and $g^*(x)=\max_{g(z)\leq 1} x^T z$ is its 
dual norm.  We propose below a path following strategy to
compute an approximate regularization path with Frank-Wolfe.

\begin{algorithm}[\textbf{Path following with Frank-Wolfe}]
\label{alg:fwpath}
\hfill\par
\smallskip\smallskip
\noindent
Fix $\gamma,m>0$, and $t_0 \in \R$.  Set
$\tx(t_0)=\hx(t_0)$, a solution in \eqref{eq:genprob} at
$t=t_0$. Repeat, for $k=1,2,3,\ldots$:
\begin{itemize}
\item Calculate
\begin{equation*}
t_k = t_{k-1} + \frac{(1-1/m)\gamma} 
{g^*\big(\nabla f\big(\tx(t_{k-1})\big)\big)},
\end{equation*}
and set $\tx(t) = \tx(t_{k-1})$ for all $t \in [t_{k-1},t_k)$.
\item Use Frank-Wolfe to compute $\tx(t_k)$, an
  (approximate) solution in \eqref{eq:genprob} at $t=t_k$,
  having duality gap at most $\gamma/m$. 
\end{itemize}
\end{algorithm}

One might notice that the above algorithm 
differs somewhat from the path following algorithms
in \citet{fwpath1,fwpath2,fwpath3} (specifically, in the way that 
it handles the regularization parameter $t$ in
\eqref{eq:genprob}); we make modifications that we feel simplify
the path following algorithm in the current
setting, but really the main idea follows entirely the work of these
authors. Algorithm \ref{alg:fwpath} constructs 
a piecewise constant regularization path, $\tx(t)$, $t \geq t_0$.  
It begins by computing a solution in \eqref{eq:genprob} at some 
initial value $t=t_0$ (and at a higher level of
accuracy than the standard set for the overall path), increases $t$
until the computed solution no longer meets the standard of accuracy
as measured by the duality gap, recomputes a solution at this new
value of the parameter, increases $t$, and so on.

It is easy to verify that the path output by this algorithm is
feasible for \eqref{eq:genprob} at all visited values of the parameter
$t$; moreover, the path has the approximation property  
\begin{equation}
\label{eq:approxprop}
f(\tx(t)) - f(\hx(t)) \leq \gamma, \;\;\; \text{for all $t$}.
\end{equation}
(By all $t$, in the above, we mean all values of $t \geq t_0$ 
visited by the path algorithm.)  To show this, 
we begin by remarking, as in \citet{jaggi}, that the quantity 
\begin{equation*}
h_t(x) = \max_{g(z) \leq t} \, \langle \nabla f(x), x-z \rangle,
\end{equation*}
serves as a valid duality gap for problem \eqref{eq:genprob},
in that for all feasible $x$, 
\begin{equation*}
f(x) - f(\hx(t)) \leq h_t(x),
\end{equation*}
It will be helpful to use an equivalent representation of
the duality gap:
\begin{equation}
\label{eq:dualgap}
h_t(x) = \langle \nabla f(x), x \rangle + t \cdot
\max_{g(z) \leq 1} \, \langle \nabla f(x), z \rangle = 
\langle \nabla f(x), x \rangle + t \cdot g^*\big(\nabla f(x)\big),  
\end{equation}
where we have used the fact that $g(az)=|a|g(z)$, and the definition
of the dual norm $g^*$.  

Now the argument for \eqref{eq:approxprop} is straightforward. By
construction, at step $k$, we compute $\tx(t_k)$ to be an approximate
solution with the property that 
\smash{$h_{t_k}(\tx(t_k)) \leq \gamma/m$}. 
Since the algorithm assigns $\tx(t)=\tx(t_k)$ for all $t \leq 
t_{k+1}$, we must check that 
\smash{$h_t(\tx(t_k)) \leq \gamma$} for all 
$t \leq t_{k+1}$.  Directly from \eqref{eq:dualgap},
\begin{align*}
h_t(\tx(t_k)) &= \langle \nabla f (\tx(t_k)), \tx(t_k) \rangle + t
\cdot g^*\big(\nabla f (\tx(t_k))\big) \\
&\leq \gamma/m + (t-t_k) \cdot g^*\big(\nabla f (\tx(t_k))\big).
\end{align*}
As $t_{k+1}-t_k = (\gamma-\gamma/m)/g^*(\nabla f (\tx(t_k)))$, the
result follows.

We now report on an example of Frank-Wolfe path following in group
lasso regression, with the data setup as in Figure
\ref{fig:big_grouplasso} (in the case of uncorrelated predictors).
We chose $\gamma=250$, hand-tuned to be the largest possible value of
the duality gap bound so that resulting Frank-Wolfe path estimates
differed no more in mean squared error from the exact solutions than
the stagewise ones did (with a step size $\epsilon=1$).  This was
measured by the maximum discrepancy in mean squared error over the 100
regularization parameter values at which exact solutions were
computed; linear interpolation was used to compute mean squared errors
for Frank-Wolfe and stagewise at these parameter values.  See the left
panel in Figure \ref{fig:fwpath} for mean squared error
curves. Under this large value of $\gamma$, and $m=5$ (the results
did not really change by varying $m$), the path
following strategy produced an adaptive sequence of only 102
regularization parameter values spanning the full path
range. However, 
it took many iterations at each parameter value (beyond the start of
the path) to meet the required duality gap, as shown in the right
panel of Figure \ref{fig:fwpath}.  The total number of iterations used
by the Frank-Wolfe path following method was over 14,000, which is
extremely inefficient, especially 
viewed next to the 250 iterations needed by stagewise.  To emphasize:
the comparison here is quite clear-cut, because iterations of
stagewise and Frank-Wolfe are computationally equivalent, and the two
methods have been tuned to yield the same mean squared error
performance.  

\begin{figure}[htb]
\centering
\includegraphics[width=0.475\textwidth]{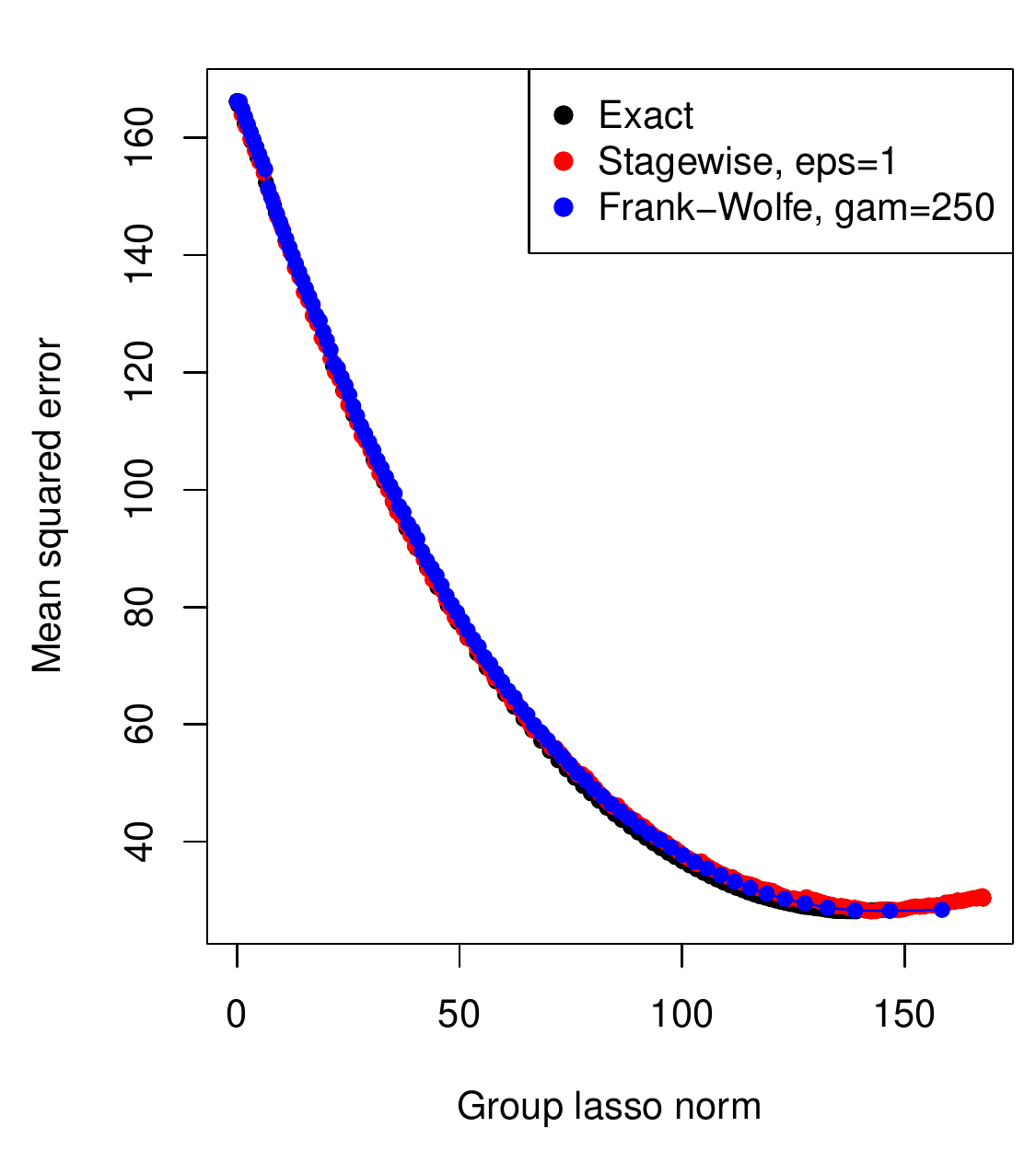}
\includegraphics[width=0.475\textwidth]{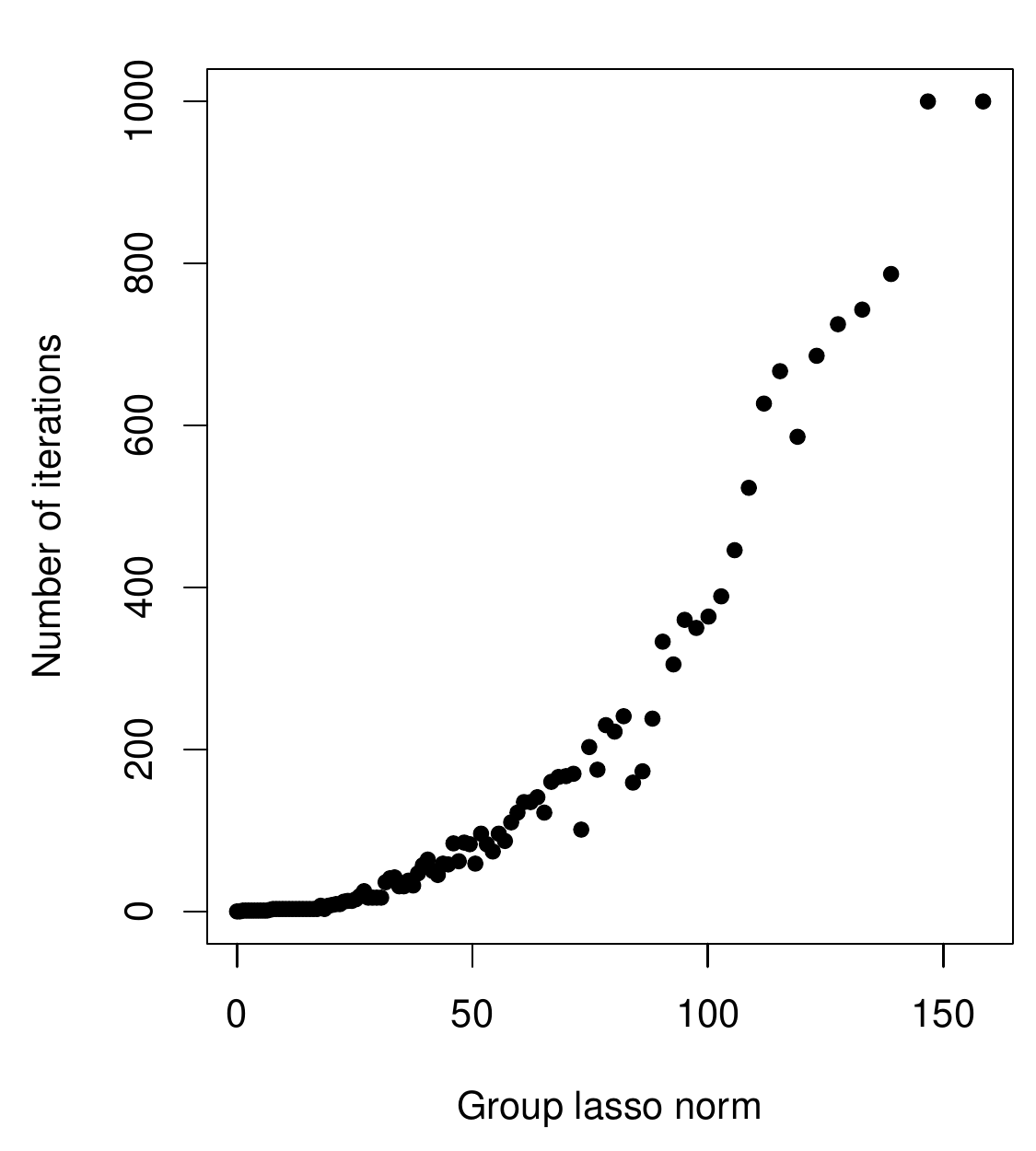}
\caption{\it For the same group lasso setup as in Figure
  \ref{fig:big_grouplasso} (in the uncorrelated predictors case), we
  ran the Frank-Wolfe path following strategy with $\gamma=250$ and
  $m=5$.  The bound $\gamma$ was chosen to be as large as possible so
  that the Frank-Wolfe estimates have competitive mean squared errors
  with the stagewise estimates and exact solutions, confirmed by the
  plot on the left.  The right plot shows the number of iterations
  needed by Frank-Wolfe to converge at the required duality gap of
  $\gamma/m$, as a function of the visited regularization parameter
  value.  The maximum number of iterations was 1000 (hence the
  algorithm did not converge for the largest two regularization
  parameter values.)  We can see that a huge number of iterations
  are needed past the start of the path.}
\label{fig:fwpath} 
\end{figure}

\subsection{Small example: fused lasso signal approximation}  
\label{app:fusedlasso}

For a small 1d fused lasso example, in the Gaussian
signal approximator setup with $n=20$, we generated a   
piecewise constant underlying sequence $\beta^* \in  \R^{20}$ with 5
segments (the levels were drawn  
uniformly at random between 1 and 10), and we added
$N(0,1)$ noise to form the observations $y \in \R^{20}$.  Figure
\ref{fig:fusedlasso12} displays the 1d fused lasso solution path 
on the left, and the stagewise path on the right, constructed from 900
steps using $\epsilon=0.01$.  The two paths look basically the same.
In a rough sense, this is not too surprising, because the 1d fused
lasso problem can be rewritten as a lasso problem with a predictor
matrix $X$ equal to the lower triangular matrix of 1s, and for this 
design, it is known that the limiting stagewise path (as
$\epsilon \rightarrow 0$) is the exact solution path. (Here $X$
satisfies the ``positive cone 
condition'', which ensures the lasso components paths are monotone,
see \citet{lars}, \citet{monotonelasso}.) 
But to be precise, this latter convergence result refers to the
stagewise algorithm applied to the lasso parametrization directly,
and this is not the same as applying the stagewise method to (the 
dual of) the 1d fused lasso parametrization, as was 
done in Figure \ref{fig:fusedlasso12}.  It may be interesting to
compare these two stagewise implementations, with the former 
iteratively adding step functions together, and the latter iteratively 
shrinking adjacent components towards each other.   
It may also be possible to prove a limiting equivalence between the
latter stagewise method and the exact solution path,
from arguments that rely on the monotonicity of the estimated
differences, but we do not pursue these ideas in this paper. 

\begin{figure}[p]
\centering

\begin{subfigure}{\textwidth}
\includegraphics[width=0.475\textwidth]{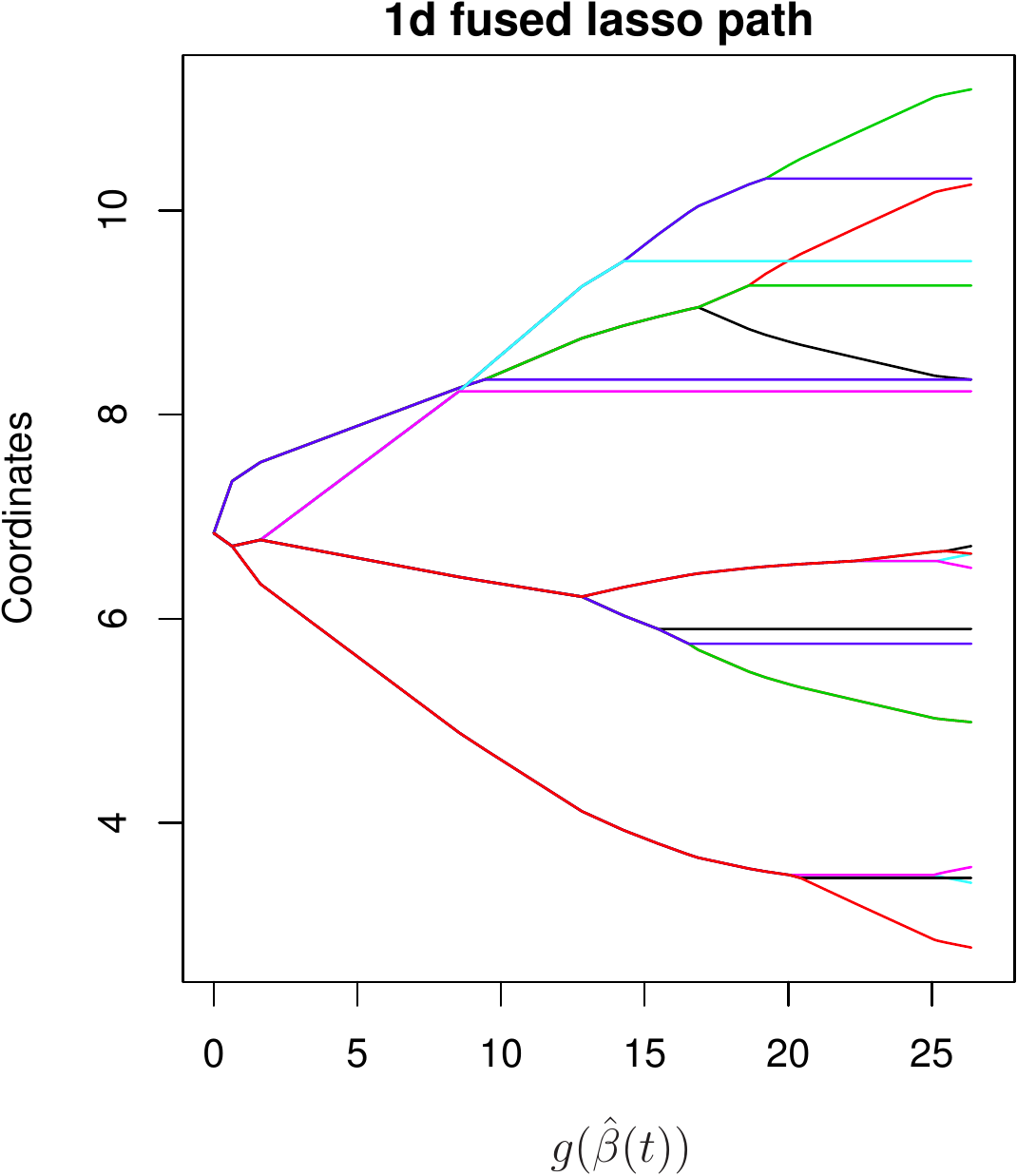} 
\includegraphics[width=0.475\textwidth]{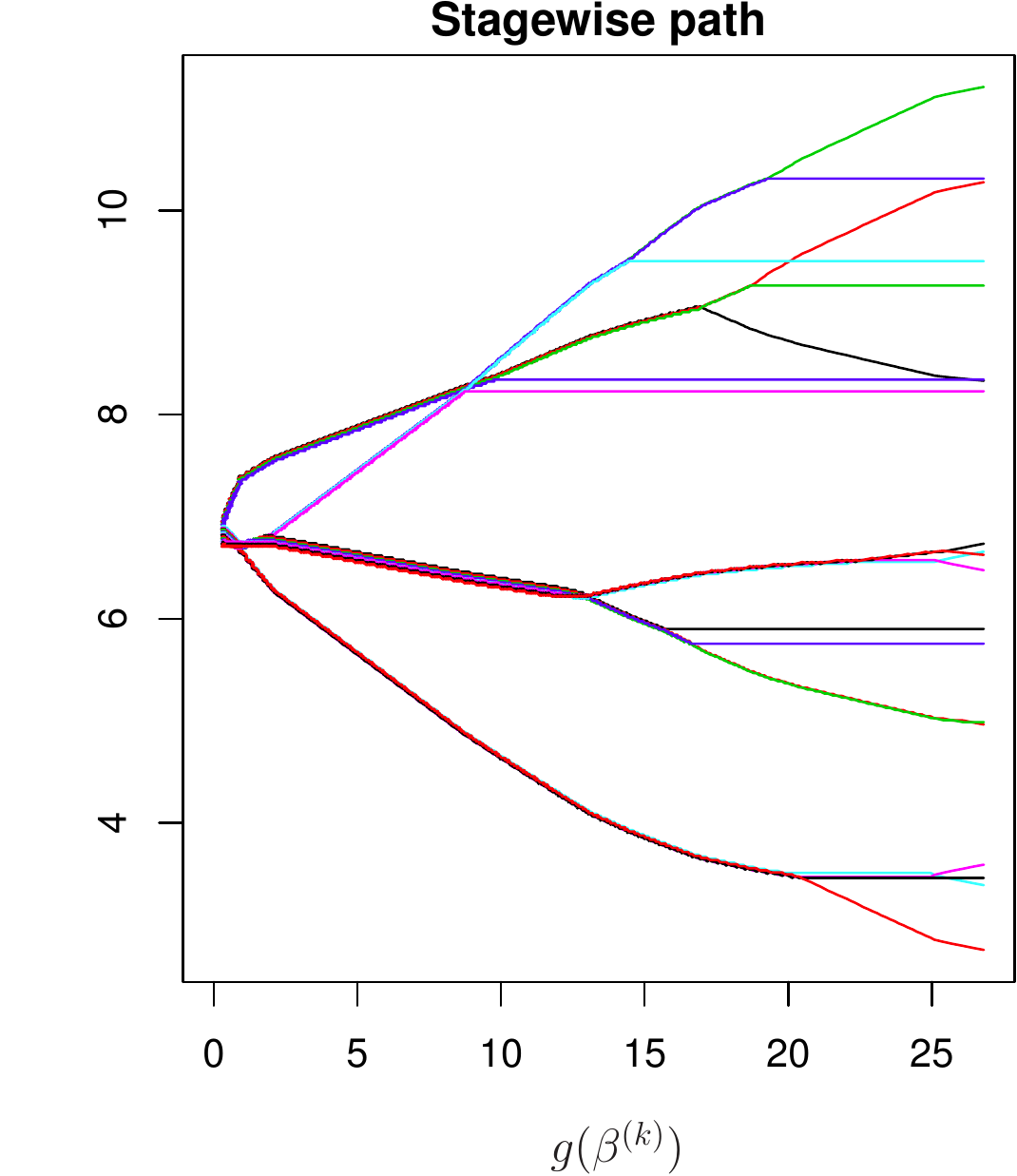} 
\caption{\it A small 1d fused lasso example with $n=20$ points. The
  left plot shows the exact solution path, and the right plot shows
  the stagewise approximation, which is essentially identical.} 
\label{fig:fusedlasso12}
\end{subfigure}

\bigskip
\bigskip
\begin{subfigure}{\textwidth}
\includegraphics[width=0.475\textwidth]{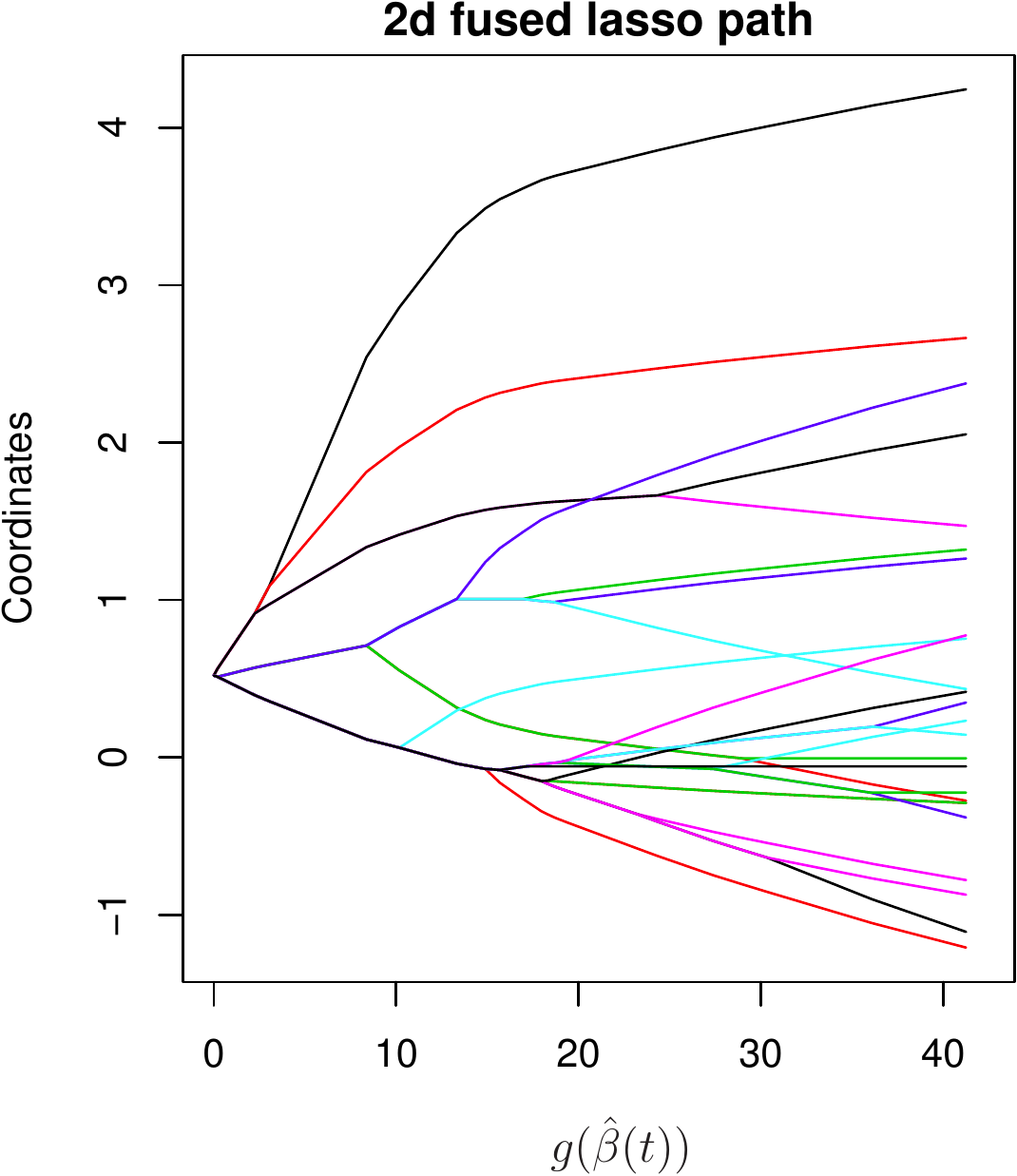} 
\includegraphics[width=0.475\textwidth]{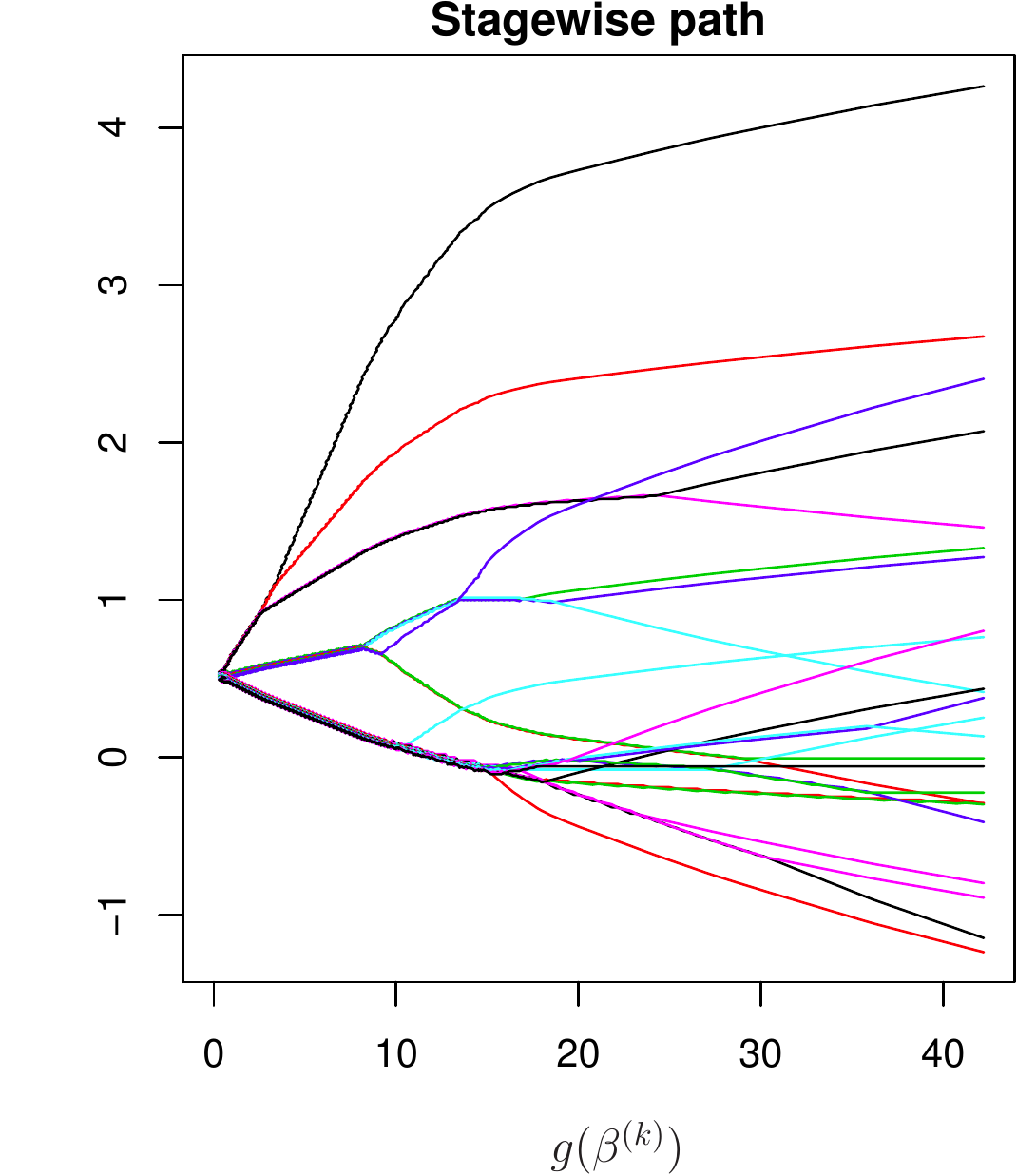} 
\caption{\it A small 2d fused lasso example, using a $5 \times 5$
  image (so that $n=25$). The solution path on the left and stagewise
  path on the right are only slightly different towards the
  regularized end of the paths.}
\label{fig:fusedlasso34}
\end{subfigure}

\caption{\it Fused lasso examples in 1d and 2d.}
\label{fig:fusedlasso}
\end{figure}

For a 2d fused lasso example, still in the Gaussian signal
approximator setup, we took $n=25$, and $\beta^* \in \R^{25}$ to be an
unraveled version of a piecewise constant $5\times 5$ image.  Pixels
in the lower $2 \times 2$ corner of the image were assigned a common
value of 3, and all other pixel values were zero.  We formed $y \in
\R^{25}$ by adding independent $N(0,1)$ noise to $\beta^*$.  
The 2d fused lasso regularizer 
uses a 2-dimensional grid graph over the optimization variable 
$\beta \in \R^{25}$ i.e., this graph joins components of $\beta$ that
correspond to vertically or horizontally adjacent pixels in the
image format.  In Figure \ref{fig:fusedlasso34}, we show the
exact 2d fused lasso solution path on the left, and the stagewise
path on the right, from 500 steps with $\epsilon=0.005$.  It is  
not easy to spot many differences between the two
(one difference can be seen when the 2d fused lasso norm is
about 10), and altogether the stagewise path appears to be a very good  
approximation. 

Finally, we emphasize that the stagewise steps in Figures 
\ref{fig:fusedlasso12} and \ref{fig:fusedlasso34} were derived from 
the dual, so the construction of paths proceeded  
{\it from right to left} in the stagewise plots (i.e., the stagewise
paths were built for decreasing  values of the regularization
parameter $t$ in \eqref{eq:genlasso}), contrary to all other
stagewise examples outside of the generalized lasso setting. 
The stagewise approximation is 
hence most accurate at the right end of the plot, and its component   
paths become more choppy at the left end. In most foreseeable
applications---fused lasso applications or otherwise---this is
unfortunately not a desirable feature.  Usually regularized
estimates are of primary concern,  
so we would not want our iterative algorithm to reach these last,   
and certainly not with a lower measure of accuracy.  
An exception is the case of image denoising (under 2d fused lasso 
regularization): here estimates at low levels of regularization are
usually interesting, as the underlying (noiseless) image itself is
usually complex, especially for real, large images.  The dual
stagewise algorithm thrives in this case, as we saw on the large image
denoising examples in Section \ref{sec:bigimg}.
 
\subsection{Large example: ridge logistic regression} 
\label{app:biglog}

{\bf Overview.}
We investigate two simulated examples of ridge regularized logistic 
regression. The ridge logistic regression solutions 
were computed with the {\tt glmnet} R package, available on CRAN,
which offers a highly optimized coordinate descent implementation 
\citep{pco,glmnet}.  The {\tt glmnet} package allows for lasso, ridge,
and mixed (elastic net) regularization, and is actually more efficient
in the presence of lasso regularization (because it utilizes an
active set approach, which takes advantage of sparsity).  However, it
is still fairly efficient for pure ridge regularization if the number of
variables $p$ is not too large compared to the number of
observations $n$, and the solutions to be computed correspond to large
or moderate amounts of regularization (i.e., it does not
need to compute solutions too close to the unregularized end of the
path).  Therefore we chose the example setups to meet these rough 
guidelines. 

Recall, as described in Section \ref{sec:quadreg}, that the stagewise
procedure for ridge regularization is very simple, and iteratively
updates the estimate small amounts in the direction of the negative 
gradient (here, the gradient of 
the logistic loss function).  Our C++ implementation of this stagewise
routine, for the 
examples in the current section, is less than 25 lines of code.
Meanwhile, the {\tt glmnet} package uses a sophisticated, nuanced
Fortran implementation of coordinate descent, which totals thousands
of lines of code. (To be fair, the {\tt glmnet} Fortran code
is multipurpose, in that it solves more than just ridge regularized
logistic regression: it handles elastic net regularized
generalized linear models. Still, the broad comparison stands,
between the complexities of the two implementations.)

\bigskip
\noindent
{\bf Examples and comparisons.}  Both simulation setups used  
$n=8000$ observations and $p=500$ predictor variables.
The binary inputs $y \in \R^{8000}$ were drawn independently
according to the logistic probabilities 
\begin{equation}
\label{eq:probs}
p_i^* =
\frac{1}{1+\exp(-[X \beta^*]_i)}, \;\;\; i=1,\ldots 8000,
\end{equation}
where the true coefficient vector $\beta^* \in \R^{500}$ had 50
nonzero components drawn independently from $N(0,1)$, and the
predictor matrix $X \in \R^{8000\times 500}$ was constructed
differently in the two setups. In the first,
the entries of $X$ were drawn independently from $N(0,1)$, and in the
second, the rows of $X$ were drawn independently from $N(0,\Sigma)$,
where $\Sigma \in \R^{500 \times 500}$ had unit diagonals and all
off-diagonal elements equal to $\rho=0.8$.  In other words, the first 
setup used uncorrelated predictors and the second 
used highly positively correlated predictors.

In both cases, we ran {\tt glmnet} over 100 regularization parameter
values (starting from the regularized end of the path, using warm 
starts).  We also ran the stagewise algorithm with two choices of step
size, $\epsilon=0.0025$ and $\epsilon=0.25$.  The results are shown
in Figure \ref{fig:big_logridge}.  Looking at the uncorrelated case,
in the left plot, first: we can see that, averaged over 10 simulated
draws of the observations $y$ (with fixed $X,\beta^*$), both stagewise
sequences achieve a competitive minimum misclassification rate to that
of the exact solution path (recorded with respect to independently
drawn test inputs drawn from \eqref{eq:probs}).  In the early stages
of the path, the exact solutions exhibit a better misclassification
rate than the stagewise estimates with $\epsilon=0.0025$, which in
turn exhibit a better error rate than the stagewise estimates with
$\epsilon=0.25$, but all estimates end up at the same minimum
misclassification rate later in the path.  The table in the bottom row
of Figure \ref{fig:big_logridge} shares the computation times for
these methods (averaged over 10 draws of the observations, and
recorded on a desktop computer).  The {\tt glmnet} coordinate descent
implementation took an average of 12 seconds to compute its 100 
solutions; stagewise with $\epsilon=0.0025$ took about 3 seconds to
compute 150 estimates; stagewise with $\epsilon=0.25$ took 0.3
seconds to compute its 15 estimates.

\begin{figure}[htb!]
\centering
\includegraphics[width=0.475\textwidth]{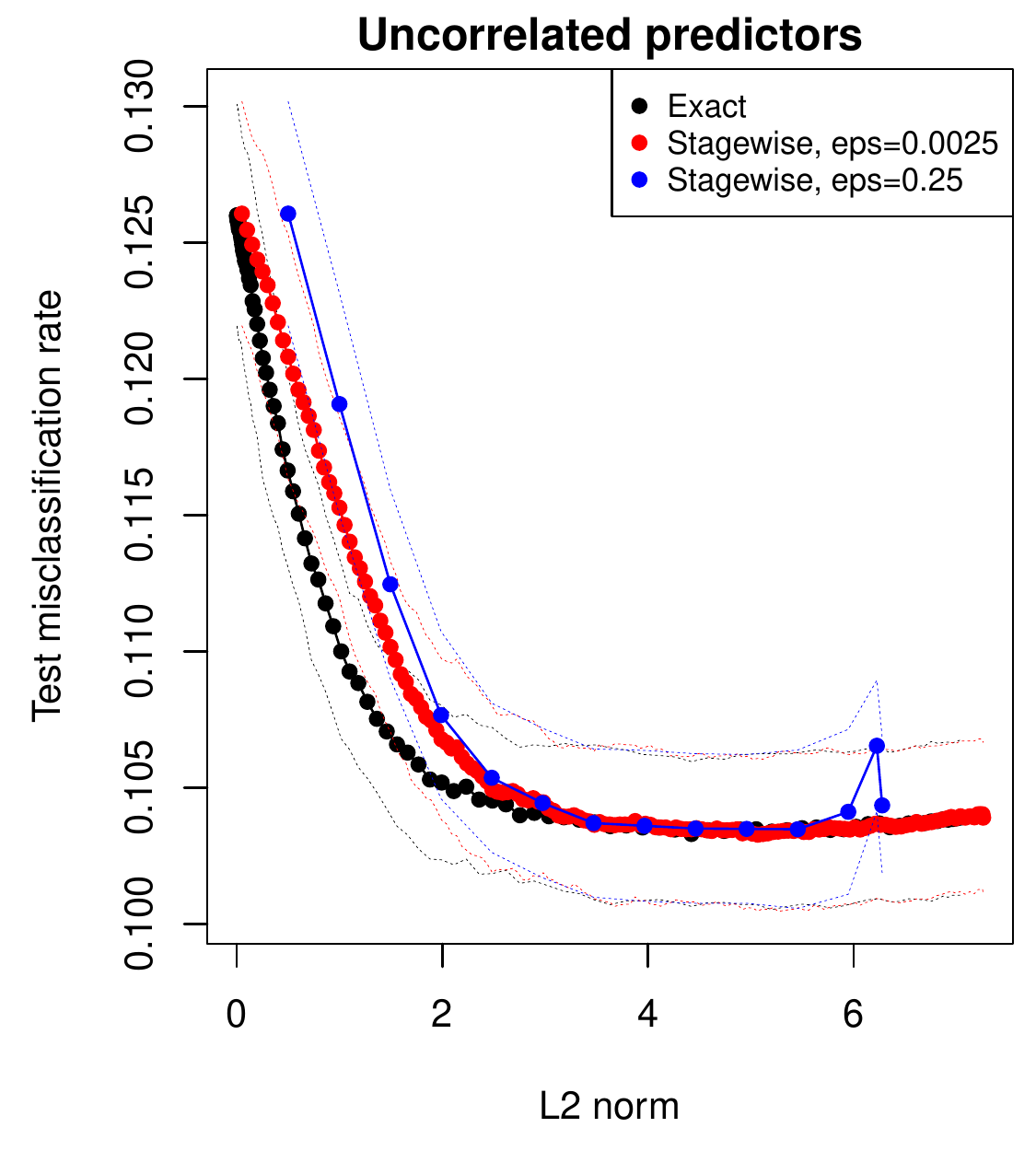}
\includegraphics[width=0.475\textwidth]{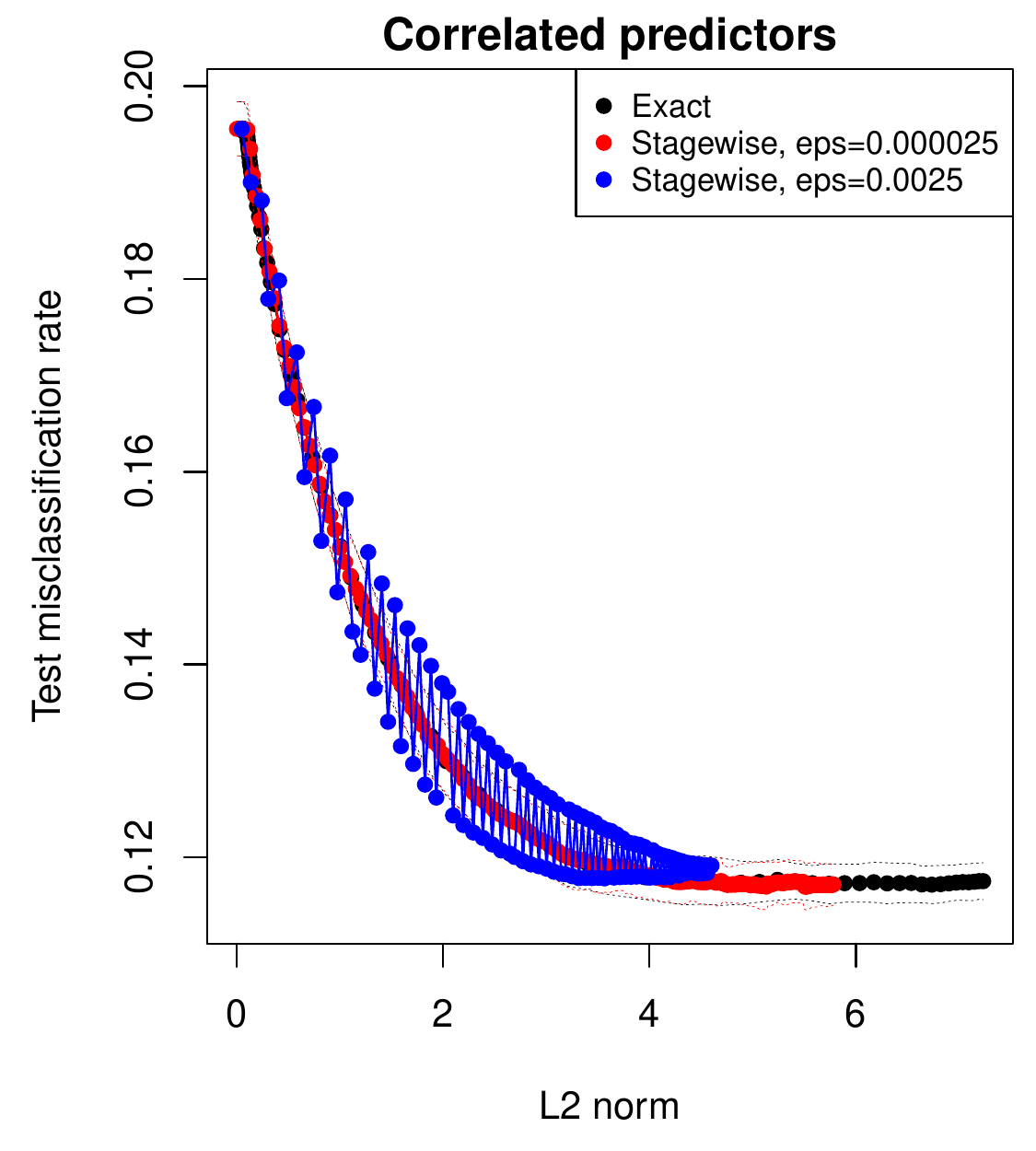} 
\vspace{5pt} \\
{\small
\begin{tabular}{|c|r|}
\hline
\multicolumn{2}{|c|}{Algorithm timings} \\
\hline
\hline
Method & Uncorrelated case \\
\hline
Exact: coordinate descent, 100 solutions & 12.12 (0.12) \\  
Stagewise: $\epsilon=0.0025$, 150 estimates & 3.01 (0.01) \\
Stagewise: $\epsilon=0.25$, 15 estimates & 0.30 (0.01) \\
\hline
Method & Correlated case \\
\hline
Exact: coordinate descent, 100 solutions & 11.32 (0.31) \\ 
Stagewise: $\epsilon=0.000025$, 2000 estimates & 40.21 (0.24) \\
Stagewise: $\epsilon=0.0025$, 1000 estimates & 20.10 (0.18) \\
\hline
\end{tabular}}
\caption{\it Comparisons between exact and stagewise estimates for
  ridge regularized logistic regression, with $n=8000$ and
  $p=500$. The top two plots show test  
  misclassification errors committed by solutions and stagewise
  estimates in two different scenarios, one with uncorrelated
  predictors on the left, and one with highly correlated predictors on
  the right. The bottom table gives timings for the {\tt glmnet}
  coordinate descent algorithm in computing exact solutions and the
  stagewise algorithms.  (All test errors and
  timings were averaged over 10 repetitions from the simulation model;
  the dotted lines in the plots show standard deviations, as do the
  parentheses in the table.) The stagewise algorithm performs ideally
  in the uncorrelated scenario, delivering statistically accurate
  estimates at very low computational cost; on the other hand, it
  seriously struggles in the correlated setup, requiring even
  smaller step sizes and far more steps to produce statistically
  meaningful estimates. (In the right plot, only 10\% of the points
  along the stagewise error curves are drawn, and the standard
  deviations are withheld from the $\epsilon=0.0025$ curve, for
  visibility.)}
\label{fig:big_logridge} 
\end{figure}

In terms of the performance of the stagewise algorithm, the correlated
problem setup stands in stark contrast to the uncorrelated one.  
In fact, this correlated case represents the closest incident to a
failure for stagewise in this paper---to be perfectly 
clear, though, the ``failure'' here is entirely computational.
Using step sizes $\epsilon=0.000025$ and
$\epsilon=0.0025$, the stagewise method needed disproportionately 
more steps to cover a comparable part of the regularization path.
This meant 2000 
and 1000 steps when $\epsilon=0.000025$ and $\epsilon=0.0025$,
respectively. Apparently the effective step length in this
problem is greatly contracted, and the runtimes for computing a full 
stagewise regularization path are significantly inflated, as reported
in the table in Figure \ref{fig:big_logridge}.  With the smaller step
size, $\epsilon=0.000025$, the right plot in Figure
\ref{fig:big_logridge} shows that the stagewise
estimates track the test misclassification rates of the exact
solutions very closely; with the larger step size, $\epsilon=0.0025$,
the estimates display an odd trend in which their test
errors bounce around the solution test errors. 

This behavior, and 
the unusually slow stagewise progress, can be explained by the
following rough geometric perspective.  The contours of the logistic
loss function $f(\beta)$ lie close to a tilted and very thin ellipse
in $\R^{500}$, due to the highly positively correlated predictor 
variables $X$.  (This contours are not exactly elliptical because the
Hessian of $f$ is not constant, but locally they are approximately
so.) Starting from the origin, the stagewise algorithm
repeatedly adds updates in the direction of the negative gradient of
$f$.  Because the ellipse is so thin, the updates will often pass
``through'' the ellipse and make little progress in advancing the
$\ell_2$ norm of the iterates. That is, the $\ell_2$ norm
of the iterates $\beta^{(k)}$, $k=1,2,3,\ldots $ does not consistently
increase, unless the step size $\epsilon$ is very small;
otherwise progress it alternates back and forth, some steps advancing
the $\ell_2$ norm than others.  Hence the stagewise 
coefficients plots, even with the fairly small step size
$\epsilon=0.0025$, display a distinct 
zigzag pattern; see Figure \ref{fig:big_logridge_cor}.
This pattern is only exacerbated by multiplication by $X$, and the
achieved error rates, which are based on the fitted values
$X\beta^{(k)}$, $k=1,2,3,\ldots$, jump around wildly.  
(It may be interesting to note that, despite this zigzag behavior, the
mean squared error cruve between the stagewise estimates $\beta^{(k)}$ 
under $\epsilon=0.0025$ and the true parameter $\beta^*$ is actually
still competitive, which is probably not surprising when staring at
the strong similarities between coefficient paths in the top row of 
Figure \ref{fig:big_logridge_cor}.)
With a small enough step size, $\epsilon=0.000025$, 
this issue disappears, but of course the downside is that it now takes
2000 steps to compute a full stagewise path. 
A more thorough understanding of this problem and (hopefully) a
computational remedy are important topics for future work. 



\begin{figure}[htb]
\centering
\includegraphics[width=0.325\textwidth]{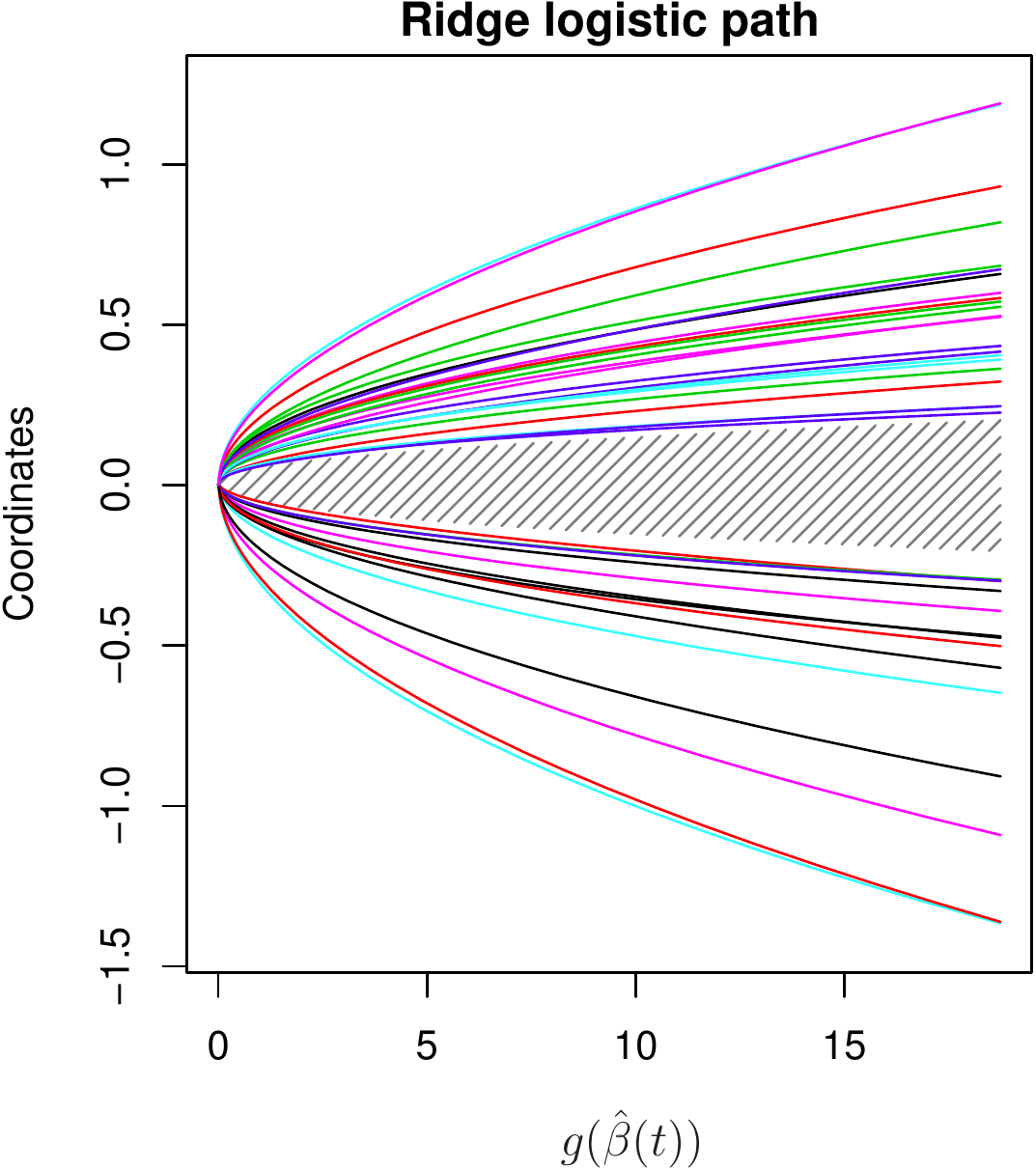}
\includegraphics[width=0.325\textwidth]{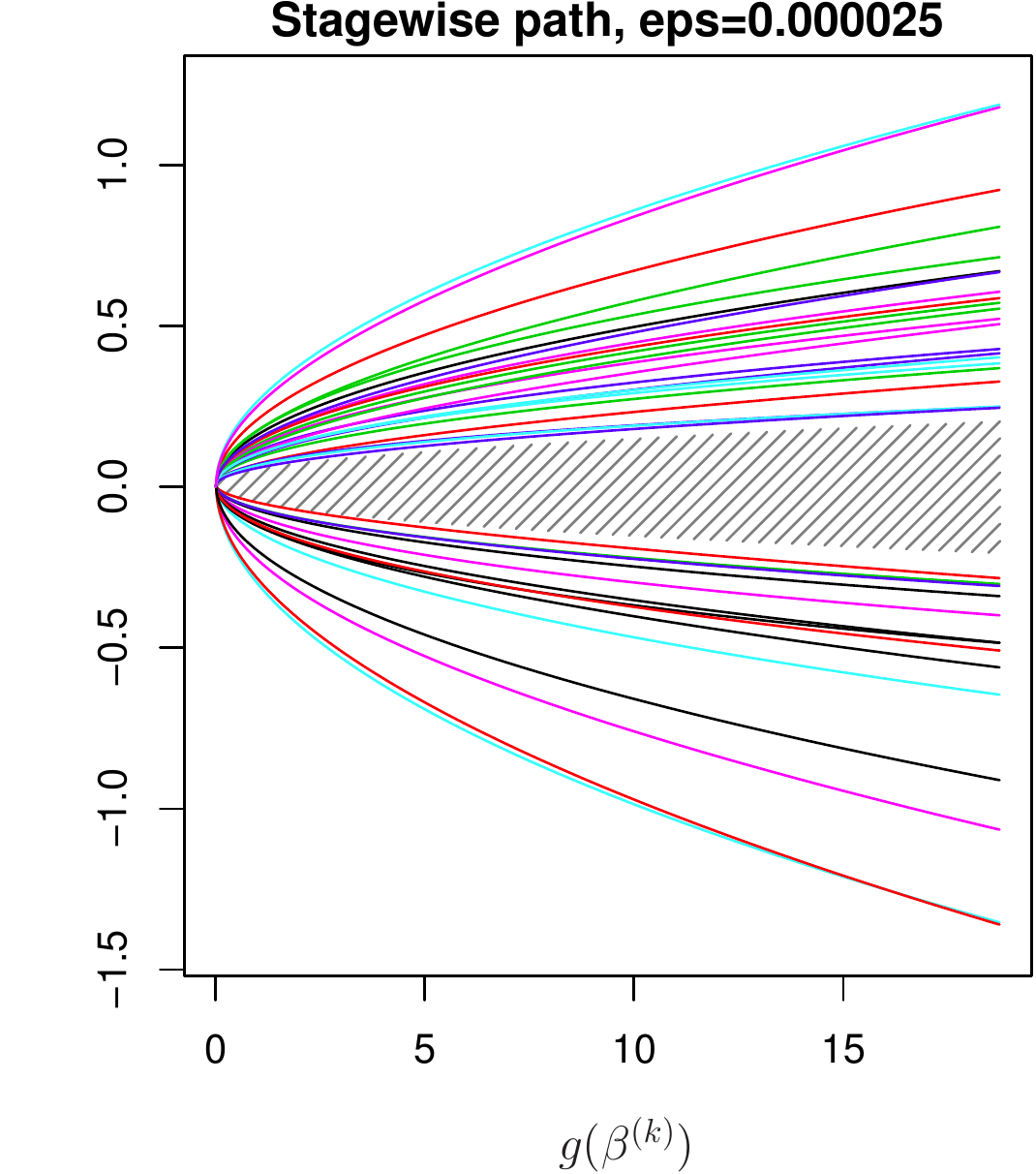} 
\includegraphics[width=0.325\textwidth]{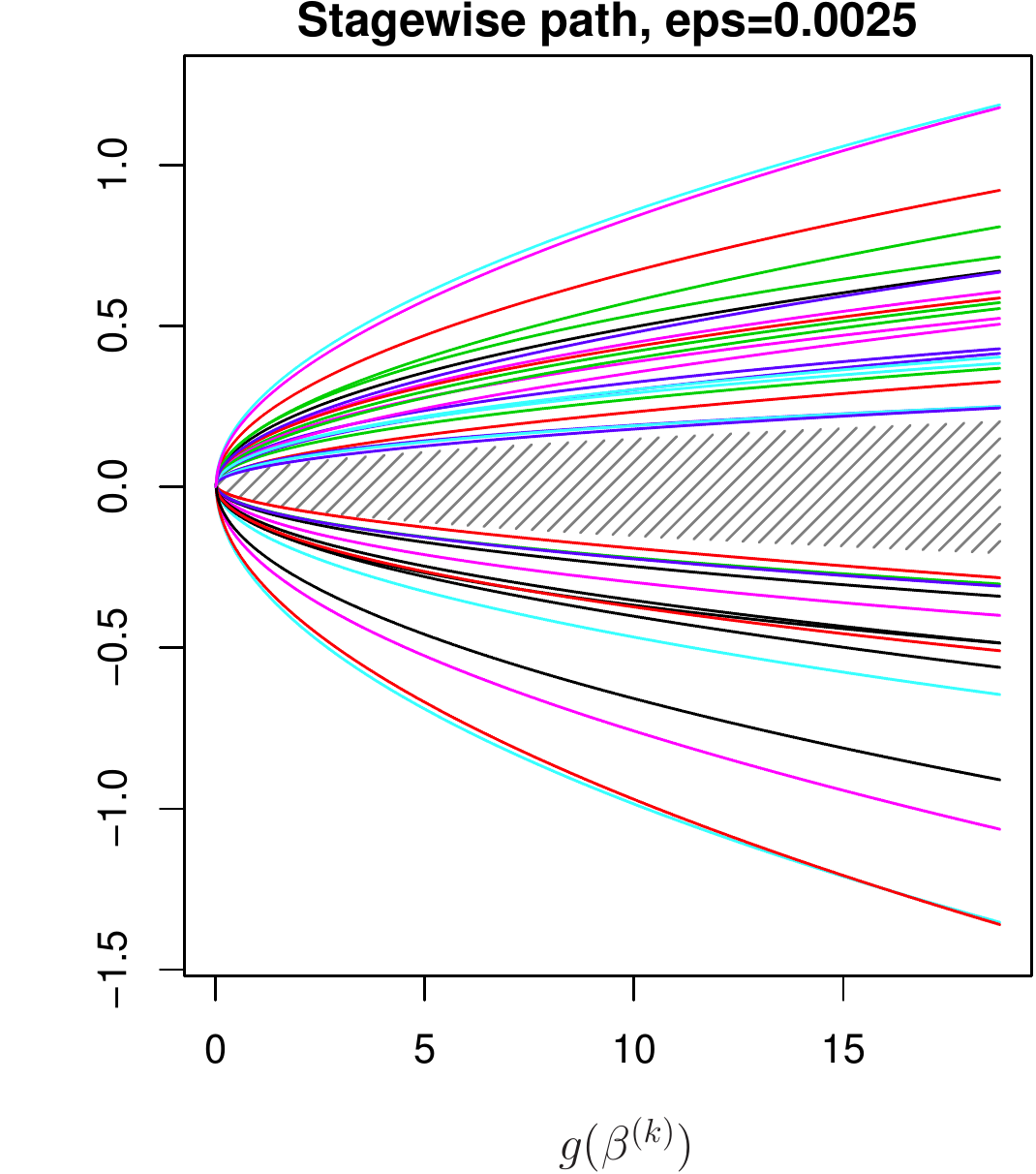}
\vspace{10pt} \\
\includegraphics[width=0.325\textwidth]{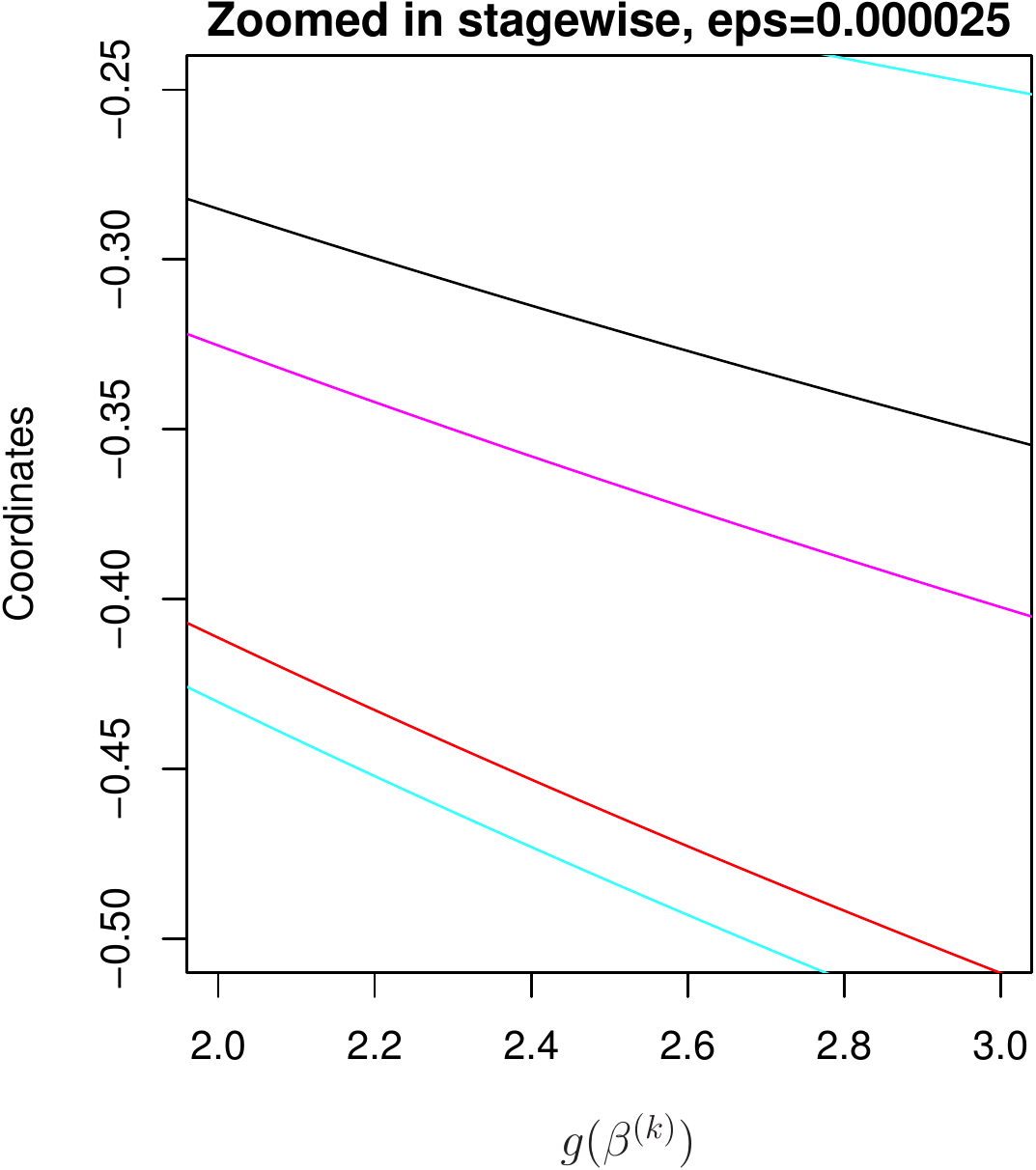} 
\includegraphics[width=0.325\textwidth]{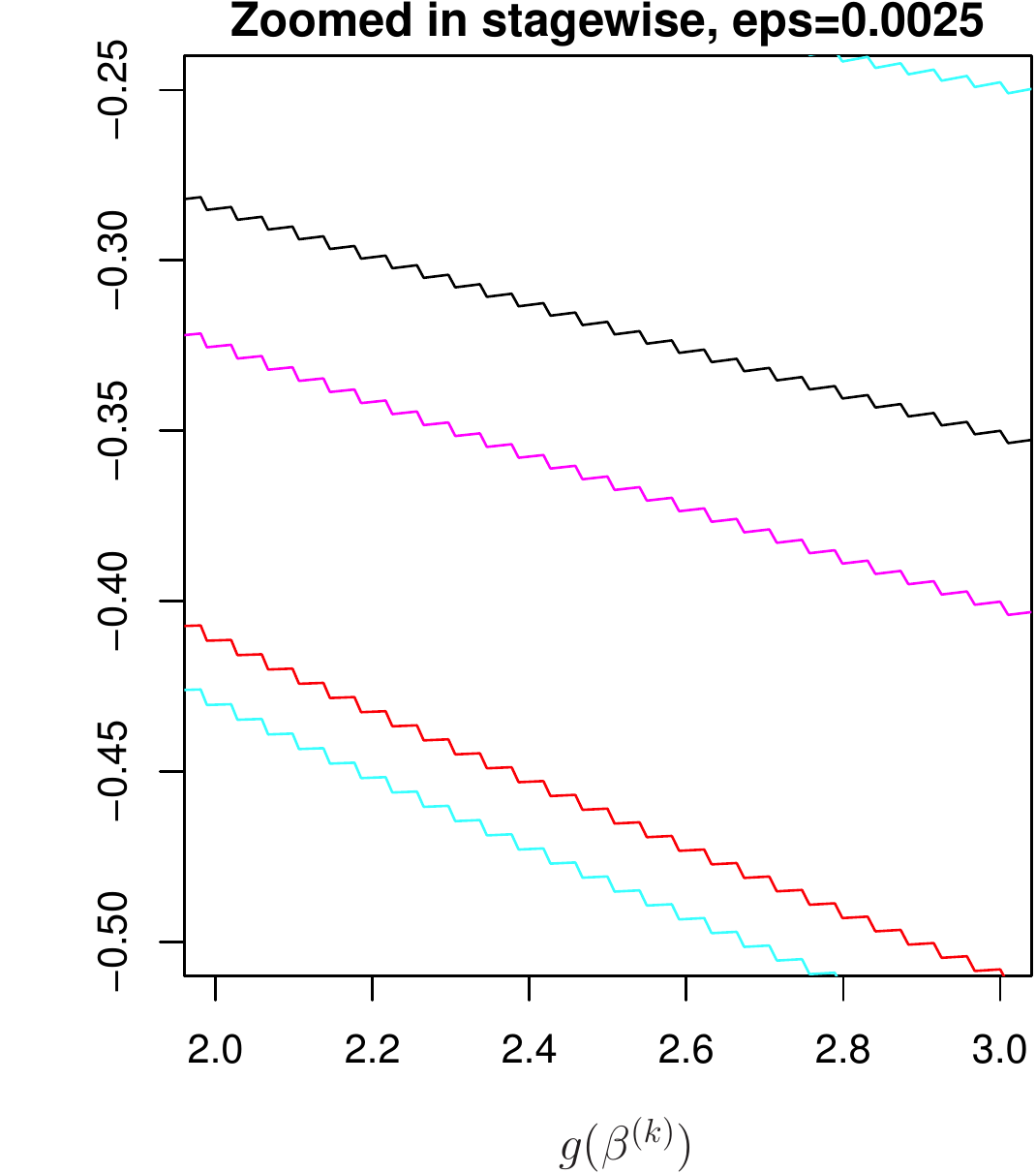}
\includegraphics[width=0.325\textwidth]{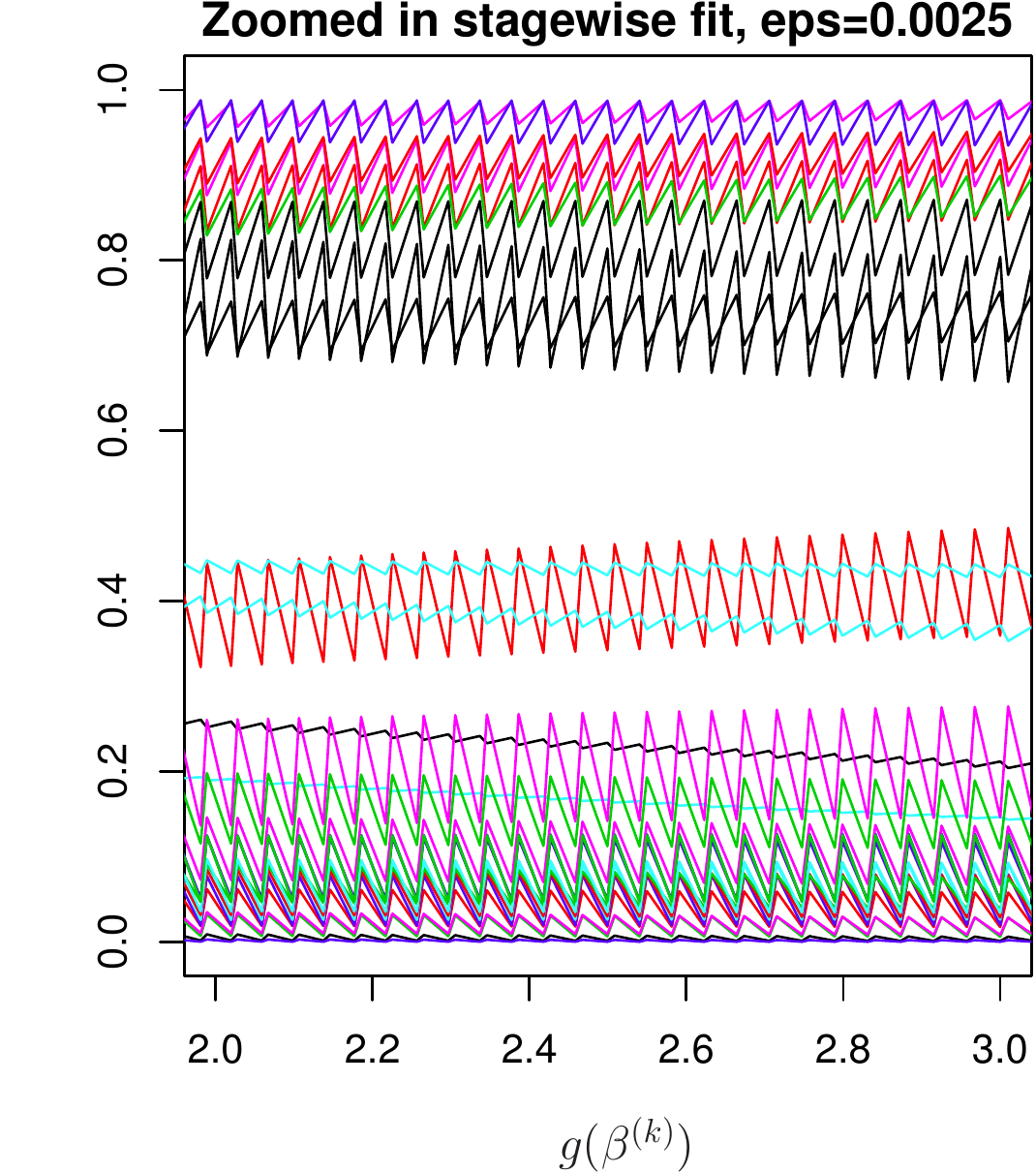} 
\caption{\it Exact and stagewise coefficient plots for one simulated
  problem under the ridge regularized logistic regression setup, with
  correlated predictors. In the top row we excluded many of the
  coefficient paths that were close to zero (dashed gray region), and
  in the bottom row we zoomed in on a select number of coefficient
  paths, for visualization purposes.} 
\label{fig:big_logridge_cor} 
\end{figure}

\subsection{Proof of Theorem \ref{thm:stagewise}}
\label{app:stagethm}

The suboptimality bound in this theorem is based on the quantity  
\begin{equation*}
h_t(x) = \max_{g(z) \leq t} \, \langle \nabla f(x), x-z \rangle.
\end{equation*}
As remarked in Appendix \ref{app:fwpath}, this serves as a valid
duality gap for the problem \eqref{eq:genprob}, in that for all 
feasible $x$, we have $f(x) - f(\hx(t)) \leq h_t(x)$, with $\hx(t)$
being a solution in \eqref{eq:genprob}.  This property follows
directly from the first order condition for convexity applied to $f$. 
We note that if $x$ is a solution in \eqref{eq:genprob}, then
$h_t(x)=0$, because in this case $\langle \nabla f(x) , x \rangle
\leq \langle \nabla f(x) , z \rangle$ for all feasible $z$.  Also, for
the proof of the theorem, it will be helpful to rewrite $h_t(x)$ in
the equivalent form: 
\begin{equation*}
h_t(x) = \langle \nabla f(x), x \rangle + t \cdot
\max_{g(z) \leq 1} \, \langle \nabla f(x), z \rangle = 
\langle \nabla f(x), x \rangle + t \cdot g^*\big(\nabla f(x)\big),
\end{equation*}
as in Appendix \ref{app:fwpath}, relying on the fact that $g$ is a
norm, and $g^*$ its dual norm. Lastly, it will be helpful to
rewrite the stagewise updates \eqref{eq:stageup}, \eqref{eq:stagedir} 
as 
\begin{gather*}
x^{(k)} = x^{(k-1)} - \epsilon \delta^{(k-1)}, \\ 
\text{where}\;\,
\delta^{(k-1)} \in \argmax_{z\in\R^n} \,\,
\langle \nabla f(x^{(k-1)}), z \rangle
\;\,\st\;\, g(z) \leq 1.
\end{gather*}

\begin{proof}[Proof of Theorem \ref{thm:stagewise}]  At any arbitrary
  step $k$, we compute
\begin{align*}
h_{t_k}(x^{(k)}) &= \langle \nabla f(x^{(k)}), x^{(k)} \rangle + 
t_k g^*\big (\nabla f(x^{(k)}) \big) \\ 
&= \langle \nabla f(x^{(k)}), x^{(k-1)} \rangle -
\epsilon \langle \nabla f(x^{(k)}), \delta^{(k-1)} \rangle + 
t_k g^*\big( \nabla f(x^{(k)}) \big).
\end{align*}
Now we add and subtract terms in order to express the right-hand side
in terms of $h_{t_{k-1}}(x^{(k-1)})$,
\begin{align*}
h_{t_k}(x^{(k)}) &= \langle \nabla f(x^{(k-1)}), x^{(k-1)}
\rangle + t_{k-1} g^*\big( \nabla f(x^{(k-1)}) \big)
 + \langle \nabla f(x^{(k)}) - \nabla f(x^{(k-1)}), x^{(k-1)} 
\rangle \\
&\qquad - \epsilon \langle \nabla f(x^{(k)}), \delta^{(k-1)} \rangle
+ t_k g^*\big(\nabla f(x^{(k)})\big) - t_{k-1} g^*\big(\nabla  
 f(x^{(k-1)})\big) \\ 
&= h_{t_{k-1}}(x^{(k-1)}) + 
 \langle \nabla f(x^{(k)}) - \nabla f(x^{(k-1)}), x^{(k-1)}
\rangle - \epsilon \langle \nabla f(x^{(k)}) - \nabla f(x^{(k-1)}),  
\delta^{(k-1)} \rangle \\
&\qquad - \epsilon \langle \nabla f(x^{(k-1)}), \delta^{(k-1)} \rangle
+ t_k g^*\big(\nabla f(x^{(k)})\big) - t_{k-1} g^*\big(\nabla
f(x^{(k-1)})\big) \\ 
&= h_{t_{k-1}}(x^{(k-1)}) + \langle \nabla f(x^{(k)}) -\nabla
f(x^{(k-1)}), x^{(k)} \rangle \\
&\qquad - \epsilon \langle \nabla f(x^{(k-1)}), \delta^{(k-1)} \rangle  
+ t_k g^*\big(\nabla f(x^{(k)})\big) - t_{k-1} g^*\big(\nabla
f(x^{(k-1)})\big) \\
&= h_{t_{k-1}}(x^{(k-1)}) + \langle \nabla f(x^{(k)}) -\nabla f(x^{(k-1)}),
x^{(k)} \rangle + t_k \big[ g^*\big(\nabla f(x^{(k)})\big) - 
g^*\big(\nabla f(x^{(k-1)})\big) \big].
\end{align*}
Recursing this, we obtain
\begin{equation*}
h_{t_k} (x^{(k)}) = \sum_{i=1}^k 
\langle \nabla f(x^{(i)}) -\nabla f(x^{(i-1)}), x^{(i)} \rangle
+  \sum_{i=1}^k t_i \big[ g^*\big(\nabla f(x^{(i)})\big) - 
g^*\big(\nabla f(x^{(i-1)})\big) \big].
\end{equation*}
where we used the fact that $h_{t_0}(x^{(0)}) = 0$.  For the
first term, we can apply H\"{o}lder's inequality over
the dual pair $g,g^*$ to each summand; for the second term, we use the
triange inequality $g^*(u-v) \geq g^*(u) - g^*(v)$.  This yields
\begin{align*}
h_{t_k} (x^{(k)}) &\leq \sum_{i=1}^k 
g^*\big(\nabla f(x^{(i)}) -\nabla f(x^{(i-1)})\big) g(x^{(i)})  
+  \sum_{i=1}^k t_i g^*\big(\nabla f(x^{(i)}) - \nabla
f(x^{(i-1)})\big) \\ 
&\leq 2 \sum_{i=1}^k t_i g^*\big(\nabla f(x^{(i)}) - 
\nabla f(x^{(i-1)})\big) \\
&\leq 2L \sum_{i=1}^k t_i g( x^{(i)} - x^{(i-1)}) \\
&\leq 2L\epsilon \sum_{i=1}^k t_i.
\end{align*}
In the second inequality, we used the fact that $g(x^{(i)}) \leq t_i$,
and in the third, we invoked the Lipschitz assumption on $\nabla f$. 
Since $t_i = t_0 + i \epsilon$, this upper bound is 
\begin{equation*}
h_{t_k} (x^{(k)}) \leq L\epsilon^2 k(k+1) + 2L\epsilon k t_0.
\end{equation*}
Finally, we recall that the number of steps $k$ is chosen so that $t_k
= t_0 + k\epsilon = t$, and hence
\begin{equation*}
h_{t_k} (x^{(k)}) \leq L(t-t_0)^2 + L(t-t_0)\epsilon + 2L(t-t_0)t_0 = 
L(t^2-t_0^2)+ L(t-t_0)\epsilon,
\end{equation*}
which completes the proof.
\end{proof}

\subsection{Proof of Theorem \ref{thm:shrunkstage}}
\label{app:shrunkthm}

This proof is similar to that of Theorem
\ref{thm:stagewise}, although it is a little more involved
technically.  We will rely on a few
limit calculations that are introduced and proved after the proof of  
the theorem, in Lemmas \ref{lem:lims} and \ref{lem:lims2}.  We will
also write the shrunken stagewise updates \eqref{eq:shrunkup},
\eqref{eq:shrunkdir} as  
\begin{gather*}
x^{(k)} = \alpha x^{(k-1)} - \epsilon \delta^{(k-1)}, \\ 
\text{where}\;\,
\delta^{(k-1)} \in \argmax_{z\in\R^n} \,\,
\langle \nabla f(x^{(k-1)}), z \rangle
\;\,\st\;\, g(z) \leq 1.
\end{gather*}

\begin{proof}[Proof of Theorem \ref{thm:shrunkstage}]
Assume without a loss of generality that $t_0=0$; the arguments needed
for the case of an arbitrary starting value $t_0$ are similar but more
tedious. As in the proof of Theorem \ref{thm:stagewise}, we compute
the duality gap at an arbirary step $k$ of the algorithm,  
\begin{align*}
h_{t_k}(x^{(k)}) &= \langle \nabla f(x^{(k)}), x^{(k)} \rangle + 
t_k g^*\big (\nabla f(x^{(k)}) \big) \\ 
&= \alpha \langle \nabla f(x^{(k)}), x^{(k-1)} \rangle -
\epsilon \langle \nabla f(x^{(k)}), \delta^{(k-1)} \rangle + 
t_k g^*\big( \nabla f(x^{(k)}) \big), \\
&=\alpha \langle \nabla f(x^{(k-1)}), x^{(k-1)}
\rangle + \alpha t_{k-1} g^*\big( \nabla f(x^{(k-1)}) \big)
 + \alpha \langle \nabla f(x^{(k)}) - \nabla f(x^{(k-1)}), x^{(k-1)} 
\rangle \\
&\qquad - \epsilon \langle \nabla f(x^{(k)}), \delta^{(k-1)} \rangle
+ t_k g^*\big(\nabla f(x^{(k)})\big) - \alpha t_{k-1} g^*\big(\nabla  
 f(x^{(k-1)})\big) \\ 
&= \alpha h_{t_{k-1}}(x^{(k-1)}) + 
\alpha \langle \nabla f(x^{(k)}) - \nabla f(x^{(k-1)}), x^{(k-1)}
\rangle - \epsilon \langle \nabla f(x^{(k)}) - \nabla f(x^{(k-1)}),  
\delta^{(k-1)} \rangle \\
&\qquad - \epsilon \langle \nabla f(x^{(k-1)}), \delta^{(k-1)} \rangle
+ t_k g^*\big(\nabla f(x^{(k)})\big) - \alpha t_{k-1} g^*\big(\nabla
f(x^{(k-1)})\big) \\ 
&=\alpha h_{t_{k-1}}(x^{(k-1)}) + \langle \nabla f(x^{(k)}) -\nabla
f(x^{(k-1)}), x^{(k)} \rangle \\
&\qquad - \epsilon \langle \nabla f(x^{(k-1)}), \delta^{(k-1)} \rangle  
+  t_k g^*\big(\nabla f(x^{(k)})\big) - \alpha t_{k-1} g^*\big(\nabla
f(x^{(k-1)})\big) \\
&=\alpha h_{t_{k-1}}(x^{(k-1)}) + \langle \nabla f(x^{(k)}) -\nabla f(x^{(k-1)}),
x^{(k)} \rangle + t_k \big[ g^*\big(\nabla f(x^{(k)})\big) - 
g^*\big(\nabla f(x^{(k-1)})\big) \big],
\end{align*}
and recursing this, we obtain
\begin{equation*}
h_{t_k} (x^{(k)}) = 
\underbrace{\sum_{i=1}^k \alpha^{k-i} \langle \nabla
f(x^{(i)}) -\nabla f(x^{(i-1)}), x^{(i)} \rangle}_{A}  \,+\,  
\underbrace{\sum_{i=1}^k \alpha^{k-i} t_i \big[ g^*\big(\nabla
  f(x^{(i)})\big) - g^*\big(\nabla f(x^{(i-1)})\big) \big]}_{B}. 
\end{equation*}
We proceed to bound terms $A$ and $B$ separately.

\bigskip
\noindent
{\it Term A.} We apply H\"{o}lder's inequality, and
then use the Lipschitz continuity of $\nabla f$,
\begin{align}
\nonumber
A &\leq \sum_{i=1}^k \alpha^{k-i} 
g^*\big(\nabla f(x^{(i)}) -\nabla f(x^{(i-1)})\big) g(x^{(i)}) \\ 
\nonumber
&\leq L \sum_{i=1}^k \alpha^{k-i} t_i g( x^{(i)} - x^{(i-1)}) \\
\label{eq:abd}
&\leq L \sum_{i=1}^k \alpha^{k-i} t_i 
\big((1-\alpha)t_{i-1}+\epsilon\big). 
\end{align}
By definition, $t_i = \alpha t_{i-1} + \epsilon$ for all 
$i=1,2,3,\ldots $, and a short inductive argument shows that   
\begin{equation*}
t_i = \alpha^i t_0 + (\alpha^{i-1} + \ldots + \alpha + 1)\epsilon =  
\alpha^i t_0 + \frac{1-\alpha^i}{1-\alpha} \epsilon.
\end{equation*}
Continuing from \eqref{eq:abd}, we can
rewrite the bound on term $A$ as 
\begin{align}
\nonumber
A &\leq \frac{L\epsilon^2}{1-\alpha}
\sum_{i=1}^k \alpha^{k-i} (1-\alpha^i)(2-\alpha^{i-1}) \\
\nonumber
&= \frac{L\epsilon^2}{1-\alpha}
\bigg( 2\sum_{i=1}^k \alpha^{k-i} - 2\sum_{i=1}^k \alpha^k -
\sum_{i=1}^k \alpha^{k-1} + \alpha^k 
\sum_{i=1}^k \alpha^{i-1} \bigg) \\ 
\nonumber
&= \frac{L\epsilon^2}{1-\alpha}
\bigg( 2 \frac{1-\alpha^k}{1-\alpha} - 2 \alpha^k k - 
\alpha^{k-1} k + \alpha^k \frac{1-\alpha^k}{1-\alpha} \bigg) \\ 
\label{eq:gapbd2}
&\leq \frac{L\epsilon^2}{(1-\alpha)^2}(1-\alpha^k)(2+\alpha^k) 
- \frac{3L\epsilon^2}{1-\alpha}\alpha^k k. 
\end{align}
Let $M=\epsilon/((1-\alpha)t)$.  Note that by the assumptions of the
theorem, $M \rightarrow \infty$ as $\epsilon \rightarrow 0$
and $\alpha \rightarrow 1$.  Hence, with
a slight reparametrization, we will express the limits $\epsilon
\rightarrow 0$, $\alpha \rightarrow 1$ as $\epsilon \rightarrow 0$,
$M \rightarrow \infty$.  Now we assume that the number of steps $k$ is 
chosen so that 
\begin{equation*}
t_k = \frac{1-\alpha^k}{1-\alpha}\epsilon = t,
\end{equation*}
i.e., $1-\alpha^k = 1/M$, or
\begin{equation*}
k = \frac{\log(1-1/M)}{\log\alpha}=  
\frac{\log(1-1/M)} {\log(1-\epsilon/(Mt))},
\end{equation*}
where we have used that $\alpha=1-\epsilon/(Mt)$.  Plugging in 
$\epsilon/(1-\alpha)=Mt$ and $1-\alpha^k=1/M$ into
\eqref{eq:gapbd2}, our bound is
\begin{align*}
A &\leq LMt^2 (3-1/M) - 3LMt(1-1/M)\epsilon k \\
&= \underbrace{3Lt M(t-\epsilon k)}_{a} \,-\, Lt^2 \,+\, 
\underbrace{3Lt \epsilon k}_{b}.
\end{align*}
By Lemma \ref{lem:lims}, we have
$a \rightarrow -3Lt^2/2$ as  
$\epsilon \rightarrow 0$, $M \rightarrow \infty$, and 
$b \rightarrow 3Lt^2$ as $\epsilon \rightarrow 0$,
$M \rightarrow \infty$.  Therefore, in the limit, we have
\begin{equation*}
A \leq -3Lt^2/2-Lt^2+3Lt^2 = Lt^2/2.
\end{equation*}

\bigskip
\noindent
{\it Term B.}  We decompose 
\begin{equation*}
B = \underbrace{\sum_{i=r+1}^k \alpha^{k-i} t_i \big[ g^*\big(\nabla 
  f(x^{(i)})\big) - g^*\big(\nabla f(x^{(i-1)})\big) \big]}_{c} + 
\underbrace{\sum_{i=1}^r \alpha^{k-i} t_i \big[ g^*\big(\nabla 
  f(x^{(i)})\big) - g^*\big(\nabla f(x^{(i-1)})\big) \big]}_{d},
\end{equation*}
and now consider each of $c,d$ in turn. 

\bigskip
\noindent
{\it Term $c$.}  Using the triangle inequality and the Lipschitz 
continuity of $\nabla f$, 
\begin{align*}
c &\leq \sum_{i=r+1}^k \alpha^{k-i} t_i g^*\big(\nabla 
  f(x^{(i)}) - \nabla f(x^{(i-1)})\big) \\
&\leq L \sum_{i=r+1}^k \alpha^{k-i} t_i g( x^{(i)} - x^{(i-1)}) \\
&= \frac{L\epsilon^2}{1-\alpha} \sum_{i=r+1}^k 
\alpha^{k-i} (1-\alpha^i)(2-\alpha^{i-1}) \\
&= \underbrace{\frac{L\epsilon^2}{1-\alpha} \sum_{i=1}^k  
\alpha^{k-i} (1-\alpha^i)(2-\alpha^{i-1})}_{c_1} \,-\,
\underbrace{\frac{L\epsilon^2}{1-\alpha} \sum_{i=1}^r 
\alpha^{k-i} (1-\alpha^i)(2-\alpha^{i-1})}_{c_2}.
\end{align*}
From the previous set of arguments, $c_1 \leq Lt^2/2$ as $\epsilon
\rightarrow 0$, $M \rightarrow \infty$.  Further, following these same 
arguments (but with $r$ in place of $k$), 
\begin{align*}
c_2 &= \alpha^{k-r}\frac{L\epsilon^2}{1-\alpha} 
\bigg( 2 \frac{1-\alpha^r}{1-\alpha} - 2 \alpha^r r - 
\alpha^{r-1} r + \alpha^r \frac{1-\alpha^r}{1-\alpha} \bigg) \\ 
&\geq \alpha^{k-r} \bigg( \frac{L\epsilon^2}{(1-\alpha)^2}
(1-\alpha^r)(2+\alpha^r) 
- \frac{3L\epsilon^2}{1-\alpha} \alpha^{r-1} r \bigg) \\
&= \alpha^{k-r} \bigg( \frac{L\epsilon^2}{(1-\alpha)^2}
(1-\alpha^{k \frac{r}{k}})(2+\alpha^{k \frac{r}{k}}) 
- \frac{3L\epsilon^2}{1-\alpha} \frac{1}{\alpha} \alpha^{k
  \frac{r}{k}} k \frac{r}{k} \bigg) \\
&= \alpha^{k-r} \bigg(  LM^2t^2 \big(1-(1-1/M)^{\frac{r}{k}}\big)
\big(2+(1-1/M)^{\frac{r}{k}}\big) - 
\frac{3LMt}{\alpha} (1-1/M)^{\frac{r}{k}} \epsilon k \frac{r}{k}
\bigg) \\
&= \alpha^{k-r}
\bigg(LM^2t^2\big(1-(1-1/M)^\varphi\big)
\big(2+(1-1/M)^\varphi\big) - 
\frac{3L Mt}{\alpha} (1-1/M)^\varphi \epsilon k \varphi
\bigg).
\end{align*}
In the last line above, we have abbreviated $\varphi=r/k$, and we 
recall that $\varphi \rightarrow \theta$ as $\epsilon
\rightarrow 0$, $M \rightarrow \infty$  by assumption.  We expand the 
last expression, and omit the leading term $\alpha^{k-r}$ since it  
converges to $1$: 
\begin{align*}
&LM^2t^2\big(1-(1-1/M)^\varphi\big)
\big(2+(1-1/M)^\varphi\big) - \frac{3L Mt}{\alpha} (1-1/M)^\varphi  
\epsilon k \varphi \\ 
&= 
3LM^2t^2\big(1-(1-1/M)^\varphi\big) - \frac{3LMt}{\alpha}
(1-1/M)^\varphi \epsilon k \varphi - 
LM^2 t^2 \big(1-(1-1/M)^\varphi\big)^2 \\
&= 3\varphi L t M(t-\epsilon k) + 
3Lt^2 M \Big( M \big(1-(1-1/M)^\varphi\big) - \varphi \Big) +  
\frac{3\varphi L t \epsilon k}{\alpha} M\big( 1-(1-1/M)^\varphi\big) 
 \\
 &\qquad  
-\smash{L t^2M^2 \big(1-(1-1/M)^\varphi\big)^2} 
+ 3\varphi L t M \epsilon k (1/\alpha-1)
\\
&\rightarrow -3 \theta Lt^2/2 +  3\theta(1-\theta) Lt^2/2 
+ 3\theta^2 Lt^2 -\theta^2 Lt^2 + 0 \\
&= \theta^2 Lt^2/2.
\end{align*}
The limits of the first four terms on the second to last line are due
to Lemmas \ref{lem:lims} and \ref{lem:lims2}; the last term converges
to 0 as $M(1-\alpha)=\epsilon/t \rightarrow 0$.
Hence, in the limit, we have $c_2 \geq \theta^2 Lt^2/2$, and
\begin{equation*}
c \leq c_1 - c_2 \leq Lt^2/2 - \theta^2 Lt^2/2 = (1-\theta^2) Lt^2/2.
\end{equation*}

\bigskip
\noindent
{\it Term $d$.}  We expand this term as
\begin{align*}
d &= \sum_{i=1}^{r-1} \alpha^{k-1-i}(\alpha t_i - t_{i+1})
g^*\big(\nabla f(x^{(i)})\big) +
\alpha^{k-r} t_r g^*\big(\nabla f(x^{(r)})\big) -
\alpha^{k-1} t_1 g^*\big(\nabla f(x^{(0)})\big)\\
&= -\epsilon \sum_{i=1}^{r-1} \alpha^{k-1-i} 
g^*\big(\nabla f(x^{(i)})\big) +
\alpha^{k-r} t_r g^*\big(\nabla f(x^{(r)})\big) -
\alpha^{k-1} t_1 g^*\big(\nabla f(x^{(0)})\big).
\end{align*}
We now apply the bounds from the decay condition \eqref{eq:decay} in
the theorem, 
which gives
\begin{align*}
d &\leq -CL \epsilon \sum_{i=1}^{r-1} \alpha^{k-1-i} t_i +
\frac{(C+1) \theta^2 - 2}{2}L \alpha^{k-r} t_r^2 \\
&= -\frac{CL \epsilon^2}{1-\alpha} \sum_{i=1}^{r-1} \alpha^{k-1-i}
(1-\alpha^i) + \frac{(C+1) \theta^2 - 2}{2}L \alpha^{k-r} t_r^2 \\ 
&= \frac{-CL\epsilon^2}{(1-\alpha)^2} \alpha^{k-r} (1-\alpha^{r-1})
+ \frac{CL\epsilon^2}{1-\alpha} \alpha^{k-1} (r-1)
+ \frac{(C+1) \theta^2 - 2}{2}L \alpha^{k-r} t_r^2 \\ 
&\leq \frac{-CL\epsilon^2}{(1-\alpha)^2} (\alpha^{k-r}-\alpha^{k-1}) 
+ \frac{CL\epsilon^2}{1-\alpha} \alpha^{k-1} r
+ \frac{(C+1) \theta^2 - 2}{2}L t^2 \\ 
&= -CL M^2 t^2 \Big((1-1/M)^{1-\varphi} - \frac{1}{\alpha}
(1-1/M)\Big) +\frac{CL Mt}{\alpha} (1-1/M) \epsilon k \varphi   
+\frac{(C+1) \theta^2 - 2}{2}L t^2.
\end{align*}
In the above, we have again written $\varphi=r/k$. Continuing with the
first part of this upper bound,
\begin{align*}
&-CL M^2 t^2 \Big((1-1/M)^{1-\varphi} - \frac{1}{\alpha}
(1-1/M)\Big) +\frac{CL Mt}{\alpha} (1-1/M) \epsilon k \varphi \\ 
&= CLM^2 t^2 \big(1/\alpha-(1-1/M)^{1-\varphi} \big) -
\frac{CLMt^2 }{\alpha} (1-\varphi) + \frac{\varphi CLt}{\alpha} 
M (\epsilon k - t) - \frac{\varphi CLt}{\alpha} \epsilon k \\
&= \frac{CLt^2}{\alpha} M \Big( M\big(1-(1-1/M)^{1-\varphi}\big) -
(1-\varphi)\Big) + \frac{\varphi CLt}{\alpha} 
M (\epsilon k - t) - \frac{\varphi CLt}{\alpha} \epsilon k \\
&\qquad + CLt^2 (1/\alpha-1) M^2 (1-1/M)^{1-\varphi}
\\
&\rightarrow \theta(1-\theta)CLt^2/2 + \theta CLt^2/2 - \theta CLt^2 +
0 \\
&= -\theta^2 CLt^2/2.
\end{align*}
In the second to last line, the first three terms converge according to
Lemmas \ref{lem:lims} and \ref{lem:lims2}.  The last one converges to
0 since $M^2(1-\alpha) = \epsilon^2/((1-\alpha) t^2) \rightarrow 0$ by
assumption.  Therefore, in the limit, we have 
\begin{equation*}
d \leq -\theta^2CLt^2/2
+ ((C+1)\theta^2-2) Lt^2/2 = (\theta^2-2)Lt^2/2.
\end{equation*}

\bigskip
\noindent
{\it Putting it all together.} To finish, we have shown that in the
limit as $\epsilon \rightarrow 0$, $M \rightarrow \infty$,
\begin{equation*}
h_t (\tx(t)) = A + B \leq Lt^2/2 + (1-\theta^2)Lt^2/2 +
(\theta^2-2)Lt^2/2 = 0.
\end{equation*}
This completes the proof.
\end{proof}

Below are two helper lemmas used in the proof of Theorem
\ref{thm:shrunkstage}. 

\begin{lemma}
\label{lem:lims}
For any fixed $t$, the following limits hold as 
$\epsilon \rightarrow 0$, $M \rightarrow \infty$:
\begin{gather*}
\epsilon \frac{\log(1-1/M)}{\log(1-\epsilon/(Mt))} 
\rightarrow t, \\
M \bigg(\epsilon \frac{\log(1-1/M)}{\log(1-\epsilon/(Mt))} 
- t \bigg) \rightarrow t/2.
\end{gather*}
\begin{proof}
Both limits can be verified using a Taylor expansion of  
$\log(1+x)$ around $x=0$.  For the first, note that
\begin{align*}
\epsilon \frac{\log(1-1/M)}{\log(1-\epsilon/(Mt))} &=
\frac{\epsilon/M + \epsilon/(2M^2) + \epsilon/(3M^3) + \ldots}
{\epsilon/(Mt) + \epsilon^2/(2M^2t^2) + \epsilon^3/(3M^3t^3) +
  \ldots} \\
&= \frac{1 + 1/(2M) + 1/(3M^2) + \ldots}
{1/t + \epsilon/(2Mt^2) + \epsilon^2/(3M^2t^3) + \ldots} \\
&\rightarrow t.
\end{align*}
And for the second,
\begin{align*}
M \bigg(\epsilon \frac{\log(1-1/M)}{\log(1-\epsilon/(Mt))} 
- t \bigg) &=
M \frac{1 + 1/(2M) + 1/(3M^2) + \ldots - 
t \big(1/t + \epsilon/(2Mt^2) + 
\ldots\big)}  
{1/t + \epsilon/(2Mt^2) + \epsilon^2/(3M^2t^3) + \ldots} \\
&= \frac{1/2 + 1/(3M) + \ldots - \big(\epsilon/(2t) + 
\epsilon^2/(3Mt^2) + \ldots\big)}
{1/t + \epsilon/(2Mt^2) + \epsilon^2/(3M^2t^3) + \ldots} \\
&\rightarrow t/2.
\end{align*}
\end{proof}
\end{lemma}

\begin{lemma}
\label{lem:lims2}
Let $\varphi=\varphi(M)$ be such that
$\varphi \rightarrow \theta$ as $M \rightarrow \infty$.
The following limits hold, as $M \rightarrow \infty$:
\begin{gather*}
M \big(1- (1-1/M)^\varphi\big) \rightarrow \theta, \\
M \Big( M \big(1- (1-1/M)^\varphi\big) - \varphi \Big) \rightarrow 
\theta(1-\theta)/2. 
\end{gather*}
\end{lemma}
\begin{proof}
We use a Taylor expansion of $(1+x)^a$ around $x=0$.  In particular, 
\begin{align*}
M \big(1 - (1-1/M)^\varphi \big) &=
M \bigg[ {\varphi \choose 1} \frac{1}{M} - 
{\varphi \choose 2} \frac{1}{M^2} + 
{\varphi \choose 3} \frac{1}{M^3} - \ldots \bigg] \\
&= {\varphi \choose 1} - {\varphi \choose 2} \frac{1}{M}
+ {\varphi \choose 3} \frac{1}{M^2} - \ldots \\
&\rightarrow {\theta \choose 1} = \theta.
\end{align*}
(Note that here ${x \choose i}$ denotes the generalized binomial
coefficient, for nonintegral $x$.)  Also,
\begin{align*}
M \Big( M \big(1- (1-1/M)^\varphi\big) - \varphi \Big) &=
M \bigg[ {\varphi \choose 1} - {\varphi \choose 2} \frac{1}{M}
+ {\varphi \choose 3} \frac{1}{M^2} - \ldots - \varphi \bigg] \\
&= - {\varphi \choose 2} 
+ {\varphi \choose 3} \frac{1}{M} - \ldots \\
&\rightarrow - {\theta \choose 2} = \frac{\theta(1-\theta)}{2}.
\end{align*}
\end{proof}

\newpage
\bibliographystyle{agsm}
\bibliography{ryantibs}

\end{document}